\documentclass[runningheads]{llncs}

\usepackage{wrapfig}
\usepackage[T1]{fontenc}
\usepackage[utf8]{inputenc}  
\usepackage{amsfonts}
\usepackage[export]{adjustbox}
\usepackage[colorlinks=true]{hyperref}
\usepackage{dsfont}
\usepackage{bm}
\usepackage{mathtools}
\usepackage{float}
\usepackage{times}
\usepackage[noend]{algorithm2e}
\usepackage{stmaryrd}
\usepackage[subrefformat=parens,labelformat=parens,caption=false]{subfig}
\usepackage{tikz}
\usetikzlibrary{arrows,automata}
\tikzset{terminal state/.style={draw,rectangle,minimum size=.3in}}
\usepackage{enumitem}
\usepackage{natbib}
\usepackage[symbol]{footmisc}
\usepackage{dblfloatfix}
\usepackage{xargs}       
\usepackage[colorinlistoftodos,prependcaption,textsize=tiny]{todonotes}
\usepackage[misc]{ifsym}
\newcommandx{\td}[2][1=]{\todo[inline,size=\large,#1]{#2}}

\DeclareMathOperator*{\argmax}{argmax}
\usepackage{dsfont}

\tikzstyle{startstop} = [rectangle, rounded corners, minimum width=3cm, minimum height=1cm,text centered, draw=black, fill=red!30]
\tikzstyle{io} = [trapezium, trapezium left angle=70, trapezium right angle=110, minimum width=3cm, minimum height=1cm, text centered, draw=black, fill=blue!30]
\tikzstyle{process} = [rectangle, minimum width=3cm, minimum height=1cm, text centered, draw=black, fill=orange!30]
\tikzstyle{decision} = [diamond, minimum width=3cm, minimum height=1cm, text centered, draw=black, fill=green!30]
\tikzstyle{arrow} = [thick,->,>=stealth]

\newtheorem{assumption}{Assumption}

\newcommand\blfootnote[1]{%
	\begingroup
	\renewcommand\thefootnote{}\footnote{#1}%
	\addtocounter{footnote}{-1}%
	\endgroup
}

\floatname{algorithm}{Pseudo-code}
\newenvironment{pseudocode}[1][htb]
  {
   \begin{algorithm}[#1]%
  }{\end{algorithm}}
  
\newenvironment{pseudocode*}[1][htb]
  {
   \begin{algorithm*}[#1]%
  }{\end{algorithm*}}

\providecommand{\qedsymbol}{$\square$}
\newcommand{\mathqed}{\quad\hbox{\qedsymbol}}
\DeclareRobustCommand{\qed}{%
  \ifmmode \mathqed
  \else
    \leavevmode\unskip\penalty9999 \hbox{}\nobreak\hfill
    \quad\hbox{\qedsymbol}%
  \fi
}

\spnewtheorem*{proofsketch}{Proof sketch}{\itshape}{\rmfamily}

\begin{document}

\title{Safe Policy Improvement with Soft Baseline Bootstrapping}
\titlerunning{Safe Policy Improvement with Soft Baseline Bootstrapping}
%
\author{
	Kimia Nadjahi$^{1\dagger *}$ 
	\and
	Romain Laroche$^{2*}$ 
	\and 
	R\'emi Tachet des Combes$^2$
}
\authorrunning{
	Kimia Nadjahi \and Romain Laroche \and R\'emi Tachet des Combes
}
\institute{LTCI, T\'el\'ecom Paris, Institut Polytechnique de Paris, France \\
\email{kimia.nadjahi@telecom-paris.fr}
\and 
Microsoft Research Montr\'eal, Canada \\
\email{$\{$romain.laroche,
remi.tachet$\}$@microsoft.com}.
}
%
\tocauthor{Kimia Nadjahi (T\'el\'ecom Paris),
Romain Laroche (Microsoft Research Montr\'eal),
R\'emi Tachet des Combes (Microsoft Research Montr\'eal),
}
\toctitle{Safe Policy Improvement with Soft Baseline Bootstrapping}

\maketitle

\begin{abstract}
Batch Reinforcement Learning (Batch RL) consists in training a policy using trajectories collected with another policy, called the behavioural policy. Safe policy improvement (SPI) provides guarantees with high probability that the trained policy performs better than the behavioural policy, also called baseline in this setting. Previous work shows that the SPI objective improves mean performance as compared to using the basic RL objective, which boils down to solving the MDP with maximum likelihood~\citep{Laroche2019}. Here, we build on that work and improve more precisely the SPI with Baseline Bootstrapping algorithm (SPIBB) by allowing the policy search over a wider set of policies. Instead of binarily classifying the state-action pairs into two sets (the \textit{uncertain} and the \textit{safe-to-train-on} ones), we adopt a softer strategy that controls the error in the value estimates by constraining the policy change according to the local model uncertainty. The method can take more risks on uncertain actions all the while remaining provably-safe, and is therefore less conservative than the state-of-the-art methods. We propose two algorithms (one optimal and one approximate) to solve this constrained optimization problem and empirically show a significant improvement over existing SPI algorithms both on finite MDPs and on infinite MDPs with a neural network function approximation.
\end{abstract}

\section{Introduction}

In sequential decision-making problems, a common goal is to find a good policy using a limited number of trajectories generated by another policy, usually called the behavioral policy. This approach, also known as Batch Reinforcement Learning \citep{Lange2012}, is motivated by the many real-world applications that naturally fit a setting where data collection and optimization are decoupled (contrary to online learning which integrates the two): \emph{e.g.} dialogue systems~\citep{Singh1999}, technical process control~\citep{Ernst2005,Riedmiller2005}, medical applications~\citep{Guez2008}.\setcounter{footnote}{1}\footnotetext{Equal contribution.} 
\stepcounter{footnote}\footnotetext{Work done while interning at Microsoft Research Montr\'eal.}\setcounter{footnote}{0}\renewcommand{\thefootnote}{\arabic{footnote}}\blfootnote{Finite MDPs code available at \href{https://github.com/RomainLaroche/SPIBB}{\texttt{https://github.com/RomainLaroche/SPIBB}}.}\blfootnote{SPIBB-DQN code available at \href{https://github.com/rems75/SPIBB-DQN}{\texttt{https://github.com/rems75/SPIBB-DQN}}.}

While most reinforcement learning techniques aim at finding a high-performance policy~\citep{Sutton1998}, the final policy does not necessarily perform well once it is deployed. In this paper, we focus on Safe Policy Improvement \citep[SPI,][]{thomas2015safe,Petrik2016}, where the goal is to train a policy on a batch of data and guarantee with high probability that it performs at least as well as the behavioural policy, called baseline in this SPI setting. The safety guarantee is crucial in real-world applications where bad decisions may lead to harmful consequences. 

Among the existing SPI algorithms, a recent computationally efficient and provably-safe methodology is SPI with Baseline Bootstrapping \citep[SPIBB,][]{Laroche2019,Dias2019}. Its principle consists in building the set of state-action pairs that are only encountered a few times in the dataset. This set is called the bootstrapped set. The algorithm then reproduces the baseline policy for all pairs in that set and trains greedily on the rest. It therefore assumes access to the baseline policy, which is a common assumption in the SPI literature~\citep{Petrik2016}. Other SPI algorithms use as reference the baseline performance, which is assumed to be known instead~\citep{thomas2015safe,Petrik2016}. We believe that the known policy assumption is both more informative and more common, since most Batch RL settings involve datasets that were collected using a previous system based on a previous algorithm (\emph{e.g.} dialogue, robotics, pharmaceutical treatment). While the empirical results show that SPIBB is safe and performs significantly better than the existing algorithms, it remains limited by the binary classification of the bootstrapped set: a pair either belongs to it, and the policy cannot be changed, or it does not, and the policy can be changed entirely.

Our contribution is a reformulation of the SPIBB objective that allows slight policy changes for uncertain state-action pairs while remaining safe. Instead of binarily classifying the state-action pairs into two sets, the uncertain and the \text{safe-to-train-on} ones, we adopt a strategy that extends the policy search to soft policy changes, which are constrained by an error bound related to the model uncertainty. The method is allowed to take more risks than SPIBB on uncertain actions, and still has theoretical safety guarantees under some assumptions. As a consequence, the safety constraint is softer: we coin this new SPI methodology \emph{Safe Policy Improvement with Soft Baseline Bootstrapping} (Soft-SPIBB). We develop two algorithms to tackle the Soft-SPIBB problem. The first one solves it exactly, but is computationally expensive. The second one provides an approximate solution but is much more efficient computation-wise. We empirically evaluate the performance and safety of our algorithms on a gridworld task and analyze the reasons behind their significant advantages over the competing Batch RL algorithms. We further demonstrate the tractability of the approach by designing a DQN algorithm enforcing the Soft-SPIBB constrained policy optimization. The empirical results, obtained on a navigation task, show that Soft-SPIBB safely improves the baseline, and again outperforms all competing algorithms.

\section{Background}\label{sec:background}

\subsection{Markov Decision Processes}

We consider problems in which the agent interacts with an environment modeled as a \emph{Markov Decision Process} (MDP): $M^* = \langle \mathcal{X}, \mathcal{A}, P^*, R^*, \gamma \rangle$, where $\mathcal{X}$ is the set of states, $\mathcal{A}$ the set of actions, $P^*$ the unknown transition probability function, $R^*$ the unknown stochastic reward function bounded by $\pm R_{max}$, and $\gamma \in [0, 1)$ the discount factor for future rewards. The goal is to find a policy $\pi: \mathcal{X} \rightarrow \Delta_{\mathcal{A}}$, with $\Delta_{\mathcal{A}}$ the set of probability distributions over the set of actions $\mathcal{A}$, that maximizes the expected return of trajectories $\rho(\pi, M^*) = V^{\pi}_{M^*}(x_0)=\mathbb{E}_{\pi, M^*}\left[ \sum_{t \geq 0} \gamma^t R^*(x_t, a_t) \right]$. $x_0$ is the initial state of the environment and $V^{\pi}_{M^*}(x)$ is the value of being in state $x$ when following policy $\pi$ in MDP $M^*$. We denote by $\Pi$ the set of stochastic policies. Similarly to $V^{\pi}_{M^*}(x)$, $Q^{\pi}_{M^*}(x,a)$ denotes the value of taking action $a$ in state $x$. $A^{\pi}_{M}(x,a) = Q^{\pi}_{M}(x,a) - V^{\pi}_{M}(x)$ quantifies the advantage (or disadvantage) of action $a$ in state $x$.

Given a dataset of transitions $\mathcal{D}=\langle x_j,a_j,r_j,x'_j \rangle_{j\in\llbracket 1,|\mathcal{D}|\rrbracket}$, we denote the state-action pair counts by $N_{\mathcal{D}}(x,a)$, and its Maximum Likelihood Estimator (MLE) MDP by $\widehat{M} = \langle \mathcal{X}, \mathcal{A}, \widehat{P}, \widehat{R}, \gamma \rangle$, with:
\begin{align*}
    \widehat{P}(x'|x,a) = \cfrac{\sum_{\langle x_j=x,a_j=a,r_j,x'_j=x' \rangle\in\mathcal{D}} 1}{N_{\mathcal{D}}(x,a)} \text{\; and \;} \widehat{R}(x,a) = \cfrac{\sum_{\langle x_j=x,a_j=a,r_j,x'_j \rangle\in\mathcal{D}} r_j}{N_{\mathcal{D}}(x,a)}.
\end{align*}

The difference between an estimated parameter and the true one can be bounded using classic concentration bounds applied to the state-action counts in $\mathcal{D}$~\citep{Petrik2016,Laroche2019}: for all state-action pairs $(x,a)$, we know with probability at least $1 - \delta$ that,
\begin{align}
    &||P^*(\cdot|x,a)-\widehat{P}(\cdot|x,a)||_1 \leq e_P(x,a),\ |R^*(x,a)-\widehat{R}(x,a)| \leq e_P(x,a)R_{max}, \label{eq:bound-P} \\
    &\big\lvert Q^{\pi_b}_{M^*}(x,a) - Q^{\pi_b}_{\widehat{M}}(x,a) \big\rvert\ \leq e_Q(x,a) V_{max}, \label{eq:bound-Q}
\end{align}
where $V_{max} \leq \dfrac{R_{max}}{1 - \gamma}\ $ is the maximum of the value function, and the two error functions may be derived from Hoeffding's inequality (see \ref{sup:error_bounds}) as
\begin{align*}
    e_P(x,a) := \sqrt{\cfrac{2}{N_{\mathcal{D}}(x,a)}\log\cfrac{2|\mathcal{X}||\mathcal{A}|2^{|\mathcal{X}|}}{\delta}}  \text{\; and \;} 
    e_Q(x,a) := \sqrt{\cfrac{2}{N_{\mathcal{D}}(x,a)}\log\cfrac{2|\mathcal{X}||\mathcal{A}|}{\delta}}.
\end{align*}
We will also use the following definition:
\begin{definition}
    A policy $\pi$ is said to be a policy improvement over a baseline policy $\pi_b$ in an MDP $M = \langle \mathcal{X}, \mathcal{A}, P, R, \gamma\rangle$ if the following inequality holds in every state $x\in\mathcal{X}$:
    \begin{align}
       V^{\pi}_{M}(x) \geq V^{\pi_b}_{M}(x)
    \end{align}
\end{definition}

\subsection{Safe Policy Improvement with Baseline Bootstrappping}

Our objective is to maximize the expected return of the target policy under the constraint of improving with high probability $1-\delta$ the baseline policy. This is known to be an NP-hard problem \citep{Petrik2016} and some approximations are required to make it tractable. This paper builds on the Safe Policy Improvement with Baseline Bootstrapping methodology \citep[SPIBB,][]{Laroche2019}. SPIBB finds an approximate solution to the problem by searching for a policy maximizing the expected return in the MLE MDP $\widehat{M}$, under the constraint that the policy improvement is guaranteed in the set of plausible MDPs $\Xi$:
\begin{align}
	&\argmax_{\pi} \rho(\pi,\widehat{M}),  \textnormal{ s.t. } \forall M \in \Xi, \rho(\pi,M) \geq \rho(\pi_b,M) - \zeta \\
	\Xi = &\left\{M = \langle \mathcal{X}, \mathcal{A}, R, P, \gamma\rangle 
	\textnormal{ s.t. } \forall x,a, 
	\begin{array}{ll}
	||P(\cdot|x,a)-\widehat{P}(\cdot|x,a)||_1 \leq e_P(x,a),\\
	|R(x,a)-\widehat{R}(x,a)| \leq e_P(x,a)R_{max} 
	\end{array}\right\}
	\end{align}
The error function $e_P$ is such that the true MDP $M^*$ has a high probability of at least $1-\delta$ to belong to $\Xi$~\citep{Iyengar2005,Nilim2005}. In other terms, the objective is to optimize the target performance in $\widehat{M}$ such that its performance is $\zeta$-approximately at least as good as $\pi_b$ in the admissible MDP set, where $\zeta$ is a precision hyper-parameter. Expressed this way, the problem is still intractable. SPIBB is able to find an approximate solution within a tractable amount of time by applying a special processing to state-action pair transitions that were not sampled enough in the batch of data. The methodology consists in building a set of rare thus uncertain state-action pairs in the dataset $\mathcal{D}$, called the bootstrapped set and denoted by $\mathcal{B}$: the bootstrapped set contains all the state-action pairs $(x,a) \in \mathcal{X} \times \mathcal{A}$ whose counts in $\mathcal{D}$ are lower than a hyper-parameter $N_{\wedge}$. SPIBB algorithms then construct a space of allowed policies, \textit{i.e} policies that are constrained on the bootstrapped set $\mathcal{B}$, and search for the optimal policy in this set by performing policy iteration. For example, $\Pi_b$-SPIBB is a provably-safe algorithm that assigns the baseline $\pi_b$ to the state-action pairs in $\mathcal{B}$ and trains the policy on the rest. $\Pi_{\leq b}$-SPIBB is a variant that does not give more weight than $\pi_b$ to the uncertain transitions.

SPIBB's principle amounts to search over a policy space constrained such that the policy improvement may be precisely assessed in $M^*$. Because of the hard definition of the bootstrapped set, SPIBB relies on a binary decision-making and may be too conservative. Our novel method, called Soft-SPIBB, follows the same principle, but relaxes this definition by allowing soft policy changes for the uncertain state-action pairs, and offers more flexibility than SPIBB while remaining safe. 

This idea might seem similar to Conservative Policy Iteration (CPI), Trust Region Policy Optimization (TRPO), or Proximal Policy Optimization (PPO) in that it allows changes in the policy under a proximity regularization to the old policy \citep{Kakade2002,Schulman2015,Schulman2017}. However, with Soft-SPIBB, the proximity constraint is tightened or relaxed according to the amount of samples supporting the policy change (see Definition~\ref{def:constraint}). Additionally, CPI, TRPO, and PPO are designed for the online setting. In the batch setting we consider, they would be either too conservative if the proximity regularization is applied with respect to the fixed baseline, or would converge to the fixed point obtained when solving the MLE MDP if the proximity regularization is moving with the last policy update~\citep[Corollary 3 of][]{Geist2019}.

\subsection{Linear Programming} \label{subsec:linearprog}

Linear programming aims at optimizing a linear objective function under a set of linear in-equality constraints. The most common methods for solving such linear programs are the simplex algorithm and interior point methods \citep[IPMs,][]{GVK180926950}. Even though the worst-case computational complexity of the simplex is exponential in the dimensions of the program being solved \citep{klee:1972}, this algorithm is efficient in practice: the number of iterations seems polynomial, and sometimes linear in the problem size \citep{borgwardt:87, dantzig:03}. Nowadays, these two classes of methods continue to compete with one another: it is hard to predict the winner on a particular class of problems \citep{Gondzio_interiorpoint}. For instance, the hyper-sparsity of the problem generally seems to favour the simplex algorithm, while IPMs can be much more efficient for large-scale linear programming. 


\section{Safe Policy Improvement with Soft Baseline Bootstrapping}\label{sec:spisbb}

SPIBB allows to make changes in state-action pairs where the model error does not exceed some threshold $\epsilon$, which may be expressed as a function of $N_\wedge$. This may be seen as a hard condition on the bootstrapping mechanism: a state-action pair policy may either be changed totally, or not at all. In this paper, we propose a softer mechanism where, for a given error function, a local error budget is allocated for policy changes in each state $x$. Similarly to SPIBB, we search for the optimal policy in the MDP model $\widehat{M}$ estimated from the dataset $\mathcal{D}$, but we reformulate the constraint by using Definitions \ref{def:constraint} and \ref{def:advantageous_constraint}.
 
    \begin{definition}
        A policy $\pi$ is said to be $(\pi_b, e,\epsilon)$-constrained with respect to a baseline policy $\pi_b$, an error function $e$, and a hyper-parameter $\epsilon$ if, for all states $x\in\mathcal{X}$, the following inequality holds:
        \begin{equation*}
            \sum_{a \in \mathcal{A}} e(x,a) \big\lvert\pi(a|x)-\pi_b(a|x)\big\rvert \leq \epsilon.
        \end{equation*}
    \label{def:constraint}
    \end{definition}
    
        \begin{definition}
        A policy $\pi$ is said to be $\pi_b$-advantageous in an MDP $M = \langle \mathcal{X}, \mathcal{A}, P, R, \gamma\rangle$ if the following inequality holds in every state $x\in\mathcal{X}$:
        \begin{align}
           \sum_{a \in \mathcal{A}} A^{\pi_b}_{M}(x,a)\pi(a|x) \geq 0
        \end{align}
    \label{def:advantageous_constraint}
    \end{definition}
    
    \begin{remark}
        By the policy improvement theorem, a $\pi_b$-advantageous policy is a policy improvement over $\pi_b$. The converse is not guaranteed.
    \end{remark}


\subsection{Theoretical safe policy improvement bounds}

We show that constraining $\pi_b$-advantageous policies appropriately allows safe policy improvements. Due to space limitation, all proofs have been moved to the appendix, Section \ref{sup:proofs}.
\begin{theorem}
    Any $(\pi_b, e_Q,\epsilon)$-constrained policy $\pi$ that is $\pi_b$-advantageous in $\widehat{M}$ satisfies the following inequality in every state $x$ with probability at least $1 - \delta$:
    \begin{align}
        V^\pi_{M^*}(x)-V^{\pi_b}_{M^*}(x) &\geq -\cfrac{\epsilon V_{max}}{1-\gamma}.
    \end{align}
    \label{th:1-step}
\end{theorem}


Constraining the target policy to be advantageous over the baseline is a strong constraint that leads to conservative solutions. To the best of our findings, it is not possible to prove a more general bound on $(\pi_b, e_Q,\epsilon)$-constrained policy improvements. However, the search over $(\pi_b, e_P,\epsilon)$-constrained policies, where $e_P$ is an error bound over the probability function $P$ (Equation \ref{eq:bound-Q}), allows us to guarantee safety bounds under Assumption \ref{ass:kappa}, which states: 
    \begin{assumption}
        There exists a constant $\kappa < \frac{1}{\gamma}$ such that, for all state-action pairs $(x,a)\in\mathcal{X}\times\mathcal{A}$, the following inequality holds:
        \begin{align}
            \sum_{x',a'} e_P(x',a') \pi_b(a'|x') P^*(x'|x,a) \leq \kappa e_P(x,a).
        \end{align}
        \label{ass:kappa}
    \end{assumption}

Lemma \ref{lem:induction}, which is essential to prove Theorem \ref{thm:SSPIBB} below, relies on Assumption \ref{ass:kappa}.
    
    \begin{lemma}  Under Assumption \ref{ass:kappa}, any $(\pi_b, e_P,\epsilon)$-constrained policy $\pi$ satisfies the following inequality for every state-action pair $(x,a)$ with probability at least $1 - \delta$:
        \begin{align}
            \big\lvert Q^{\pi}_{M^*}(x,a) - Q^{\pi}_{\widehat{M}}(x,a)\big\rvert \leq \left(\cfrac{e_P(x,a)}{1-\kappa\gamma} + \cfrac{\gamma\epsilon}{\left(1-\gamma\right)\left(1-\kappa\gamma\right)}\right)V_{max}. \nonumber
        \end{align}
        \label{lem:induction}
    \end{lemma}
    
    \begin{theorem}
        Under Assumption \ref{ass:kappa}, any $(\pi_b, e_P,\epsilon)$-constrained policy $\pi$ satisfies the following inequality in every state $x$ with probability at least $1 - \delta$:
        \begin{align}
             V^\pi_{M^*}(x)-V^{\pi_b}_{M^*}(x) \geq V^\pi_{\widehat{M}}(x)-V^{\pi_b}_{\widehat{M}}(x) &- 2\big\lVert d^{\pi_b}_{M^*}(\,\boldsymbol{\cdot}\,|x)-d^{\pi_b}_{\widehat{M}}(\,\boldsymbol{\cdot}\,|x) \big\rVert_1 V_{max} \nonumber  \\
             &- \cfrac{1 + \gamma}{\left(1-\gamma\right)^2\left(1-\kappa\gamma\right)}\;\;\epsilon V_{max}.
        \end{align}
        \label{thm:SSPIBB}
    \end{theorem}

\begin{remark} The theorems hold for any error function $e_P$ verifying \ref{eq:bound-Q} w.p. $1 - \delta$.
\end{remark}

\begin{remark} $\Pi_b$-SPIBB~\citep{Laroche2019} is a particular case of Soft-SPIBB where the error function $e_P(x,a)$ equals $\infty$ if $(x,a)\in\mathcal{B}$ and $\frac{\epsilon}{2}$ otherwise.
\end{remark}

\begin{remark} Theorem \ref{thm:SSPIBB} has a cubic dependency in the horizon $\frac{1}{1-\gamma}$, which is weaker than SPIBB's bounds, but allow us to safely search over more policies, when using tighter error functions. We will observe in Section \ref{sec:emp} that Soft-SPIBB empirically outperforms SPIBB both in mean performance and in safety.
\end{remark}

\subsection{Algorithms}\label{subsec:algos}

In this section, we design two safe policy improvement algorithms to tackle the problem defined by the Soft-SPIBB approach. They both rely on the standard policy iteration process described in Pseudo-code~\ref{ps:Soft-SPIBB}, where the policy improvement step consists in solving in every state $x \in \mathcal{X}$ the locally constrained optimization problem below:
\begin{equation}
    \pi^{(i+1)}(\cdot | x) = \argmax_{\pi \in \Pi} \sum_{a \in \mathcal{A}} Q^{(i)}_{\widehat{M}}(x, a) \pi(a|x) \label{eq:lp}
\end{equation}
subject to:

\textbf{Constraint 1:} $\pi$ being a probability: $\sum_{a \in \mathcal{A}} \pi(a|x) = 1$ and $\forall a$, $\pi(a|x) \geq 0$.

\textbf{Constraint 2:} $\pi$ being $(\pi_b, e, \epsilon)$-constrained.

\RestyleAlgo{ruled}
\begin{pseudocode}
\caption{Policy iteration process for Soft-SPIBB}  \label{ps:Soft-SPIBB}
 \KwIn{Baseline policy $\pi_b$, MDP model precision level $\epsilon$ and dataset $\mathcal{D}$.}
 Compute the model error concentration bounds $e(x,a)$.\\
 Initialize $i=0$ and $\pi^{(0)}(\cdot | x) = \pi_b (\cdot | x)$.\\
 \While{policy iteration stopping criterion not met}{
    Policy evaluation: compute $Q_{\widehat{M}}^{(i)}$ with dynamic programming. \\
    Policy improvement: set $\pi^{(i+1)}(\cdot | x)$ as the (exact or approximate) solution of the optimization problem defined in Equation \ref{eq:lp}.\\
    $i \leftarrow i+1$
}
\Return $\pi^{(i)}$ \\
\end{pseudocode}

\subsubsection{Exact-Soft-SPIBB:} The Exact-Soft-SPIBB algorithm computes the exact solution of the local optimization problem in \eqref{eq:lp} during the policy improvement step. For that, we express the problem as a Linear Program (LP) and solve it by applying the simplex algorithm. Note that we chose the simplex over IPMs as it turned out to be efficient enough for our experimental settings. For tractability in large action spaces, we reformulate the non-linear Constraint 2 as follows: we introduce $|\mathcal{A}|$ auxiliary variables $\{ z(x,a) \}_{(x,a) \in \mathcal{X}\times\mathcal{A}}$, which bound from above each element of the sum. For a given $x \in \mathcal{X}$, Constraint 2 is then replaced by the following $2|\mathcal{A}|+1$ linear constraints: 
\begin{align}
    &\forall a \in \mathcal{A},&\ \pi(a|x) - \pi_b(a|x) &\leq z(x,a), \\
    &\forall a \in \mathcal{A},&\ -\pi(a|x) + \pi_b(a|x) &\leq z(x,a), \\
    &&\sum_a e(x,a) z(x,a) &\leq \epsilon.
\end{align}

\subsubsection{Approx-Soft-SPIBB:} We also propose a computationally-efficient algorithm, which returns a sub-optimal target policy $\pi^\odot_{\sim}$. It relies on the same policy iteration, but computes an approximate solution to the optimization problem. The approach still guarantees to improve the baseline in $\widehat{M}$: $\rho(\pi^\odot_{\sim}, \widehat{M}) \geq \rho(\pi_b, \widehat{M})$, and falls under the Theorems \ref{th:1-step} and \ref{thm:SSPIBB} SPI bounds. Approx-Soft-SPIBB's local policy improvement step consists in removing, for each state $x$, the policy probability mass $m^-$ from the action $a^-$ with the lowest $Q$-value. Then, $m^-$ is attributed to the action that offers the highest $Q$-value improvement by unit of error $\partial \epsilon$:
\begin{align}
    a^+ &= \argmax_{a\in\mathcal{A}} \cfrac{\partial \pi(a|x)}{\partial \epsilon}\left(Q^{(i)}_{\widehat{M}}(x,a)-Q^{(i)}_{\widehat{M}}(x,a^-)\right) \\
    &= \argmax_{a\in\mathcal{A}} \cfrac{Q^{(i)}_{\widehat{M}}(x,a)-Q^{(i)}_{\widehat{M}}(x,a^-)}{e(x,a)}
\end{align}

Once $m^-$ has been reassigned to another action with higher value, the budget is updated accordingly to the error that has been spent, and the algorithm continues with the next worst action until a stopping criteria is met: the budget is fully spent, or $a^-=a^*$, where $a^*$ is the action with maximal state-action value. The policy improvement step of Approx-Soft-SPIBB is further formalized in Pseudo-code~\ref{alg:SSPIBB-sort-Q}, found in the appendix, Section \ref{sup:policy-approximate-policy-improvement}. 

\begin{theorem}
    The policy improvement step of Approx-Soft-SPIBB generates policies that are guaranteed to be $(\pi_b,e,\epsilon)$-constrained.
    \label{thm:PI}
\end{theorem}


\begin{remark}
The $\argmax$ operator in the result returned by Pseudo-code \ref{alg:SSPIBB-sort-Q} is a convergence condition. Indeed, the approximate algorithm does not guarantee that the current iteration policy search space includes the previous iteration policy, which can cause divergence: the algorithm may indefinitely cycle between two or more policies. To ensure convergence, we update $\pi^{(i)}$ with $\pi^{(i+1)}$ only if there is a local policy improvement, \textit{i.e.} when $\mathbb{E}_{a \sim \pi^{(i+1)}(\cdot|x)}[Q^{(i)}_{\widehat{M}}(x,a)] \geq \mathbb{E}_{a \sim \pi^{(i)}(\cdot|x) }[Q^{(i)}_{\widehat{M}}(x,a)]$.
\end{remark}

Both implementation of the Soft-SPIBB strategy comply to the requirements of Theorem \ref{th:1-step} if only one policy iteration is performed. In Section \ref{sec:RandomMDPs}, we empirically evaluate the 1-iteration versions, which are denoted by the `1-step' suffix.

\subsubsection{Complexity Analysis:}
\label{subsec:complexity}
We study the computational complexity of Exact-Soft-SPIBB and Approx-Soft-SPIBB. The error bounds computation and the policy evaluation step are common to both algorithms, and have a complexity of $\mathcal{O}(|\mathcal{D}|)$ and $\mathcal{O}(|\mathcal{X}|^3|\mathcal{A}|^3)$ respectively. The part that differs between them is the policy improvement.

Exact-Soft-SPIBB solves the LP with the simplex algorithm, which, as recalled in Section~\ref{subsec:linearprog}, is in practice polynomial in the dimensions of the program being solved. In our case, the number of constraints is $3|\mathcal{A}| + 1$.  

\begin{theorem}
Approx-Soft-SPIBB policy improvement has a complexity of $\mathcal{O}(|\mathcal{X}||\mathcal{A}|^2)$.
\label{thm:complexity}
\end{theorem}

\subsubsection{Model-free Soft-SPIBB:}
\label{subsec:Soft-SPIBB-DQN}
The Soft-SPIBB fixed point may be found in a model-free manner by fitting the $Q$-function
to the target $y^{(i+1)}$ on the transition samples $\mathcal{D} = \langle x_j,a_j,r_j,x'_j\rangle_{j\in\llbracket 1,N\rrbracket}$:
\begin{align}
    y^{(i+1)}_j &= r_j + \gamma \sum_{a'\in\mathcal{A}} \pi^{(i+1)}(a'|x_j') Q^{(i)}(x_j',a'), \label{eq:fittedQ}
\end{align}
where $\pi^{(i+1)}$ is obtained either exactly or approximately with the policy improvement steps described in Section \ref{subsec:algos}. Then, the policy evaluation consists in fitting $Q^{(i+1)}(x,a)$ to the set of $y^{(i+1)}_j$ values computed using the samples from $\mathcal{D}$.

\begin{theorem}
    Considering an MDP with exact counts, the model-based policy iteration of (Exact or Approx)-Soft-SPIBB is identical to the model-free policy iteration of (resp. Exact or Approx)-Soft-SPIBB.
    \label{thm:modelfree}
\end{theorem}

The model-free versions are less computationally efficient than their respective model-based versions, but are particularly useful since it makes function approximation easily applicable. In our infinite MDP experiment, we consider Approx-Soft-SPIBB-DQN as the DQN algorithm fitted to the model-free Approx-Soft-SPIBB targets. The Exact-Soft-SPIBB counterpart is not considered for tractability reasons. We recall that the computation of the policy improvement step relies on the estimates of an error function $e_P$, which may, for instance, be indirectly inferred from pseudo-counts $\widetilde{N}_{\mathcal{D}}(x,a)$ \citep{Bellemare2016,Fox2018,Burda2019}.

\section{Soft-SPIBB Empirical Evaluation}\label{sec:emp}
This section intends to empirically validate the advances granted by Soft-SPIBB. We perform the study on two domains: on randomly generated finite MDPs, where the Soft-SPIBB algorithms are compared to several Batch RL competitors: basic RL, High Confidence Policy Improvement~\citep[][HCPI]{thomas2015safe}, Reward-Adjusted MDPs \citep[][RaMDP]{Petrik2016}, Robust MDPs \citep{Iyengar2005,Nilim2005}, and to Soft-SPIBB natural parents: $\Pi_b$-SPIBB and $\Pi_{\leq b}$-SPIBB~\citep{Laroche2019}; and on a helicopter navigation task requiring function approximation, where Soft-SPIBB-DQN is compared to basic DQN, RaMDP-DQN, and SPIBB-DQN. All the benchmark algorithms had their hyper-parameters optimized beforehand. Their descriptions and the results of the hyper-parameter search is available in the appendix, Section \ref{sup:benchmarkalgos} for finite MDPs algorithms and Section \ref{sup:helicopter_training} for DQN-based algorithms.

In order to assess the safety of an algorithm, we run a large number of times the same experiment with a different random seed. Since the environments and the baselines are stochastic, every experiment generates a different dataset, and the algorithms are evaluated on their mean performance over the experiments, and on their conditional value at risk performance (CVaR), sometimes also called the expected shortfall: $X\%$-CVaR corresponds to the mean performance over the $X\%$ worst runs.

\subsection{Random MDPs}
\label{sec:RandomMDPs}
In the random MDPs experiment, the MDP and the baseline are themselves randomly generated too. The full experimental process is formalized in Pseudo-code \ref{alg:garnet_benchmark} found in the appendix, Section \ref{sup:randommdps}. Because every run involves different MDP and baseline, there is the requirement for a normalized performance. This is further defined as $\overline{\rho}$:
\begin{align} \label{eq:normalized_perf}
    \overline{\rho}(\pi,M^*)=\cfrac{\rho(\pi,M^*) - \rho(\pi_b,M^*)}{\rho(\pi^*,M^*)-\rho(\pi_b,M^*)}.
\end{align}

In order to demonstrate that Soft-SPIBB algorithms are safely improving the baselines on most MDPs in practice, we use a random generator of MDPs. All the details may be found in the appendix, Section \ref{sup:MDPgen}. The number of states is set to $|\mathcal{X}|=50$, the number of actions to $|\mathcal{A}|=4$ and the connectivity of the transition function to 4, \emph{i.e.}, for a given state-action pair $(x,a)$, its transition function $P(x'|x,a)$ is non-zero on four states $x'$ only. The reward function is 0 everywhere except when entering the goal state, which is terminal and where the reward is equal to 1. The goal is chosen in such a way that the optimal value function is minimal. 

\paragraph{Random baseline:} For a randomly generated MDP $M$, baselines are generated according to a predefined level of performance $\eta\in \left\{0.1,0.2,0.3,0.4,0.5,0.6,0.7,0.8,0.9\right\}$: $\rho(\pi_b,M) = \eta\rho(\pi^*,M) + (1-\eta)\rho(\tilde{\pi},M)$, where $\pi^*$ and $\tilde{\pi}$ are respectively the optimal and the uniform policies. The generation of the baseline consists in three steps: optimization, where the optimal $Q$-value is computed; softening, where a softmax policy is generated; and randomization, where the probability mass is randomly displaced in the baseline. The process is formally and extensively detailed in the appendix, Section~\ref{sup:baselinegen}.

\paragraph{Dataset generation:} Given a fixed size number of trajectories, a dataset is generated on the following modification of the original MDPs: addition of another goal state (reward is set to 1). Since the original goal state was selected so as to be the hardest to reach, the new one, which is selected uniformly, is necessarily a better goal.

\begin{wrapfigure}{r}{0.43\columnwidth}
    \centering
    \includegraphics[trim = 5pt 15pt 5pt 5pt, clip, width=0.43\columnwidth]{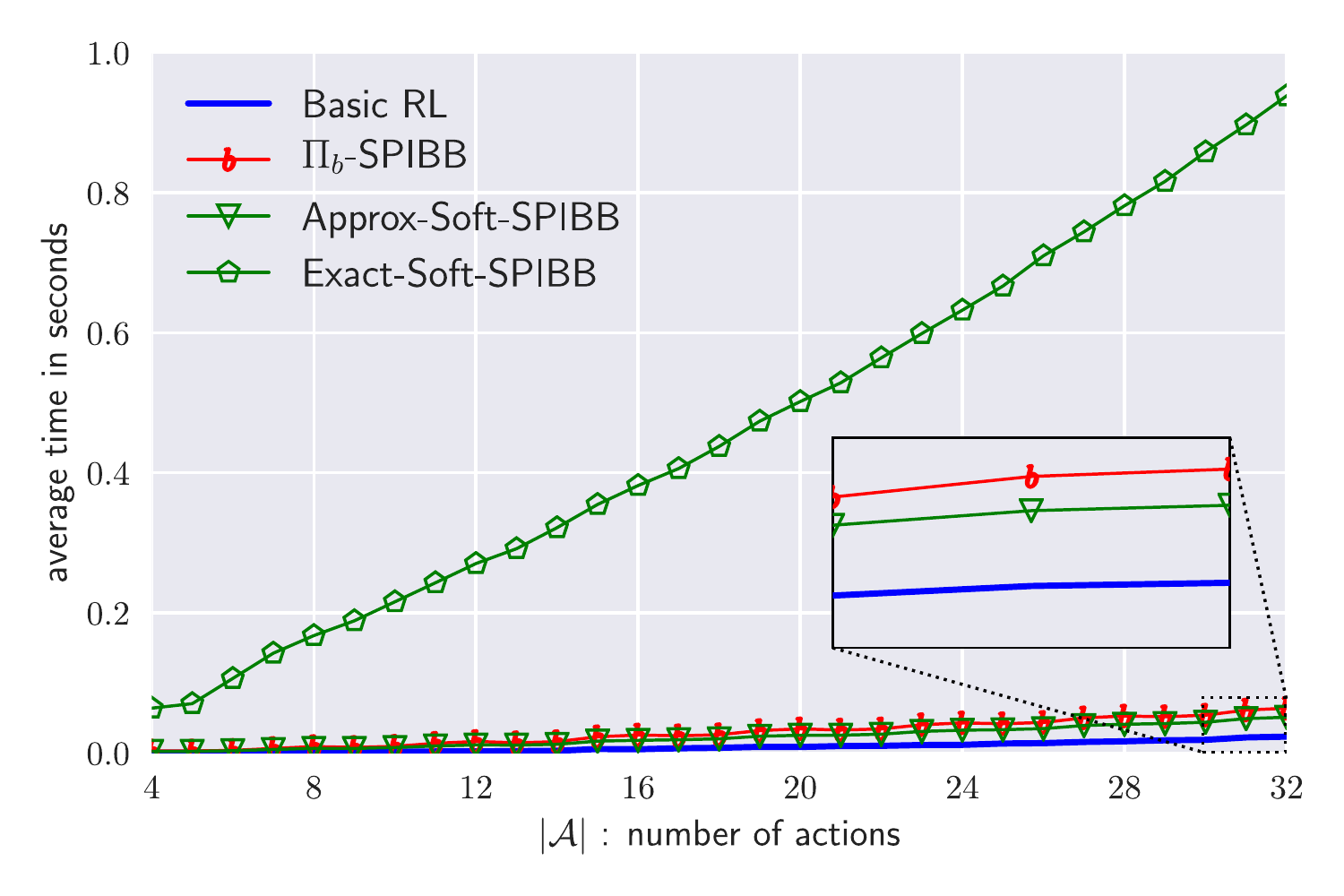}
    \caption{Average time to convergence.}
    \label{subfig:complarge}
    \vspace{-10pt}
\end{wrapfigure}
\paragraph{Complexity empirical analysis:} In Figure \ref{subfig:complarge}, we show an empirical confirmation of the complexity results on the gridworld task. Exact-Soft-SPIBB has a linear dependency in the number of actions. We also notice that Approx-Soft-SPIBB runs much faster: even faster than $\Pi_b$-SPIBB, and 2 times slower than basic RL. Note that the policy improvement step is by design exactly linearly dependent on the number of states $|\mathcal{X}|$. This is the reason why we do not report experiments on the dependency on $|\mathcal{X}|$. We do not report complexity empirical analysis of the other competitors because we do not pretend to have optimal implementations of them, and the purpose of this analysis is to show that Approx-Soft-SPIBB solves the tractability issues of Exact-Soft-SPIBB. In Theory, Robust MDPs and HCPI are more complex by at least an order of magnitude.

\begin{figure}[b!]
	\centering
	\subfloat[Mean: $\eta = 0.9$, $\epsilon = 2$]{
		\includegraphics[trim = 5pt 5pt 5pt 5pt, clip, width=0.5\columnwidth]{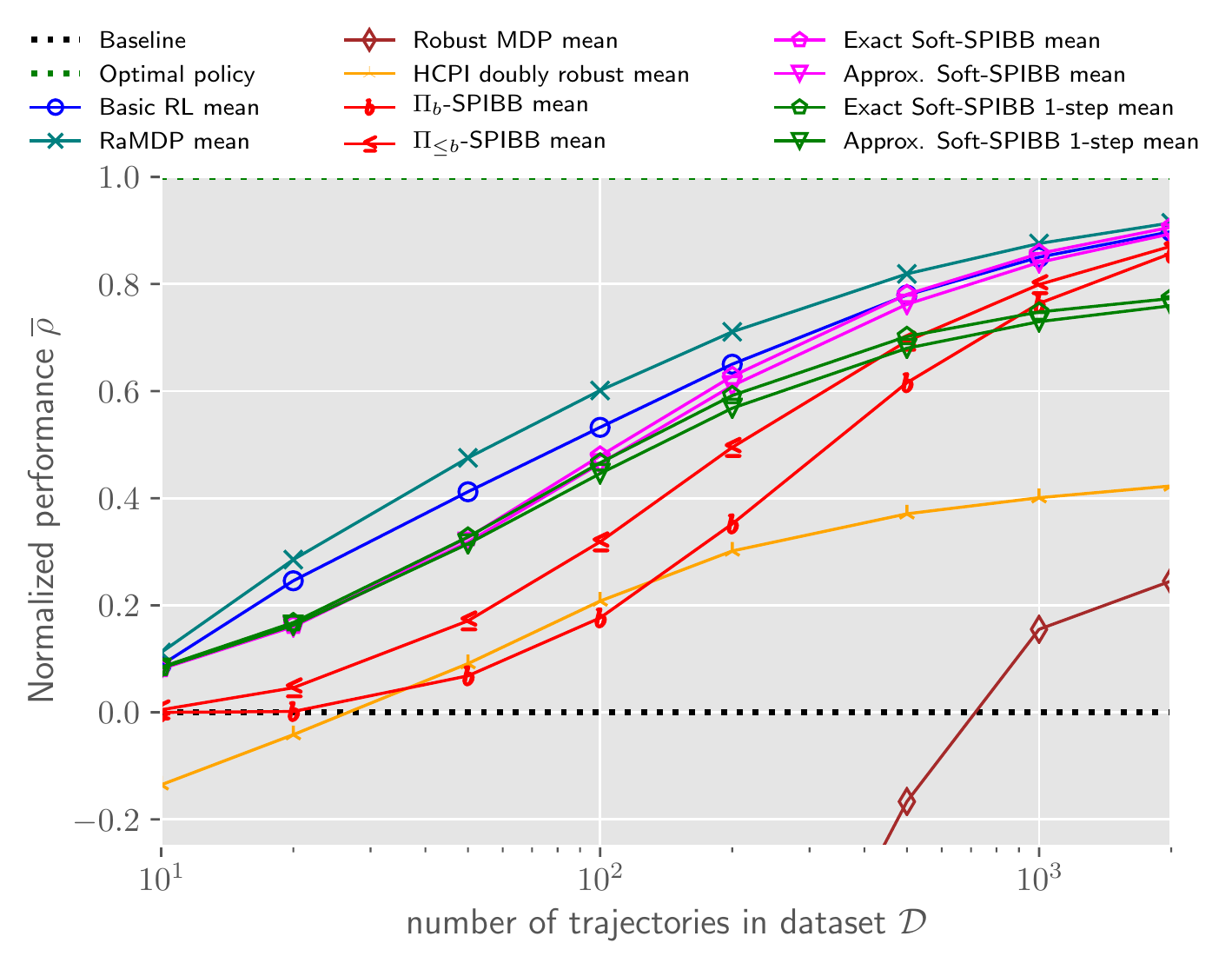}
        \label{subfig:benchmark_mean_eta=0.9_eps=2}
	}
	\subfloat[1\%-CVAR: $\eta = 0.9$, $\epsilon = 2$]{
		\includegraphics[trim = 5pt 5pt 5pt 5pt, clip, width=0.5\columnwidth]{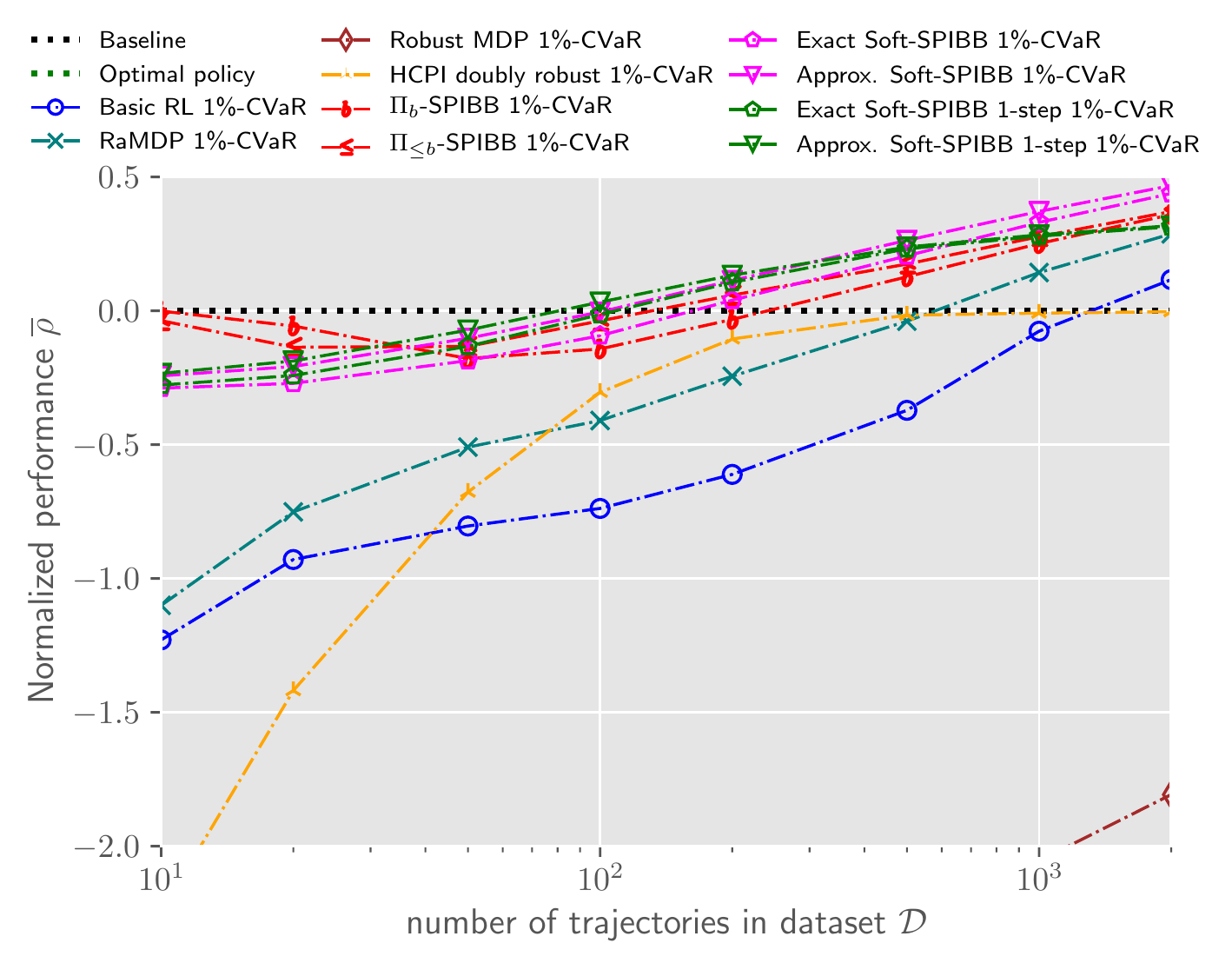}
        \label{subfig:benchmark_percent_eta=0.9_eps=2}
	}
    \caption{Benchmark on Random MDPs domain: mean and 1\%-CVAR performances for a hard scenario ($\eta = 0.9$) and Soft-SPIBB with $\epsilon = 2$}
    \label{fig:soft_benchmark}
\end{figure}
\paragraph{Benchmark results:} Figures~\ref{subfig:benchmark_mean_eta=0.9_eps=2} and \ref{subfig:benchmark_percent_eta=0.9_eps=2} respectively report the mean and 1\%-CVAR performances with a strong baseline ($\eta = 0.9$). Robust MDPs and HCPI perform poorly and are not further discussed. Basic RL and RaMDP win the benchmark in mean, but fail to do it safely, contrary to Soft-SPIBB and SPIBB algorithms. Exact-Soft-SPIBB is slightly better than Approx-Soft-SPIBB in mean, but also slightly worse in safety. Still in comparison to Approx-Soft-SPIBB, Exact-Soft-SPIBB's performance does not justify the computational complexity increase and will not be further discussed. Approx-Soft-SPIBB demonstrates a significant improvement over SPIBB methods, both in mean and in safety. Finally, the comparison of Approx-Soft-SPIBB with Approx-Soft-SPIBB 1-step shows that the safety is not improved in practice, and that the asymptotic optimality is compromised when the dataset becomes larger. 

\begin{figure}[t!]
	\centering
	\subfloat[0.1\%-CVaR: RaMDP]{
		\includegraphics[trim = 10pt 40pt 45pt 60pt, clip, width=0.5\columnwidth]{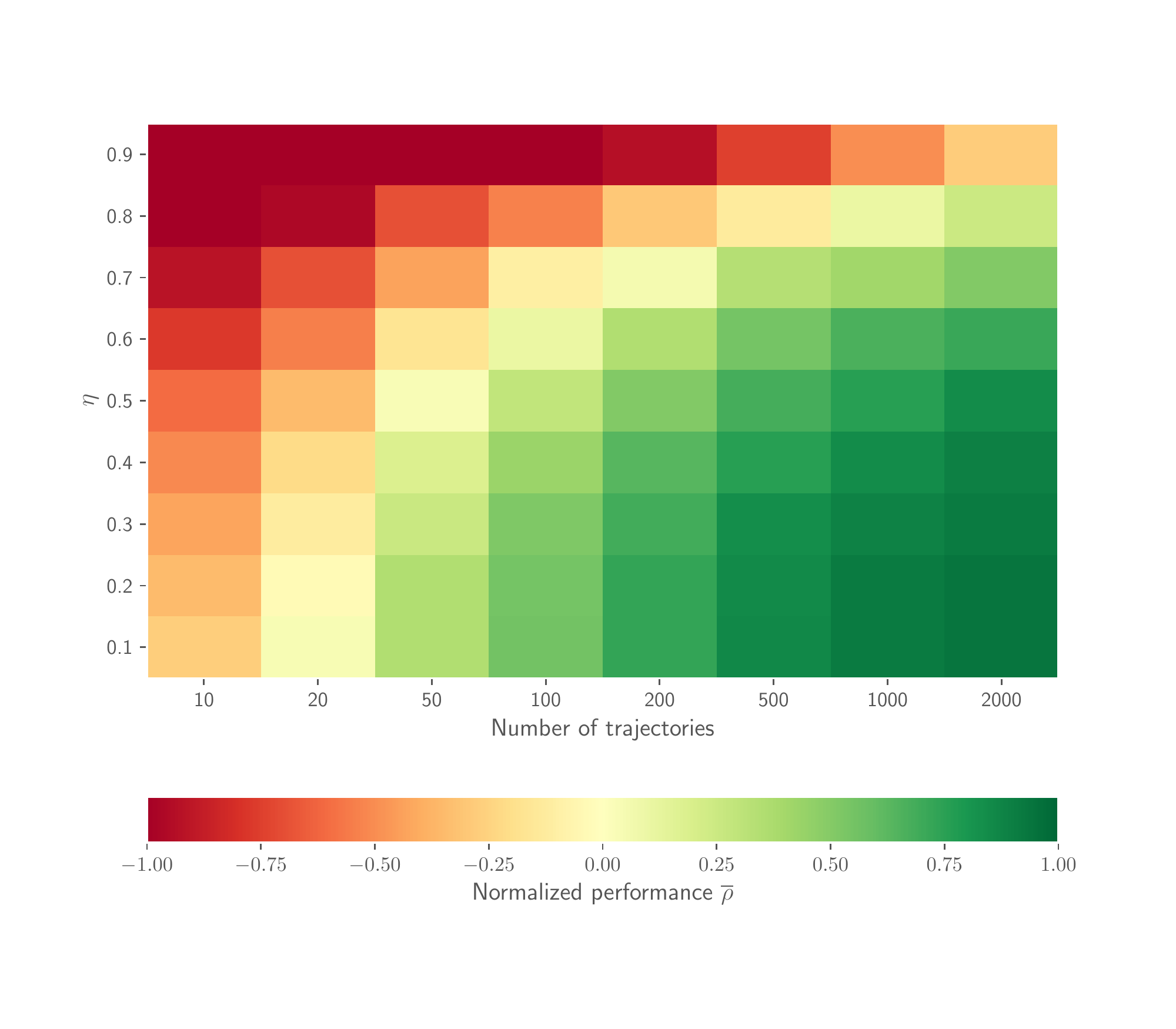}
        \label{subfig:influence_eta_ramdp}
	}
	\subfloat[0.1\%-CVaR: Approx-Soft-SPIBB, $\epsilon=2$]{
		\includegraphics[trim = 10pt 40pt 45pt 60pt, clip, width=0.5\columnwidth]{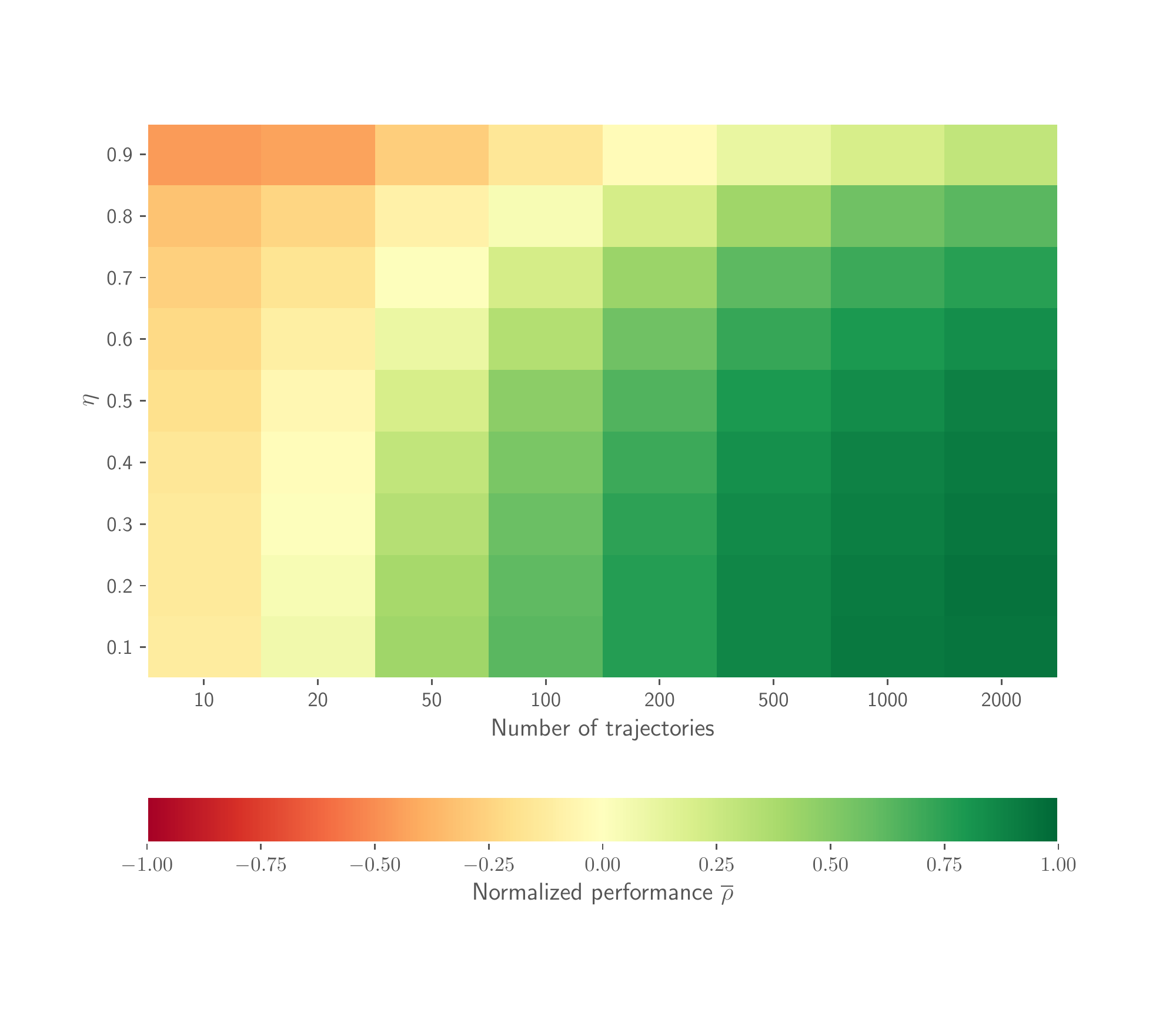}
        \label{subfig:influence_eta_approx_soft_spibb}
	} 
    \caption{Influence of $\eta$ on Random MDPs domain: 0.1\%-CVaR heatmaps as a function of $\eta$}
    \label{fig:influence_eta}
\end{figure}


\begin{figure}[b!]
	\centering
	\subfloat[1\%-CVaR: RaMDP, $\eta = 0.9$]{
		\includegraphics[trim = 10pt 40pt 45pt 60pt, clip, width=0.5\columnwidth]{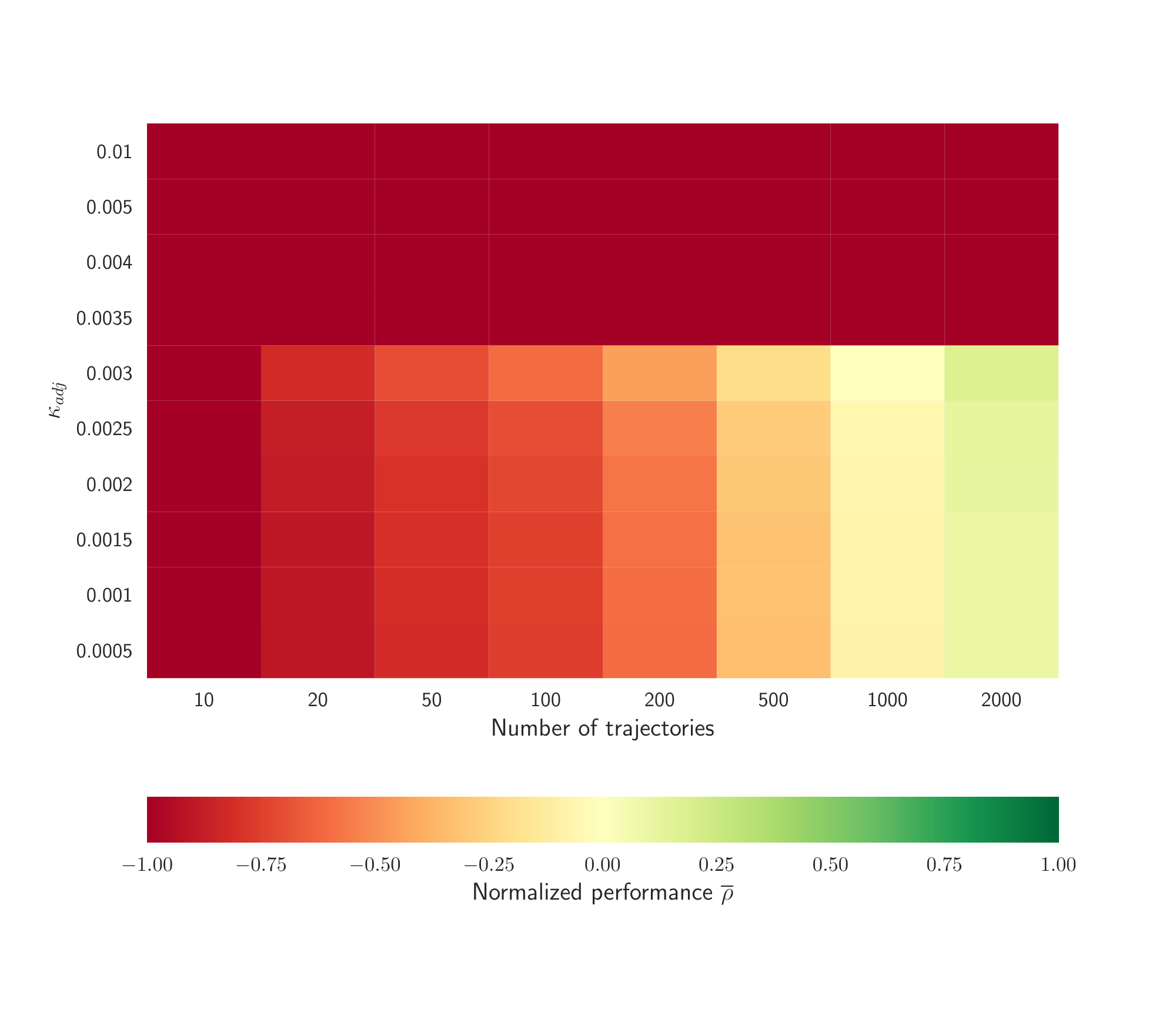}
        \label{subfig:influence_eps_exact_soft_spibb}
	}
	\subfloat[1\%-CVaR: Approx-Soft-SPIBB, $\eta = 0.9$]{
		\includegraphics[trim = 10pt 40pt 45pt 60pt, clip, width=0.5\columnwidth]{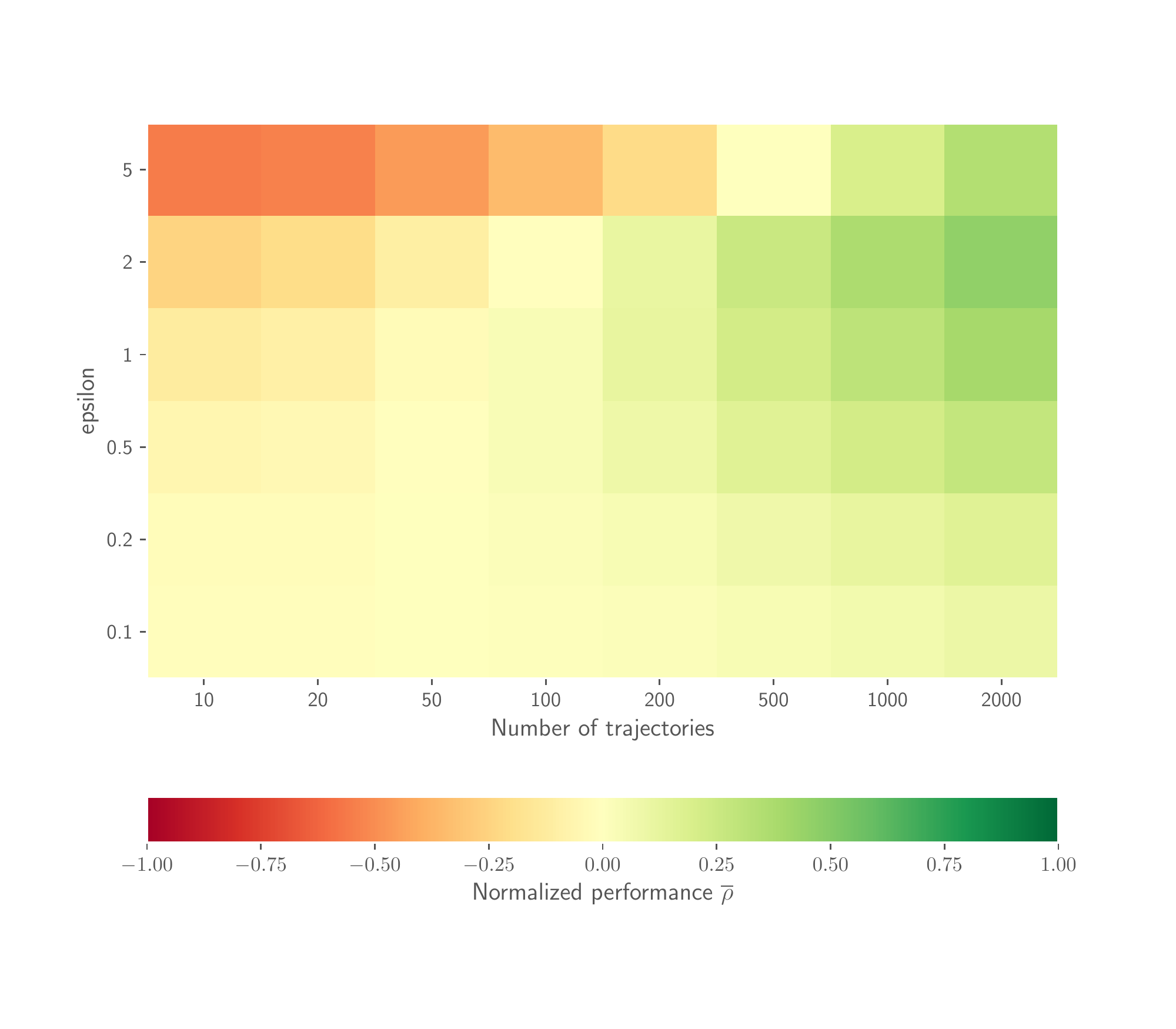}
        \label{subfig:influence_eps_approx_soft_spibb}
	}
    \caption{Sensitivity to hyperparameter on Random MDPs: 1\%-CVaR heatmaps for $\eta=0.9$}
    \label{fig:influence_eps}
\end{figure}


\paragraph{Sensitivity to the baseline:} We continue the analysis with a heatmap representation as a function of the strength of the baseline: Figures~\ref{subfig:influence_eta_ramdp} and \ref{subfig:influence_eta_approx_soft_spibb} display heatmaps of the 0.1\%-CVaR performance for RaMDP and Approx-Soft-SPIBB ($\epsilon = 2$) respectively. The colour of a cell indicates the improvement over the baseline normalized with respect to the optimal performance:
red, yellow, and green respectively mean below, equal to,
and above baseline performance. We observe that RaMDP is unsafe for strong baselines (high $\eta$ values) and small datasets, while Soft-SPIBB methods become slightly unsafe only with $\eta = 0.9$ and less than 20 trajectories, but are safe everywhere else.

\paragraph{Sensitivity to hyper-parameters:} We carry on with 1\%-CVaR performance heatmaps as a function of the hyper-parameters for RaMDP (Figure~\ref{subfig:influence_eps_exact_soft_spibb}) and Approx-Soft-SPIBB (Figure~\ref{subfig:influence_eps_approx_soft_spibb}) in the hardest scenario ($\eta = 0.9$). The choice of 1\%-CVaR instead of 0.1\%-CVaR is justified by the fact that the 0.1\%-CVaR RaMDP heatmap is almost completely red, which would not allow us to notice the interesting thresholding behaviour: when $\kappa_{adj}\geq 0.0035$, RaMDP becomes over-conservative to the point of not trying to reach the goal anymore. In contrast, Approx-Soft-SPIBB behaves more smoothly with respect to its hyper-parameter, its optimal value being in interval [0.5,2], depending on the safety/performance trade-off one wants to achieve. In the appendix, Section~\ref{sup:random_MDPs_full_results}, the interested reader may find the corresponding heatmaps for all Soft-SPIBB algorithms for mean and 1\%-CVaR performances. In particular, we may observe that, despite not having as strong theoretical guarantees as their 1-step versions, the Soft-SPIBB algorithms demonstrate similar CVaR performances.


\subsection{Helicopter domain}
\label{sec:spibb-dqn-exp}

To assess our algorithms on tasks with more complex state spaces, making the use of function approximation inevitable, we apply them to a helicopter navigation task (Figure \subref*{fig:helicopter}). The helicopter's start point is randomly picked in the teal region, its initial velocity is random as well. The agent can choose to apply or not a fixed amount of thrust forward and backward in the two dimensions, resulting in 9 actions total. An episode ends when the agent reaches the boundary of the blue box or has a speed larger than some maximal value. In the first case, it receives a reward based on its position with respect to the top right corner of the box (the exact reward value is chromatically indicated in the figure). In the second, it gets a reward of $-1$. The dynamics of the helicopter obey Newton's second law with an additive centered Gaussian noise applied to its position and velocity. We refer the reader to the appendix, Section \ref{sup:dummy-parameters} for the detailed specifications. We generated a baseline by training online a DQN \citep{Mnih2015} and applying a softmax on the learnt $Q$-network. During training, a discount factor of $0.9$ is used, but the reported results show the undiscounted return obtained by the agent. 

\begin{figure}[t]
	\subfloat[Helicopter benchmark with $|\mathcal{D}|=10,000$]{
        \includegraphics[trim = 5pt 5pt 15pt 5pt, clip, width=0.92\columnwidth,left]{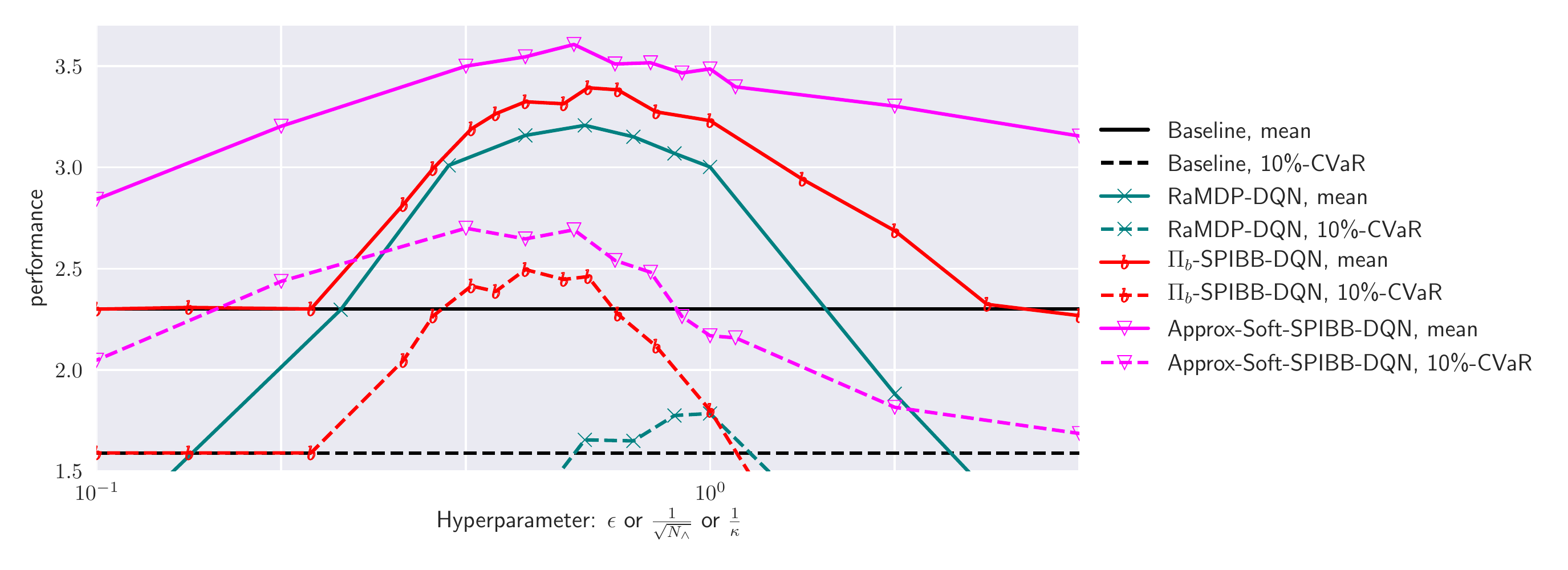}
        \label{fig:soft_spibb_dqn_10k}
	}\\
    \centering
	\subfloat[Helicopter benchmark with $|\mathcal{D}|=3,000$]{
        \includegraphics[trim = 5pt 5pt 240pt 5pt, clip, width=0.65\columnwidth]{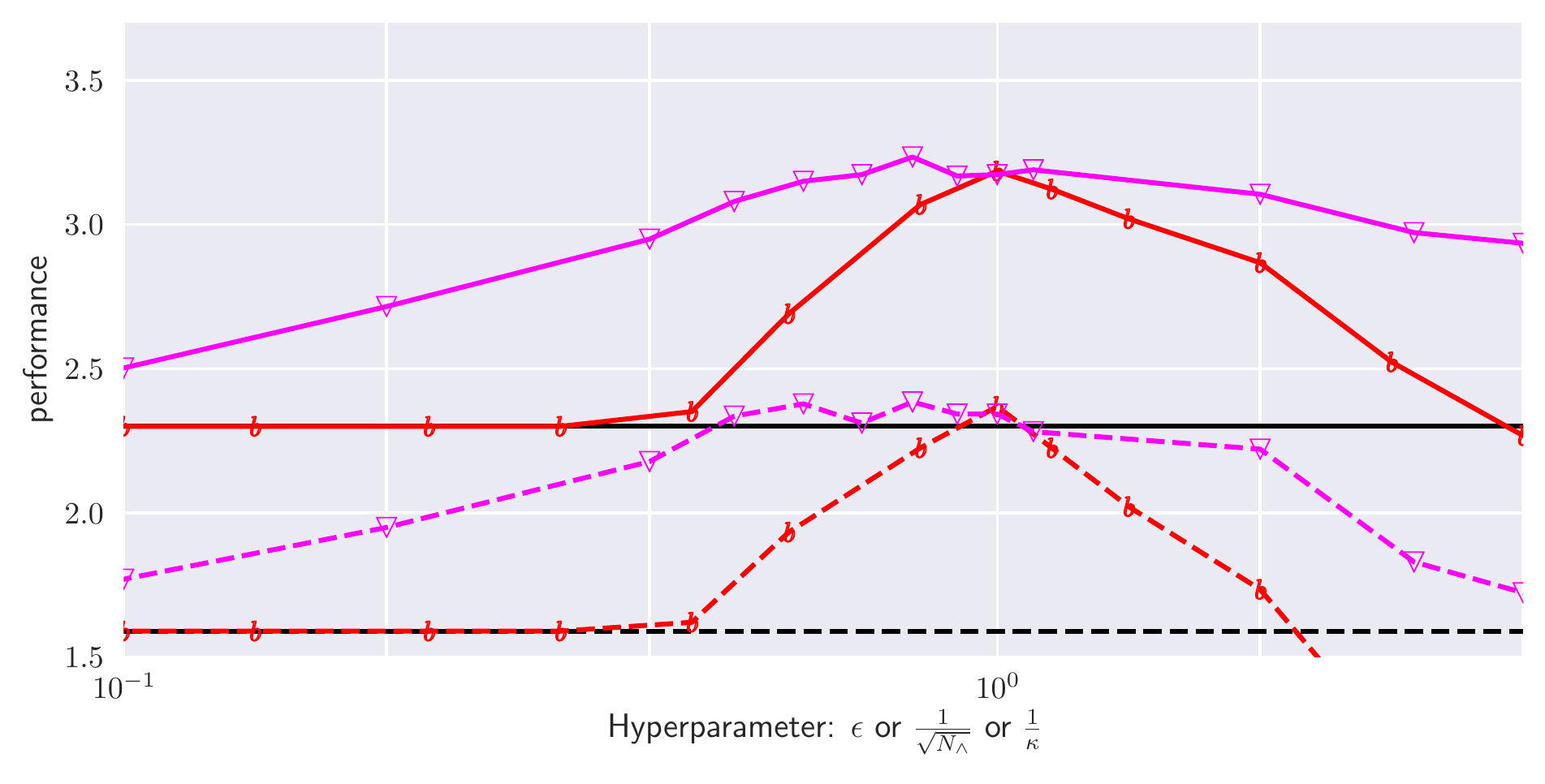}
        \label{fig:soft_spibb_dqn_3k}
	}
	\subfloat[Helicopter domain]{
		\includegraphics[trim = 5pt 5pt 5pt 5pt, clip, width=0.35\columnwidth]{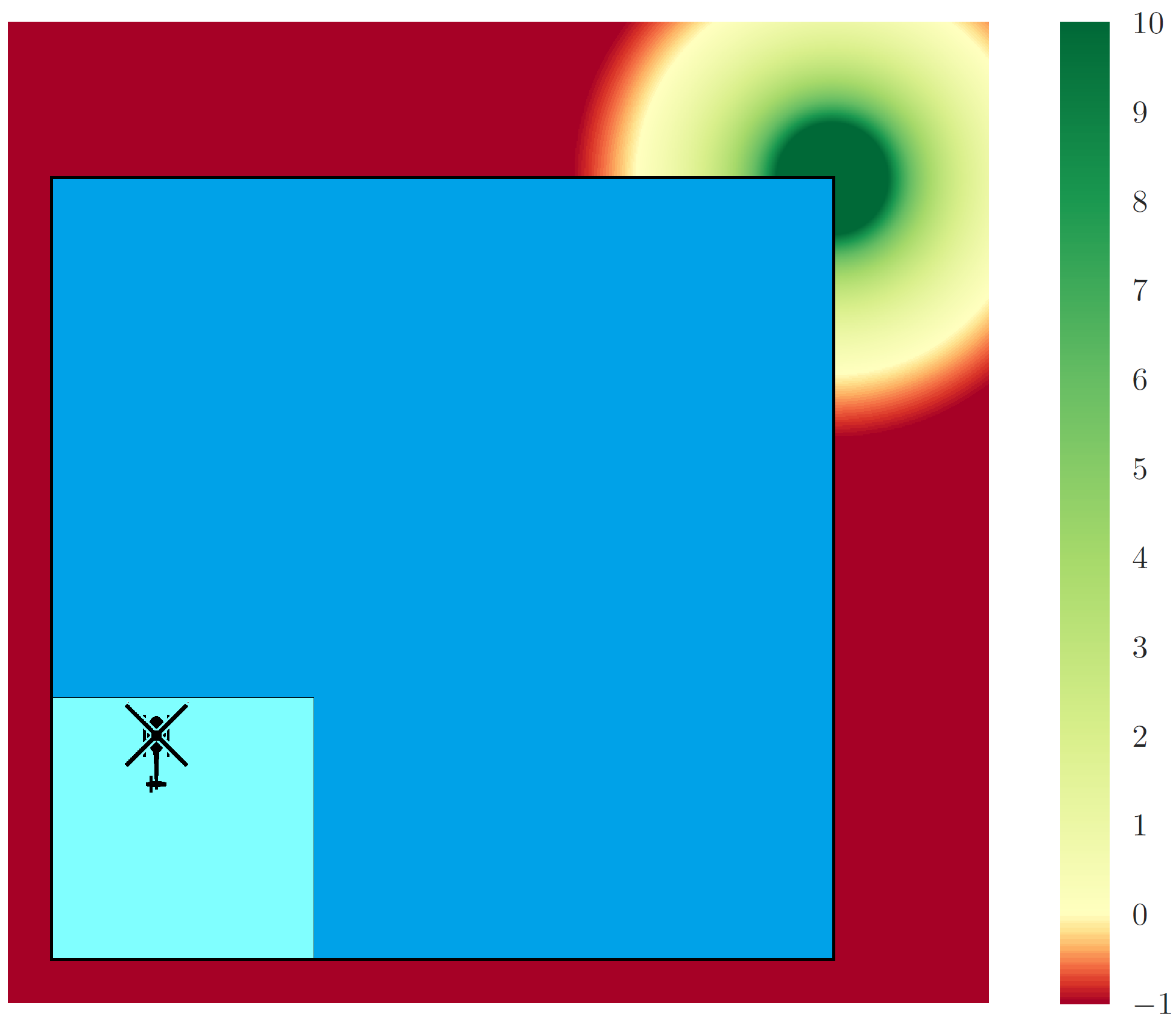}
        \label{fig:helicopter}
	}
    \caption{Helicopter: mean and 10\%-CVaR as a function of the hyper-parameter value}
    \label{fig:Helicopter_experiments}
\end{figure}

The experiments consist in 300 training runs (necessary to obtain reasonable estimates of algorithms' safety, the full training procedure is described in the appendix, Section \ref{sup:helicopter_training}) of RaMDP-DQN, SPIBB-DQN and Approx-Soft-SPIBB-DQN, for different values of their hyper-parameters (resp. $\kappa$, $N_{\wedge}$ and $\epsilon$). We note that for $\kappa = 0$, $N_{\wedge} = 0$ or $\epsilon = +\infty$, those three algorithms become standard DQN, and that for $N_{\wedge} = \infty$ or $\epsilon = 0$, the SPIBB and Soft-SPIBB algorithms produce a copy of the baseline. The three algorithms rely on some estimates of the state-action counts. In this work, we used a pseudo-count estimate heuristic based on Euclidean distance, also detailed in Section \ref{sup:helicopter_training}. For scalability, we may consider several pseudo-count methodologies from the literature~\cite{Bellemare2016,Fox2018}. This is left for future work.

The results of our evaluation can be found in Figure \ref{fig:Helicopter_experiments}, where we plot the mean and 10\%-CVaR performances of the different algorithms for two sizes of datasets (more results may be found in the appendix, Section \ref{sup:helicopter_more_results}). In order to provide meaningful comparisons, the abscissa represents the different hyper-parameters transformed to account for their dimensional homogeneity (except for a scaling factor). Both Approx-Soft-SPIBB-DQN and SPIBB-DQN outperform RaMDP-DQN by a large margin on the datasets of size 10,000. On the smaller datasets, RaMDP-DQN performs very poorly and does not even appear on the graph. For the same reason, vanilla DQN (mean: 0.22 and 10\%-CVaR: -1 with $|\mathcal{D}|=10,000$) does not appear on any of the graphs. The two SPIBB algorithms significantly outperform the baseline both in mean and 10\%-CVaR. At their best hyper-parameter value, their 10\%-CVaR is actually better than the mean performance of the baseline. Approx-Soft-SPIBB-DQN performs better than SPIBB-DQN both in mean and 10\%-CVaR performances. Finally, it is less sensitive than SPIBB-DQN with respect to their respective hyperparameters, and demonstrates a better stability over different dataset sizes. That stability is a useful property as it reduces the requirement for hyper-parameter optimization, which is crucial for Batch RL.

\section{Conclusion}
\label{sec:conclusion}
We study the problem of safe policy improvement in a Batch RL setting. Building on the SPIBB methodology, we relax the constraints of the policy search to propose a family of algorithms coined Soft-SPIBB. We provide proofs of safety and of computational efficiency for an algorithm called Approx-Soft-SPIBB based on the search of an approximate solution that does not compromise the safety guarantees. We support the theoretical work with an extensive empirical analysis where Approx-Soft-SPIBB shines as the best compromise average performance vs. safety. We further develop Soft-SPIBB in a model-free manner which helps its application to function approximation. Despite the lack of theoretical safety guarantees with function approximation, we observe in our experiments where the function approximation is modelled as a neural network, that Soft-SPIBB allows safe policy improvement in practice and significantly outperforms the competing algorithms both in safety and in performance.


\newpage
\appendix



\section{Proofs}
\label{sup:proofs}

\setcounter{theorem}{0}
\setcounter{constraint}{0}
\setcounter{remark}{0}
\setcounter{definition}{0}
\setcounter{corollary}{0}
\setcounter{lemma}{0}
\setcounter{assumption}{0}

\subsection{Preliminaries}

We start by recalling the various definitions used in the proofs. 

\begin{definition}
    A policy $\pi$ is said to be a policy improvement over a baseline policy $\pi_b$ in an MDP $M = \langle \mathcal{X}, \mathcal{A}, P, R, \gamma\rangle$ if the following inequality holds in every state $x\in\mathcal{X}$:
    \begin{align}
       V^{\pi}_{M}(x) \geq V^{\pi_b}_{M}(x)
    \end{align}
\end{definition}

\begin{definition}
        A policy $\pi$ is said to be $\pi_b$-advantageous in an MDP $M = \langle \mathcal{X}, \mathcal{A}, P, R, \gamma\rangle$ if the following inequality holds in every state $x\in\mathcal{X}$:
        \begin{align}
           \sum_{a \in \mathcal{A}} A^{\pi_b}_{M}(x,a)\pi(a|x) \geq 0
        \end{align}
        where $A^{\pi_b}_{M}(x,a) = Q^{\pi_b}_{M}(x,a) - V^{\pi_b}_{M}(x)$ quantifies the advantage (or disadvantage) of taking action $a$ in state $x$ instead of following policy $\pi_b$.
\end{definition}
    
    \begin{definition}
        A policy $\pi$ is said to be $(\pi_b, e,\epsilon)$-constrained for baseline policy $\pi_b$, error function $e$, and a hyper-parameter $\epsilon$ if, for all states $x\in\mathcal{X}$, the following inequality holds:
        \begin{equation*}
            \sum_{a \in \mathcal{A}} e(x,a) \big\lvert\pi(a|x)-\pi_b(a|x)\big\rvert \leq \epsilon.
        \end{equation*}
    \end{definition}
    
\subsection{Error Bounds} 
\label{sup:error_bounds}

The difference between an estimated parameter and the true one can be bounded using concentration bounds (or equivalently, Hoeffding's inequality) applied to the state-action counts in $\mathcal{D}$~\citep{Petrik2016,Laroche2019}. Specifically, the following inequalities hold with probability at least $1 - \delta$ for any state-action pair $(x,a) \in \mathcal{X} \times \mathcal{A}$:
\begin{align}
    ||P^*(\cdot|x,a)-\widehat{P}(\cdot|x,a)||_1 &\leq e_P(x,a)  \label{eq:bound-P-2} \\
    |R^*(x,a)-\widehat{R}(x,a)| &\leq e_P(x,a)R_{max} \label{eq:bound-R} \\
    \big\lvert Q^{\pi_b}_{M^*}(x,a) - Q^{\pi_b}_{\widehat{M}}(x,a) \big\rvert &\leq e_Q(x,a) V_{max} \label{eq:bound-Q-2}
\end{align}
where: 
\begin{align}
    e_P(x,a) &:= \sqrt{\cfrac{2}{N_{\mathcal{D}}(x,a)}\log\cfrac{2|\mathcal{X}||\mathcal{A}|2^{|\mathcal{X}|}}{\delta}} \label{eq:error-function-P-2} \\
    e_Q(x,a) &:= \sqrt{\cfrac{2}{N_{\mathcal{D}}(x,a)}\log\cfrac{2|\mathcal{X}||\mathcal{A}|}{\delta}}. \label{eq:error-function-Q-2}
\end{align}

\begin{proof}
    The three inequalities can be proved similarly to \citep[Proposition 9]{Petrik2016}. We only detail the proof for \eqref{eq:bound-Q-2}: for any $(x,a) \in \mathcal{X} \times \mathcal{A}$, and from the two-sided Hoeffding's inequality, 
    \begin{align*}
        &\mathbb{P} \left( \big\lvert Q^{\pi_b}_{M^*}(x,a) - Q^{\pi_b}_{\widehat{M}}(x,a) \big\rvert > e_Q(x,a) V_{max} \right) \\
        &= \mathbb{P} \left( \frac{\big\lvert Q^{\pi_b}_{M^*}(x,a) - Q^{\pi_b}_{\widehat{M}}(x,a) \big\rvert}{2 V_{max}} > \sqrt{\frac{1}{2 N_\mathcal{D}(x,a)} \log \frac{2 |\mathcal{X}||\mathcal{A}|}{\delta}} \right) \\
        &\leq 2 \exp \left( -2 N_\mathcal{D}(x,a) \frac{1}{2 N_\mathcal{D}(x,a)} \log \frac{2 |\mathcal{X}||\mathcal{A}|}{\delta}  \right) \\
        &\leq \frac{\delta}{| \mathcal{X} | | \mathcal{A} | }
    \end{align*}
    
    By summing all $|\mathcal{X}| |\mathcal{A}|$ state-action pairs error probabilities lower than $\frac{\delta}{| \mathcal{X} | | \mathcal{A} | }$, we obtain \eqref{eq:bound-Q-2}.
    
\end{proof}

\subsection{Proof of Theorem \ref{th:1-step}}

In this section, we prove Theorem \ref{th:1-step}:
\begin{theorem}
    Any $(\pi_b, e_Q,\epsilon)$-constrained policy $\pi$ that is $\pi_b$-advantageous in $\widehat{M}$ satisfies the following inequality in every state $x$ with probability at least $1 - \delta$:
    \begin{align}
        V^\pi_{M^*}(x)-V^{\pi_b}_{M^*}(x) &\geq -\cfrac{\epsilon V_{max}}{1-\gamma}.
    \end{align}
\end{theorem}

    \begin{proof}
        We start from Proposition \ref{prop:Q-dpi-d} with $M=M^*$, $\pi_1=\pi$, and $\pi_2=\pi_b$:
    \begin{align}
        V^{\pi}_{M^*}-V^{\pi_b}_{M^*} &= Q^{\pi_b}_{M^*}\left(\pi-\pi_b\right)d_{M^*}^{\pi} \\
        &= \left(Q^{\pi_b}_{M^*} - Q^{\pi_b}_{\widehat{M}} + Q^{\pi_b}_{\widehat{M}}\right)\left(\pi-\pi_b\right)d_{M^*}^{\pi} \\
        &= \left(Q^{\pi_b}_{M^*} - Q^{\pi_b}_{\widehat{M}}\right)\left(\pi-\pi_b\right)d_{M^*}^{\pi}
            + Q^{\pi_b}_{\widehat{M}}\left(\pi-\pi_b\right)d_{M^*}^{\pi}. \label{eq:30}
    \end{align}
    
    The first term is bounded by $\frac{\epsilon V_{max}}{1-\gamma}$ thanks to inequality \ref{eq:error-function-Q-2}, the $(\pi_b, e_Q,\epsilon)$-constrained policy property and Holder's inequality:
    \begin{align}
        \Big\lVert&\left(Q^{\pi_b}_{M^*} - Q^{\pi_b}_{\widehat{M}}\right)\left(\pi-\pi_b\right)d_{M^*}^{\pi}\Big\rVert_\infty \leq \Big\lVert\left(Q^{\pi_b}_{M^*} - Q^{\pi_b}_{\widehat{M}}\right)\left(\pi-\pi_b\right)\Big\rVert_\infty\big\lVert d_{M^*}^{\pi}\big\rVert_1 \\
        &\leq \max_{x} \sum_{a} (Q^{\pi_b}_{M^*}(x,a) - Q^{\pi_b}_{\widehat{M}}(x,a)) \left(\pi(a|x) -\pi_b(a|x)\right) \cfrac{1}{1-\gamma} \\
        &\leq \cfrac{\epsilon V_{max}}{1-\gamma}.
    \end{align}

    Let us now prove that the second term of Equation \ref{eq:30} is positive. It is the product of vector $Q^{\pi_b}_{\widehat{M}}\left(\pi-\pi_b\right)$ with matrix $d_{M^*}^{\pi}$. All the terms of $d_{M^*}^{\pi}$ are positive so it suffices to show that each element of the vector is positive. We have for each $x\in\mathcal{X}$:
    \begin{align}
        \left(Q^{\pi_b}_{\widehat{M}}\left(\pi-\pi_b\right)\right)(x) &= \sum_a Q^{\pi_b}_{\widehat{M}}(x,a)\left(\pi(a|x)-\pi_b(a|x)\right) \\
        &= \sum_a Q^{\pi_b}_{\widehat{M}}(x,a)\pi(a|x) - V^{\pi_b}_{\widehat{M}}(x) \\
        &= \sum_{a \in \mathcal{A}} A^{\pi_b}_{M}(x,a)\pi(a|x) \geq 0
    \end{align}
    where the last inequality comes from $\pi$ being $\pi_b$-advantageous in $\widehat{M}$. This concludes the proof.
    \qed\end{proof}

\subsection{Proof of Theorem \ref{thm:SSPIBB}}

In this section, we prove Theorem \ref{thm:SSPIBB} from the main text. Let us start with:
    \begin{proposition}
    Let $M = \langle \mathcal{X}, \mathcal{A}, P, R, \gamma\rangle$ be an MDP, and $\pi_1$ and $\pi_2$ be two policies defined on this MDP. We have:
    \begin{align}
        V^{\pi_1}_{M}-V^{\pi_2}_{M} = Q^{\pi_1}_{M}\left(\pi_1-\pi_2\right)d_M^{\pi_2} = Q^{\pi_2}_{M}\left(\pi_1-\pi_2\right)d_M^{\pi_1},
    \end{align}
    where $d_M^{\pi}$ is a matrix that assigns to element $d_M^{\pi}(x'|x)$ the expectation of the discounted sum of visits to state $x'$, starting from $x$, and when following policy $\pi$ in MDP $M$ (it is actually only dependent on the transition dynamics $P$). Formally, it is written as:
    \begin{equation}
        d_M^{\pi}(x'|x) = \sum_{t=0}^{\infty} \gamma^t \mathbb{P}(X_t=x'|X_t\sim P\pi X_{t-1}, X_0=x)
    \end{equation}
    \label{prop:Q-dpi-d}
    \end{proposition}
    \begin{proof}
    \begin{align}
        V^{\pi_1}_{M}-V^{\pi_2}_{M} &= (R+\gamma V^{\pi_1}_{M} P)\pi_1 - (R+\gamma V^{\pi_2}_{M} P)\pi_2 \\
        &= (R+\gamma V^{\pi_1}_{M} P)(\pi_1-\pi_2) + \gamma (V^{\pi_1}_{M} - V^{\pi_2}_{M}) P\pi_2 \\
        &= Q^{\pi_1}_{M}(\pi_1-\pi_2) + \gamma (V^{\pi_1}_{M} - V^{\pi_2}_{M}) P\pi_2 \\
        \left(V^{\pi_1}_{M}-V^{\pi_2}_{M}\right)\left(\mathbb{I} - \gamma P\pi_2\right)&= Q^{\pi_1}_{M}(\pi_1-\pi_2) \\
        V^{\pi_1}_{M}-V^{\pi_2}_{M}&= Q^{\pi_1}_{M}(\pi_1-\pi_2)\left(\mathbb{I} - \gamma P\pi_2\right)^{-1}
    \end{align}
    The last line is obtained because $P\pi_2$ is a stochastic matrix and $0\leq\gamma<1$. As a result, $\mathbb{I} - \gamma P\pi_2$ is invertible and is indeed the state distribution matrix $d_{M}^{\pi_2}$. The other equality: $V^{\pi_1}_{M}-V^{\pi_2}_{M}= Q^{\pi_2}_{M}(\pi_1-\pi_2)d_{M}^{\pi_1}$ is obtained in a symmetrical way.
    \qed\end{proof}

    \begin{proposition}
    Let $M_1$ and $M_2$ be two MDPs: $M_1 = \langle \mathcal{X}, \mathcal{A}, P_1, R_1, \gamma\rangle$ and $M_2 = \langle \mathcal{X}, \mathcal{A}, P_2, R_2, \gamma\rangle$. Let $\pi_1, \pi_2$ be two policies defined on these MDPs, then:
    \begin{align}
        V^{\pi_1}_{M_1}-V^{\pi_2}_{M_1} = V^{\pi_1}_{M_2}-V^{\pi_2}_{M_2} &+ Q^{\pi}_{M_2}\left(\pi_1-\pi_2\right)\left(d_{M_1}^{\pi_2} - d_{M_2}^{\pi_2}\right) \\ 
        &+ \left(Q^{\pi}_{M_1} - Q^{\pi}_{M_2}\right)\left(\pi_1-\pi_2\right)d_{M_1}^{\pi_2}
    \end{align}
    \label{prop:dV*-d^V}
    \end{proposition}
    \begin{proof}
    From Proposition \ref{prop:Q-dpi-d}:
    \begin{align}
        V^{\pi_1}_{M_1}-V^{\pi_2}_{M_1} &= Q^{\pi_1}_{M_1}\left(\pi_1-\pi_2\right)d_{M_1}^{\pi_2}\label{eq:48}\\
         &= \left(Q^{\pi_1}_{M_2} + Q^{\pi_1}_{M_1} - Q^{\pi_1}_{M_2}\right)\left(\pi_1-\pi_2\right)d_{M_1}^{\pi_2} \\
         &= Q^{\pi_1}_{M_2}\left(\pi_1-\pi_2\right)d_{M_1}^{\pi_2} + \left(Q^{\pi_1}_{M_1} - Q^{\pi_1}_{M_2}\right)\left(\pi_1-\pi_2\right)d_{M_1}^{\pi_2} \\
         &= Q^{\pi_1}_{M_2}\left(\pi_1-\pi_2\right)\left(d_{M_2}^{\pi_2} + d_{M_1}^{\pi_2} - d_{M_2}^{\pi_2}\right) + \left(Q^{\pi_1}_{M_1} - Q^{\pi_1}_{M_2}\right)\left(\pi_1-\pi_2\right)d_{M_1}^{\pi_2} \\
         &= Q^{\pi_1}_{M_2}\left(\pi_1-\pi_2\right)d_{M_2}^{\pi_2} + Q^{\pi_1}_{M_2}\left(\pi_1-\pi_2\right)\left(d_{M_1}^{\pi_2} - d_{M_2}^{\pi_2}\right)\\
         &\quad\quad\quad\quad\quad\quad\quad\quad\quad\;\: + \left(Q^{\pi_1}_{M_1} - Q^{\pi_1}_{M_2}\right)\left(\pi_1-\pi_2\right)d_{M_1}^{\pi_2} \label{eq:37}
    \end{align}
    
    The first term of Equation \ref{eq:37} is transformed by applying Proposition \ref{prop:Q-dpi-d} to $M_2$:
    \begin{equation}
        Q^{\pi_1}_{M_2}\left(\pi_1-\pi_2\right)d_{M_2}^{\pi_2} = V^{\pi_1}_{M_2}-V^{\pi_2}_{M_2}, \label{eq:54}
    \end{equation}
    which concludes the proof.
    \qed\end{proof}

The proof of the theorem requires the following assumption (see the main text for a discussion on it):
    \begin{assumption}\label{ass:kappa_supp}
        There exists a constant $\kappa < \frac{1}{\gamma}$ such that, for all state-action pairs $(x,a)\in\mathcal{X}\times\mathcal{A}$ the following inequality holds:
        \begin{align}
            \sum_{x',a'} e_P(x',a') \pi_b(a'|x') P^*(x'|x,a) \leq \kappa e(x,a)
        \end{align}
    \end{assumption}

That assumption allows to prove Lemma \ref{lem:induction}: 
    \begin{lemma}  Under Assumption \ref{ass:kappa}, any $(\pi_b, e_P,\epsilon)$-constrained policy $\pi$ satisfies the following inequality for every state-action pair $(x,a)$ with probability at least $1 - \delta$:
        \begin{align}
            \big\lvert Q^{\pi}_{M^*}(x,a) - Q^{\pi}_{\widehat{M}}(x,a)\big\rvert \label{eq:lem_1} \leq \left(\cfrac{e_P(x,a)}{1-\kappa\gamma} + \cfrac{\gamma\epsilon}{\left(1-\gamma\right)\left(1-\kappa\gamma\right)}\right)V_{max}. \nonumber
        \end{align}
    \end{lemma}
    \begin{proofsketch} Let $Q^*_0 = 0$, $\widehat{Q}_0 = 0$, $Q^*_{t+1} = \mathcal{B}^*(Q^*_{t})$ and $\widehat{Q}_{t+1} = \widehat{\mathcal{B}}(\widehat{Q}_{t})$, where $\mathcal{B}^*$ and $\widehat{\mathcal{B}}$ denote the Bellman operators for policy $\pi$ in the true and estimated MDP. We proceed by induction to prove inequality \ref{eq:lem_1} on $\lvert Q^*_{t}(x,a) - \widehat{Q}_{t}(x,a)\lvert$:
    \begin{align}
        Q^*_{t+1}(x,a) - \widehat{Q}_{t+1}(x,a) &= R^*(x,a) - \widehat{R}(x,a) \\
        &+ \sum_{x',a'}\gamma \left(Q^*_{t}(x',a')-\widehat{Q}_{t}(x',a')\right) \pi(a'|x')P^*(x'|x,a)\nonumber\\
        &+ \sum_{x',a'}\gamma \widehat{Q}_{t}(x',a')\pi(a'|x')\left(P^*(x'|x,a) - \widehat{P}(x'|x,a)\right). \nonumber
    \end{align}
    
    By applying the  inductive hypothesis and inequality \ref{eq:bound-P} (which holds true with probability at least $1 - \delta$):
    \begin{align}
         \big\lvert Q^*_{t+1}(x,a) - \widehat{Q}_{t+1}(x,a)\big\rvert &\leq e_P(x,a) V_{max} + \cfrac{\gamma^2 \epsilon V_{max}}{\left(1-\gamma\right)\left(1-\kappa\gamma\right)} \nonumber \\
         &+ \sum_{x',a'} \cfrac{\gamma e_P(x',a')}{1-\kappa\gamma} \pi_b(a'|x') P^*(x'|x,a) V_{max} \\
         &+ \sum_{x',a'} \cfrac{\gamma e_P(x',a')}{1-\kappa\gamma} \lvert\pi(a'|x') - \pi_b(a'|x') \rvert P^*(x'|x,a) V_{max}. \nonumber
    \end{align}
    
    The $(\pi_b, e_P,\epsilon)$-constrained property of $\pi$ ensures that 
    $$\forall x'\in\mathcal{X},\  {\sum_{a'} e_P(x',a')\lvert\pi(a'|x') - \pi_b(a'|x')\rvert\leq\epsilon}$$ Applying Assumption \ref{ass:kappa} and some algebraic manipulations allows to conclude the induction. Taking the limit proves the lemma since $(Q^*_{t},\widehat{Q}_{t}) \rightarrow (Q^{\pi}_{M^*}, Q^{\pi}_{\widehat{M}})$.
    \qed\end{proofsketch}
    \begin{proof} We let $\mathcal{B}^*$ and $\widehat{\mathcal{B}}$ denote the Bellman operators for policy $\pi$ in the true and estimated MDP. We have \textit{e.g.}: 
    \begin{align}
        \mathcal{B}^*(Q)(x,a) &= R^*(x,a) + \sum_{x',a'}\gamma Q(x',a')\pi(a'|x')P^*(x'|x,a). \nonumber
    \end{align}
    We now proceed by induction. Let $Q^*_0 = 0$, $\widehat{Q}_0 = 0$, $Q^*_{t+1} = \mathcal{B}^*(Q^*_{t})$ and $\widehat{Q}_{t+1} = \widehat{\mathcal{B}}(\widehat{Q}_{t})$. We wish to prove that :
    \begin{align}
        \lvert Q^*_{t}(x,a) - \widehat{Q}_{t}(x,a)\rvert \leq \left(\cfrac{e_P(x,a)}{1-\kappa\gamma} + \cfrac{\gamma\epsilon}{\left(1-\gamma\right)\left(1-\kappa\gamma\right)}\right)V_{max}. \nonumber
    \end{align}    
    By definition we have:
    \begin{align}
        Q^*_{t+1}(x,a) - \widehat{Q}_{t+1}(x,a) &= \mathcal{B}^*(Q^*_{t})(x,a) - \widehat{\mathcal{B}}(\widehat{Q}_{t})(x,a) \\
        &= R^*(x,a) - \widehat{R}(x,a) + \sum_{x',a'}\gamma Q^*_{t}(x',a')\pi(a'|x')P^*(x'|x,a) \nonumber\\
        & \quad\quad\quad\quad\quad - \sum_{x',a'}\gamma \widehat{Q}_{t}(x',a')\pi(a'|x')\widehat{P}(x'|x,a)\\
        &= R^*(x,a) - \widehat{R}(x,a) \nonumber \\
        & \quad\quad + \sum_{x',a'}\gamma \left(Q^*_{t}(x',a')-\widehat{Q}_{t}(x',a')\right) \pi(a'|x')P^*(x'|x,a)\nonumber\\
        & \quad\quad + \sum_{x',a'}\gamma \widehat{Q}_{t}(x',a')\pi(a'|x')\left(P^*(x'|x,a) - \widehat{P}(x'|x,a)\right)
    \end{align}
    We apply inequality \ref{eq:bound-R} to the first term, the inductive hypothesis at time $t$ to the second and inequality \ref{eq:bound-P-2} and Holder's inequality to the third to get:
    \begin{align}
         \lvert Q^*_{t+1}&(x,a) - \widehat{Q}_{t+1}(x,a) \rvert \nonumber\\
         & \leq e_P(x,a) R_{max} + \gamma e_P(x,a) V_{max} \nonumber \\
         & \quad + \sum_{x',a'}\gamma \left(\cfrac{e_P(x',a')}{1-\kappa\gamma} + \cfrac{\gamma\epsilon}{\left(1-\gamma\right)\left(1-\kappa\gamma\right)}\right) \pi(a'|x')P^*(x'|x,a) V_{max} \\
         &= e_P(x,a) V_{max} + \cfrac{\gamma^2 \epsilon V_{max}}{\left(1-\gamma\right)\left(1-\kappa\gamma\right)} \nonumber \\
         & \quad + \sum_{x',a'} \cfrac{\gamma e_P(x',a')}{1-\kappa\gamma} \pi(a'|x')P^*(x'|x,a) V_{max}\\
         &= e_P(x,a) V_{max} + \cfrac{\gamma^2 \epsilon V_{max}}{\left(1-\gamma\right)\left(1-\kappa\gamma\right)} \nonumber \\
         & \quad + \sum_{x',a'} \cfrac{\gamma e_P(x',a')}{1-\kappa\gamma} \left(\pi(a'|x') - \pi_b(a'|x') + \pi_b(a'|x')\right) P^*(x'|x,a) V_{max}\\
         &\leq e_P(x,a) V_{max} + \cfrac{\gamma^2 \epsilon V_{max}}{\left(1-\gamma\right)\left(1-\kappa\gamma\right)} \nonumber \\
         & \quad + \sum_{x',a'} \cfrac{\gamma e_P(x',a')}{1-\kappa\gamma} \pi_b(a'|x') P^*(x'|x,a) V_{max} \nonumber\\
         & \quad + \sum_{x',a'} \cfrac{\gamma e_P(x',a')}{1-\kappa\gamma} \lvert\pi(a'|x') - \pi_b(a'|x') \rvert P^*(x'|x,a) V_{max}
    \end{align}
    
    By the $(\pi_b, e_P,\epsilon)$-constrained property, we know that $\forall x', \sum_{a'} e_P(x',a')\lvert\pi(a'|x') - \pi_b(a'|x')\rvert\leq\epsilon$. Applying Holder's inequality and combining the second and fourth term, we find:
    \begin{align}
        \lvert Q^*_{t+1}(x,a) - \widehat{Q}_{t+1}(x,a)\rvert &\leq e_P(x,a) V_{max} + \cfrac{\gamma \epsilon V_{max}}{\left(1-\gamma\right)\left(1-\kappa\gamma\right)}  \nonumber \\
         & \quad + \sum_{x',a'} \cfrac{\gamma e_P(x',a')}{1-\kappa\gamma} \pi_b(a'|x') P^*(x'|x,a) V_{max}
    \end{align}
    
    By Assumption \ref{ass:kappa}, we know that $\forall x, \sum_{x',a'} e_P(x',a') \pi_b(a'|x') P^*(x'|x,a) \leq \kappa e_P(x,a)$. This gives:
    \begin{align}
        \lvert Q^*_{t+1}(x,a) - \widehat{Q}_{t+1}(x,a)\rvert &\leq e_P(x,a) V_{max} + \cfrac{\gamma \epsilon V_{max}}{\left(1-\gamma\right)\left(1-\kappa\gamma\right)} + \cfrac{\kappa\gamma e_P(x,a)}{1-\kappa\gamma} V_{max} \\
        &= \cfrac{e_P(x,a) V_{max}}{1-\kappa\gamma} + \cfrac{\gamma \epsilon V_{max}}{\left(1-\gamma\right)\left(1-\kappa\gamma\right)}  \\
        &= \left(\cfrac{e_P(x,a)}{1-\kappa\gamma} + \cfrac{\gamma \epsilon}{\left(1-\gamma\right)\left(1-\kappa\gamma\right)}\right) V_{max}, \label{eq:68}
    \end{align}
    which concludes the induction. From standard properties of the Bellman operator, taking the limit proves the lemma.
    \qed\end{proof}

    \begin{theorem}
        Under Assumption \ref{ass:kappa}, any $(\pi_b, e_P,\epsilon)$-constrained policy $\pi$ satisfies the following inequality in every state $x$ with probability at least $1 - \delta$:
        \begin{align}
             V^\pi_{M^*}(x)-V^{\pi_b}_{M^*}(x) \geq V^\pi_{\widehat{M}}(x)-V^{\pi_b}_{\widehat{M}}(x) &- 2\big\lVert d^{\pi_b}_{M^*}(\,\boldsymbol{\cdot}\,|x)-d^{\pi_b}_{\widehat{M}}(\,\boldsymbol{\cdot}\,|x) \big\rVert_1 V_{max} \nonumber  \\
             &- \cfrac{1 + \gamma}{\left(1-\gamma\right)^2\left(1-\kappa\gamma\right)}\;\;\epsilon V_{max}.
        \end{align}
    \end{theorem}
    
\begin{proofsketch}
    We first prove the following:
    \begin{align}
        V^{\pi}_{M^*}-V^{\pi_b}_{M^*} = V^{\pi}_{\widehat{M}}-V^{\pi_b}_{\widehat{M}} &+ Q^{\pi}_{\widehat{M}}\left(\pi-\pi_b\right)\left(d_{M^*}^{\pi_b} - d_{\widehat{M}}^{\pi_b}\right) \nonumber \\
        &+ \left(Q^{\pi}_{M^*} - Q^{\pi}_{\widehat{M}}\right)\left(\pi-\pi_b\right)d_{M^*}^{\pi_b}.
    \end{align}
    
    We then apply Holder's inequality twice to the second term:
    \begin{align}
        \Big\lVert Q^{\pi}_{\widehat{M}}\left(\pi-\pi_b\right)\left(d^{\pi_b}_{M^*} - d^{\pi_b}_{\widehat{M}}\right) (\,\boldsymbol{\cdot}\,|x) \Big\rVert_1 \leq 2\Big\lVert \left(d^{\pi_b}_{M^*} - d^{\pi_b}_{\widehat{M}}\right)(\,\boldsymbol{\cdot}\,|x) \Big\rVert_1 V_{max}.
    \end{align}
    
    Afterwards, thanks to Assumption \ref{ass:kappa}, we may use Lemma \ref{lem:induction} to bound the third term in function of $V_{max}, \epsilon, \gamma$, and $\kappa$:
    \begin{align}
        &\Big\lVert\left(Q^{\pi}_{M^*} - Q^{\pi}_{\widehat{M}}\right)\left(\pi-\pi_b\right)d^{\pi_b}_{M^*}(\,\boldsymbol{\cdot}\,|x)\Big\rVert_1\leq \cfrac{1 + \gamma}{\left(1-\gamma\right)^2\left(1-\kappa\gamma\right)}\;\;\epsilon V_{max} .
    \end{align}
    
    Plugging all the pieces together proves the theorem.
\qed\end{proofsketch}

    \begin{proof}
    We first apply Proposition \ref{prop:dV*-d^V} with $M_1=M^*, M_2=\widehat{M}, \pi_1=\pi$, and $\pi_2=\pi_b$:
    \begin{align}
        V^{\pi}_{M^*}-V^{\pi_b}_{M^*} &= V^{\pi}_{\widehat{M}}-V^{\pi_b}_{\widehat{M}} + Q^{\pi}_{\widehat{M}}\left(\pi-\pi_b\right)\left(d_{M^*}^{\pi_b} - d_{\widehat{M}}^{\pi_b}\right)  \nonumber \\
         & \quad + \left(Q^{\pi}_{M^*} - Q^{\pi}_{\widehat{M}}\right)\left(\pi-\pi_b\right)d_{M^*}^{\pi_b} \\
        V^\pi_{M^*}(x)-V^{\pi_b}_{M^*}(x) &\leq V^\pi_{\widehat{M}}(x)-V^{\pi_b}_{\widehat{M}}(x) - \Big\lVert Q^{\pi}_{\widehat{M}}\left(\pi-\pi_b\right)\left(d_{M^*}^{\pi_b} - d_{\widehat{M}}^{\pi_b}\right)(\,\boldsymbol{\cdot}\,|x)\Big\rVert_1  \nonumber \\
         & \quad - \Big\lVert\left(Q^{\pi}_{M^*} - Q^{\pi}_{\widehat{M}}\right)\left(\pi-\pi_b\right)d^*_{\pi_b}(\,\boldsymbol{\cdot}\,|x)\Big\rVert_1.
    \end{align}
    
    Then, we apply twice Holder's inequality to the second term 1-norm:
    \begin{align}
        \Big\lVert Q^{\pi}_{\widehat{M}}\left(\pi-\pi_b\right)&\left(d^{\pi_b}_{M^*} - d^{\pi_b}_{\widehat{M}}\right)(\,\boldsymbol{\cdot}\,|x) \Big\rVert_1 \nonumber \\ 
        &\leq \big\lVert Q^{\pi}_{\widehat{M}}\big\rVert_\infty \big\lVert\left(\pi-\pi_b\right)(\,\boldsymbol{\cdot}\,|x')\big\rVert_1 \Big\lVert \left(d^{\pi_b}_{M^*} - d^{\pi_b}_{\widehat{M}}\right)(\,\boldsymbol{\cdot}\,|x) \Big\rVert_1\\
        &\leq 2\Big\lVert \left(d^{\pi_b}_{M^*} - d^{\pi_b}_{\widehat{M}}\right)(\,\boldsymbol{\cdot}\,|x) \Big\rVert_1. V_{max}.\label{eq:42}
    \end{align}
    
    Similarly, we obtain the following with the third term:
    \begin{align}
        \Big\lVert&\left(Q^{\pi}_{M^*} - Q^{\pi}_{\widehat{M}}\right)\left(\pi-\pi_b\right)d^*_{\pi_b}(\,\boldsymbol{\cdot}\,|x)\Big\rVert_1 \nonumber \\ 
        &\leq \big\lVert d^*_{\pi_b}(\,\boldsymbol{\cdot}\,|x)\big\rVert_1 \max_{x'} \sum_{a'}\left(Q^{\pi}_{M^*}(x',a') - Q^{\pi}_{\widehat{M}}(x',a')\right)\left(\pi(a'|x')-\pi_b(a'|x') \right) \\
        &\leq \cfrac{1}{1-\gamma} \max_{x'} \sum_{a'}\big\lvert Q^{\pi}_{M^*}(x',a') - Q^{\pi}_{\widehat{M}}(x',a')\big\rvert\lvert\pi(a'|x')-\pi_b(a'|x') \rvert .
    \end{align}

    
    We use Lemma \ref{lem:induction} to bound the third term in function of $V_{max}, \epsilon, \gamma$, and $e_P(x',a')$:
    \begin{align}
        \Big\lVert&\left(Q^{\pi}_{M^*} - Q^{\pi}_{\widehat{M}}\right)\left(\pi-\pi_b\right)d^*_{\pi_b}(\,\boldsymbol{\cdot}\,|x)\Big\rVert_1 \nonumber \\ 
        &\leq \cfrac{1}{1-\gamma} \max_{x'} \sum_{a'}\left(\cfrac{e_P(x',a')}{1-\kappa\gamma} + \cfrac{\gamma\epsilon}{\left(1-\gamma\right)\left(1-\kappa\gamma\right)}\right)V_{max}\lvert\pi(a'|x')-\pi_b(a'|x') \rvert \\
        &\leq \cfrac{1}{1-\gamma} \max_{x'} 
        \left(\cfrac{\sum_{a'}e_P(x',a')\lvert\pi(a'|x')-\pi_b(a'|x')\rvert}{1-\kappa\gamma} + \cfrac{2\gamma \epsilon}{\left(1-\gamma\right)\left(1-\kappa\gamma\right)}\right) V_{max} \\
        &\leq \cfrac{1}{1-\gamma} \max_{x'} 
        \cfrac{1 + \gamma}{\left(1-\gamma\right)\left(1-\kappa\gamma\right)}\;\;\epsilon V_{max} \\
        &\leq \cfrac{1 + \gamma}{\left(1-\gamma\right)^2\left(1-\kappa\gamma\right)}\;\;\epsilon V_{max}
    \end{align}
    
    where the third line uses the $(\pi_b, e_P,\epsilon)$-constrained assumption. Plugging all the pieces together proves the theorem.
    \qed\end{proof}

    \begin{proposition}
    Let $M$ be an MDP: $\langle \mathcal{X}, \mathcal{A}, P, R, \gamma\rangle$, $\pi_b$ be the baseline policy on which trajectories $\mathcal{D}$ have been collected, $\widehat{M}$ be the MLE MDP: $\langle \mathcal{X}, \mathcal{A}, P, R, \gamma\rangle$, and $N_{\mathcal{D}}(x)$ be the count of transitions starting from state $x\in\mathcal{X}$ in $\mathcal{D}$. Then, the following concentration bound holds with high probabilities $1-\delta$:
    \begin{align}
        \Big\lVert \left(d^{\pi_b}_{M} - d^{\pi_b}_{\widehat{M}}\right)(\,\boldsymbol{\cdot}\,|x) \Big\rVert_1 &\leq \cfrac{1}{1-\gamma}\sqrt{\cfrac{2}{N_{\mathcal{D}}(x)}\log\cfrac{2^{|\mathcal{X}|}}{\delta}}.
        \end{align} \label{prop:Q-dpi-dd}
    \end{proposition}
    \begin{proof}
    We apply concentration bounds. The proof relies on Theorem 2.1 in \cite{weissman2003inequalities} and is identical to that of Equation \ref{eq:error-function-P-2} borrowed from \cite{Petrik2016}, except that $d^{\pi_b}_{M}$ and $d^{\pi_b}_{\widehat{M}}$ are state distributions and not probability distributions and that their sums are therefore bounded by $\frac{1}{1-\gamma}$. It gives:
    \begin{align}
        \Big\lVert \left(d^{\pi_b}_{M} - d^{\pi_b}_{\widehat{M}}\right)(\,\boldsymbol{\cdot}\,|x) \Big\rVert_1 &\leq \cfrac{1}{1-\gamma}\sqrt{\cfrac{2}{N_{\mathcal{D}}(x)}\log\cfrac{2^{|\mathcal{X}|}}{\delta}},
    \end{align}
    which concludes the proof.
    \qed\end{proof}
  
\subsection{Proof of Theorem \ref{thm:PI}}
    \begin{theorem}
        The policy improvement step of Approx-Soft-SPIBB generates policies that are guaranteed to be $(\pi_b, e, \epsilon)$-constrained.
    \end{theorem}
    \begin{proof}
        The summation of the product of probability mass moved and the error function is below the budget $\epsilon$ by design.
    \qed\end{proof}

\subsection{Proof of Theorem \ref{thm:complexity}}
    \begin{theorem}
    Approx-Soft-SPIBB policy improvement has a complexity of $\mathcal{O}(|\mathcal{X}||\mathcal{A}|^2)$.
    \end{theorem}
    \begin{proof}
    Approx-Soft-SPIBB performs a loop over the state space $\mathcal{X}$ and within that loop a set of operations over the action space. It sorts once the action space in terms of $Q_{\widehat{M}}^{(i)}$ and for each action performs an $\argmax$ over $|\mathcal{A}|$. This results in a $|\mathcal{A}|^2$ complexity for each state. 
    \qed\end{proof}
    
\subsection{Proof of Theorem \ref{thm:modelfree}}
    \begin{theorem}
        Considering an MDP with exact counts, the model-based policy iteration of (Exact or Approx)-Soft-SPIBB is identical to the model-free policy iteration of (resp. Exact or Approx)-Soft-SPIBB.
    \end{theorem}
    \begin{proof}
        Let $\pi^{(i+1)}$ be the (exact or approximate) policy improvement, which is assumed to be identical in the model-free and the model-based implementations of Soft-SPIBB. Then, it means that:
        \begin{align}
            Q^{(i+1)}_{M\text{-}Based}(x,a) &= \widehat{R}(x,a) + \gamma \sum_{x'\in\mathcal{X}} \widehat{P}(x'|x,a)\sum_{a'\in\mathcal{A}}\pi^{(i+1)}(a'|x') Q^{(i)}(x',a') \\
            &= \cfrac{\displaystyle\sum_{\langle x_j=x,a_j=a,x'_j,r_j\rangle \in \mathcal{D}} r_j}{N_\mathcal{D}(x,a)} \nonumber \\
            &\quad + \gamma \cfrac{\displaystyle\sum_{\langle x_j=x,a_j=a,x'_j,r_j\rangle \in \mathcal{D}} \:\sum_{a'\in\mathcal{A}}\pi^{(i+1)}(a'|x'_j) Q^{(i)}(x'_j,a')}{N_\mathcal{D}(x,a)} \label{eq:23} \\
            &= \cfrac{y^{(i)}_j}{N_\mathcal{D}(x,a)} = Q^{(i+1)}_{M\text{-}Free}(x,a),  \label{eq:24} 
        \end{align}
        where Equation \ref{eq:23} is obtained from the MLE MDP $\widehat{M}=\langle \mathcal{X},\mathcal{A},\widehat{P},\widehat{R},\gamma\rangle$ definition, and Equation \ref{eq:24} is obtained thanks to the target definition from Equation \ref{eq:fittedQ}.
    \qed\end{proof}
    
    \subsection{Approximate policy improvement pseudo-code}
    \label{sup:policy-approximate-policy-improvement}
    \RestyleAlgo{ruled}
    \begin{pseudocode} 
    \caption{Approx-Soft-SPIBB: Policy improvement} \label{alg:SSPIBB-sort-Q}
     \KwIn{Baseline policy $\pi_b$, current state $x$, action set $\mathcal{A}$, errors $e(x,a)$ for each $a \in \mathcal{A}$, precision level $\epsilon$ and last iteration value function $Q^{(i)}_{\widehat{M}}$}
     Initialize $\pi^{(i+1)}(\cdot | x) = \pi_b(\cdot | x)$ and $E = \epsilon$. \\
     Define $\mathcal{A^-}$ as $\mathcal{A}$ sorted in increasing order of $Q_{\widehat{M}}^{(i)}(x, \cdot)$. \\
     \For{$a^- \in \mathcal{A^-}$}{
        $m^- = \min\left(\pi^{(i+1)}(a^-|x), \cfrac{E}{2e(x,a^-)}\right)$ \\
        $a^+ = \displaystyle\argmax_{a\in\mathcal{A}}\cfrac{Q_{\widehat{M}}^{(i)}(x, a) - Q_{\widehat{M}}^{(i)}(x, a^-)}{e(x,a)}$ \\
        $m^+ = \min\left(m^-, \cfrac{E}{2e(x,a^+)} \right)$ \\
        \If{$m^+ > 0$}{
           $\pi^{(i+1)}(a^+|x) = \pi^{(i+1)}(a^+|x) + m^+$ \\
            $\pi^{(i+1)}(a^-|x) = \pi^{(i+1)}(a^-|x) - m^+$ \\
            $m^- = m^- - m^+$ \\
            $E = E - m^+\left(e(x, a^+) + e(x, a^-)\right)$ \\
        }
     }
    \Return $\displaystyle\argmax_{\pi\in\left\{\pi^{(i)},\pi^{(i+1)}\right\}} \sum_{a \in \mathcal{A}} \pi(a|x)Q^{(i)}_{\widehat{M}}(x,a)$ \\
    \end{pseudocode}

	\newpage
	\section{Random MDPs benchmark design}
	\label{sup:expe_finite}
	
	\subsection{Experiment details}
	\label{sup:randommdps}
	
	We describe the full experimental process in Pseudo-code \ref{alg:garnet_benchmark}. 
	
	\SetKwFor{RepTimes}{repeat}{times}{end}
	\begin{pseudocode}[h]
		\caption{Random MDPs benchmark}
		\KwIn{List of hyper-parameter values for the baseline}
		\KwIn{List of dataset size}
		\KwIn{List of algorithms in the benchmark}
		\KwIn{List of hyper-parameter values for each algorithm}
		\BlankLine
		\RepTimes{$10^5$}{
            Generate an MDP. (see Section \ref{sup:MDPgen})
            
		    \For{each hyper parameter value for the baseline}{
		        
		        Generate a baseline. (see Section \ref{sup:baselinegen})
		        
		        \For{each dataset size}{
		            
		            Generate a dataset. (see Section \ref{sup:datasetgen})
		            
		            \For{each algorithm}{
		            
		                \For{each algorithm hyper-parameter value}{
		                
		                    Train a policy.
		                    
		                    Evaluate the policy. (see Section \ref{sup:evaluationgen})
		                    
		                    Record the performance of the trained policy on this run.
		                
		                }
		            }
		        }
		    }
		}
		
		\label{alg:garnet_benchmark}
	\end{pseudocode}
	
	\subsubsection{MDP generation}
	\label{sup:MDPgen}
	See \cite[Appendix B.1.3]{Laroche2019}. Note that in our work, as well as in \cite{Laroche2019}, the number of states is set to 50.

	

	\subsubsection{Baseline generation}
	\label{sup:baselinegen}
	
	See \citep[Appendix B.1.4]{Laroche2019}.
	
	
	
	
	\subsubsection{Dataset generation}
	\label{sup:datasetgen}
	The MDP is modified to include another goal: terminal state with a reward of 1 when accessing it. The resulting environment is $M^*$. We do so to demonstrate the fact that Soft-SPIBB is less conservative than SPIBB, but still safe. The dataset generation depends a single parameter $|\mathcal{D}|\in\left\{10,20,50,100,200,500,1000,2000\right\}$: its size expressed in the number of trajectories. A trajectory generation simply consists in sampling the environment and the baseline policy until reaching the final state. The output is the dataset $\mathcal{D}$.
	
	\subsubsection{Policy iteration stopping criterion} The policy iteration process stops at iteration $i$ if the Frobenius norm between $Q_{\widehat{M}}^{(i)}$ and $Q_{\widehat{M}}^{(i-1)}$ is lower than $10^{-10}$. We did not specify a maximum number of iterations.  
	
	\subsubsection{Trained policy evaluation}
	\label{sup:evaluationgen}
	See \citep[Appendix B.1.6]{Laroche2019}.
	
	
	\subsubsection{Mean and CVaR performance}
	See \citep[Appendix B.1.7]{Laroche2019}.
	

	
	\subsubsection{Figures} We recall \citep[Appendix B.1.8]{Laroche2019} for clarity of this paper. We present three types of figure in the paper (main document and appendix).
	
	\paragraph{Performance vs. dataset size:} These figures (for instance Figure \subref*{subfig:benchmark_mean_eta=0.9_eps=2}) show the (mean and/or CVaR) performance of the algorithms as a function of the dataset size.
	
	\paragraph{Hyper-parameter search heatmaps:} These figures (for instance Figure \subref*{fig:heatmap_per_arg_norm_mean_perf_Pi_b_SPIBB_ratio}) show the (mean or CVaR) normalized performance of the algorithms as a function of both the dataset size and the hyper-parameter value of the evaluated algorithm. The normalized performance is computed with Equation \ref{eq:normalized_perf} and represented with colour. Red means that the performance is worse than that of the baseline, yellow means that it is equal and green means that it improves the baseline.
	
	\paragraph{Random MDPs heatmaps:} These figures (for instance Figure \subref*{fig:heatmap_norm_percentile_perf_Pi_b_SPIBB_N_wedge}) are very similar to the other heatmaps except that the normalized performance is shown as a function of both the dataset size and the hyper-parameter $\eta$ used for the baseline generation (instead of the hyper-parameter of the evaluated algorithm).
	
	\subsection{Other benchmark algorithms: competitors}
	\label{sup:benchmarkalgos}
	Since the \emph{baseline} meaning is overridden in this paper, we refer to the non-Soft-SPIBB benchmark algorithms with the term \emph{competitors}. 
	
	\subsubsection{Basic RL, HCPI, Robust MDP, Reward-adjusted MDP} See \cite[Appendix B.2.1 to B.2.4]{Laroche2019}.

    \subsubsection{SPIBB algorithms}
    The theory developed in \cite{Laroche2019} first computes from the dataset $\mathcal{D}$, the bootstrapped set $\mathfrak{B}$ of the state-action pairs that were not sampled enough to allow a good estimate of the model. The construction of $\mathfrak{B}$ depends on a hyper-parameter $N_\wedge$ which defines the minimal number of samples for a state-action pair to be considered well-known. Then, based on $\mathfrak{B}$, two algorithms are considered:
    \begin{itemize}
        \item The safety-proven $\Pi_b$-SPIBB copies the baseline policy for state-action pairs in $\mathfrak{B}$ and greedily optimizes in the complementary state-action pairs set.
        \item The empirically-improved $\Pi_{\leq b}$-SPIBB copies the baseline policy for state-action pairs in $\mathfrak{B}$ and greedily optimizes in the complementary state-action pairs set.
    \end{itemize}
    
    Table~\ref{tab:comprehensiveexample} illustrates the difference between $\Pi_b$-SPIBB and $\Pi_{\leq b}$-SPIBB in the policy improvement step of the policy iteration process. It shows how the baseline probability mass is locally redistributed among the different actions for the two policy-based SPIBB algorithms. We observe that for $\Pi_b$-SPIBB, the bootstrapped state-action pairs probabilities remain untouched whatever their $Q$-value estimates are. On the contrary, $\Pi_{\leq b}$-SPIBB removes all mass from the bootstrapped state-action pairs that are performing worse than the current $Q$-value estimates.
    
    Figure \ref{fig:SPIBB} show the result of the hyper-parameter search. The mean performance improves as the safety decreases. We choose $N_\wedge=10$ in our experiments because it seems to offer the best compromise and Soft-SPIBB turns out the improve at the same time SPIBB safety and mean performance under this setting.
    
    \subsection{More Random MDPs experiment results}
	\label{sup:random_MDPs_full_results}
	We provide next more raw results of our experiments in an arbitrary (not random) order: see Figures \ref{fig:random_mdps_full_mean_heatmaps_eta}, \ref{fig:random_mdps_full_heatmaps_eta}, \ref{fig:random_mdps_full_mean_heatmaps_eta_0.9}, \ref{fig:random_mdps_full_percentile_heatmaps_eta_0.1} and \ref{fig:random_mdps_RaMDP}.

    \clearpage
    
	\begin{table*}[t!]
		\caption{Policy improvement step at iteration $(i)$ for the two policy-based SPIBB algorithms.} 
		\centering
		\setlength\tabcolsep{3pt}
		\def\arraystretch{1.6}
		\small
		\begin{tabular}{|l|l|l|l|l|l|}
			\hline
			\normalsize$Q$-value estimate & \normalsize Baseline policy & \normalsize Boostrapping & \normalsize$\Pi_b$-SPIBB & \normalsize$\Pi_{\leq b}$-SPIBB \\ 
			\hline
			$Q_{\widehat{M}}^{(i)}(x,a_1) = 1$ & $\pi_b(a_1|x) = 0.1$ & $(x,a_1)\in\mathfrak{B}$ & $\pi^{(i+1)}(a_1|x) = 0.1$ &  $\pi^{(i+1)}(a_1|x) = 0$ \\
			$Q_{\widehat{M}}^{(i)}(x,a_2) = 2$ & $\pi_b(a_2|x) = 0.4$ & $(x,a_2)\notin\mathfrak{B}$ & $\pi^{(i+1)}(a_2|x) = 0$ &  $\pi^{(i+1)}(a_2|x) = 0$ \\
			$Q_{\widehat{M}}^{(i)}(x,a_3) = 3$ & $\pi_b(a_3|x) = 0.3$ & $(x,a_3)\notin\mathfrak{B}$ & $\pi^{(i+1)}(a_3|x) = 0.7$ &  $\pi^{(i+1)}(a_3|x) = 0.8$ \\
			$Q_{\widehat{M}}^{(i)}(x,a_4) = 4$ & $\pi_b(a_4|x) = 0.2$ & $(x,a_4)\in\mathfrak{B}$ & $\pi^{(i+1)}(a_4|x) = 0.2$ &  $\pi^{(i+1)}(a_4|x) = 0.2$ \\
			\hline
		\end{tabular}
		\normalsize
		\label{tab:comprehensiveexample}
	\end{table*}
	
	\begin{figure*}[t!]
		\centering
		\subfloat[Mean $\Pi_b$-SPIBB with $\eta=0.9$]{
			\includegraphics[trim = 10pt 140pt 45pt 60pt, clip, width=0.5\textwidth]{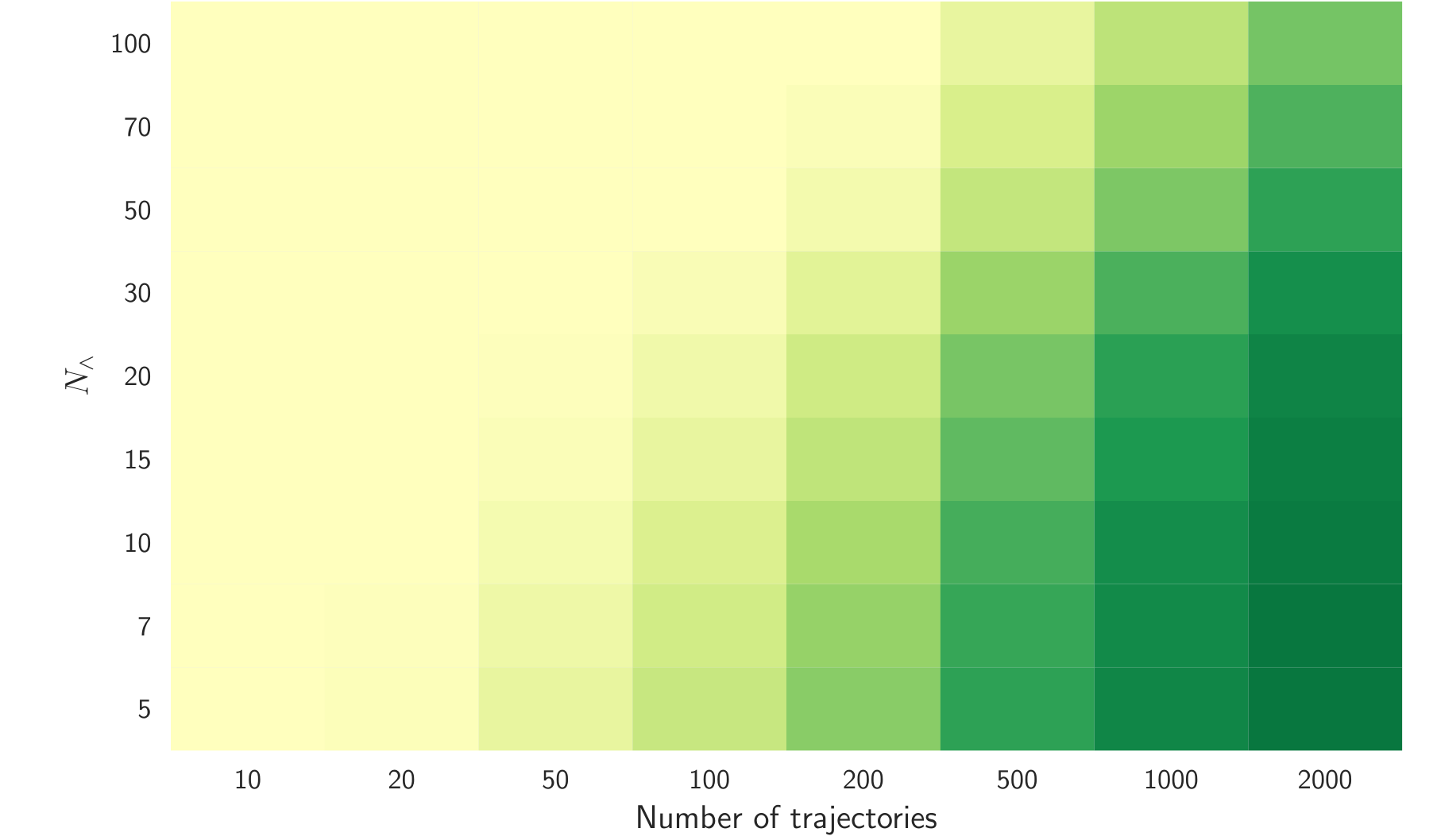}
			\label{fig:heatmap_per_arg_norm_mean_perf_Pi_b_SPIBB_ratio}
		}
		\subfloat[Mean $\Pi_{\leq b}$-SPIBB with $\eta=0.9$]{
			\includegraphics[trim = 10pt 140pt 45pt 60pt, clip, width=0.5\textwidth]{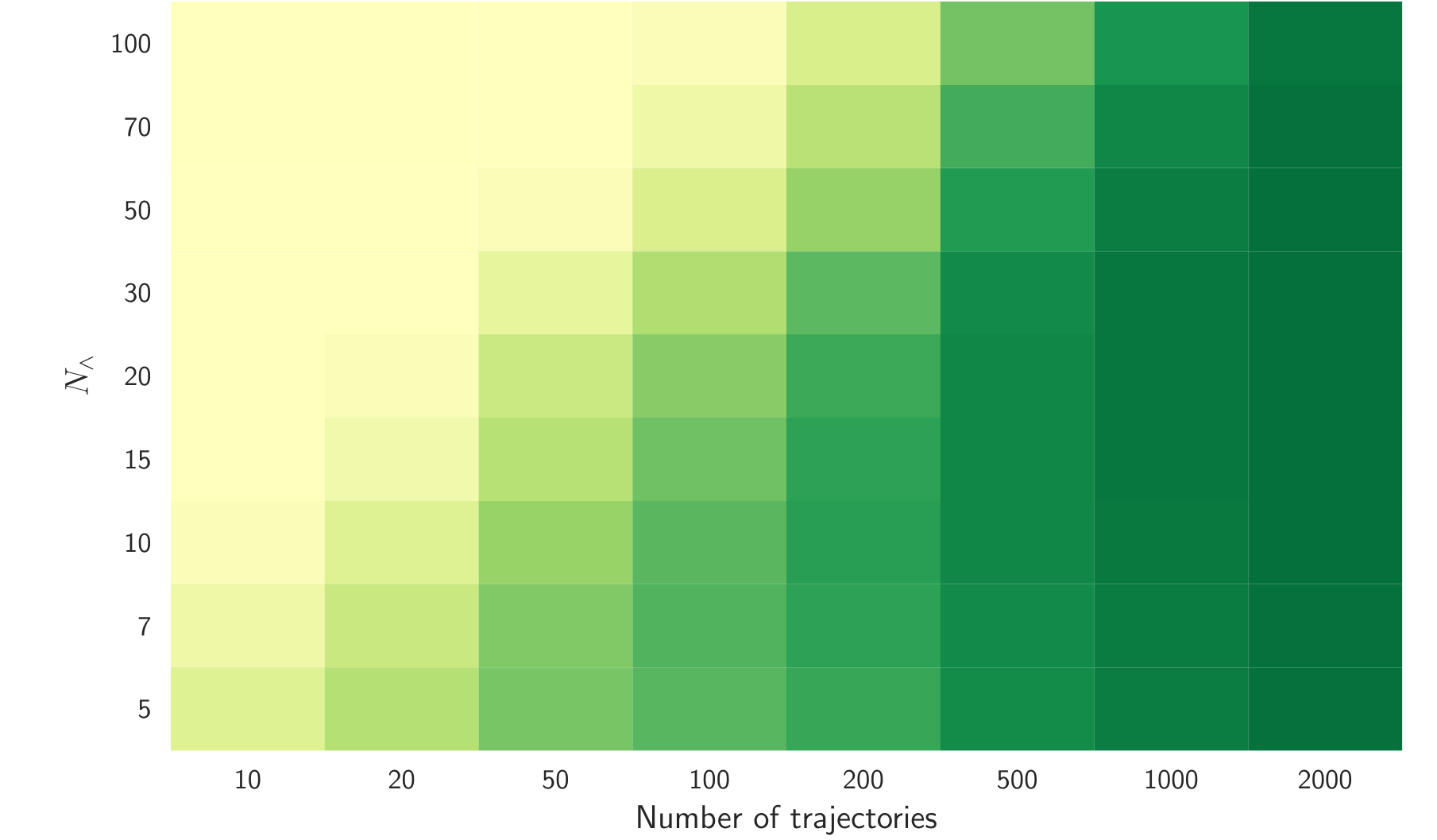}
			\label{fig:heatmap_per_arg_norm_mean_perf_Pi_leq_b_SPIBB_ratio}
		} \\
		\centering
		\subfloat[1\%-CVaR $\Pi_b$-SPIBB with $\eta=0.9$]{
			\includegraphics[trim = 10pt 140pt 45pt 60pt, clip, width=0.5\textwidth]{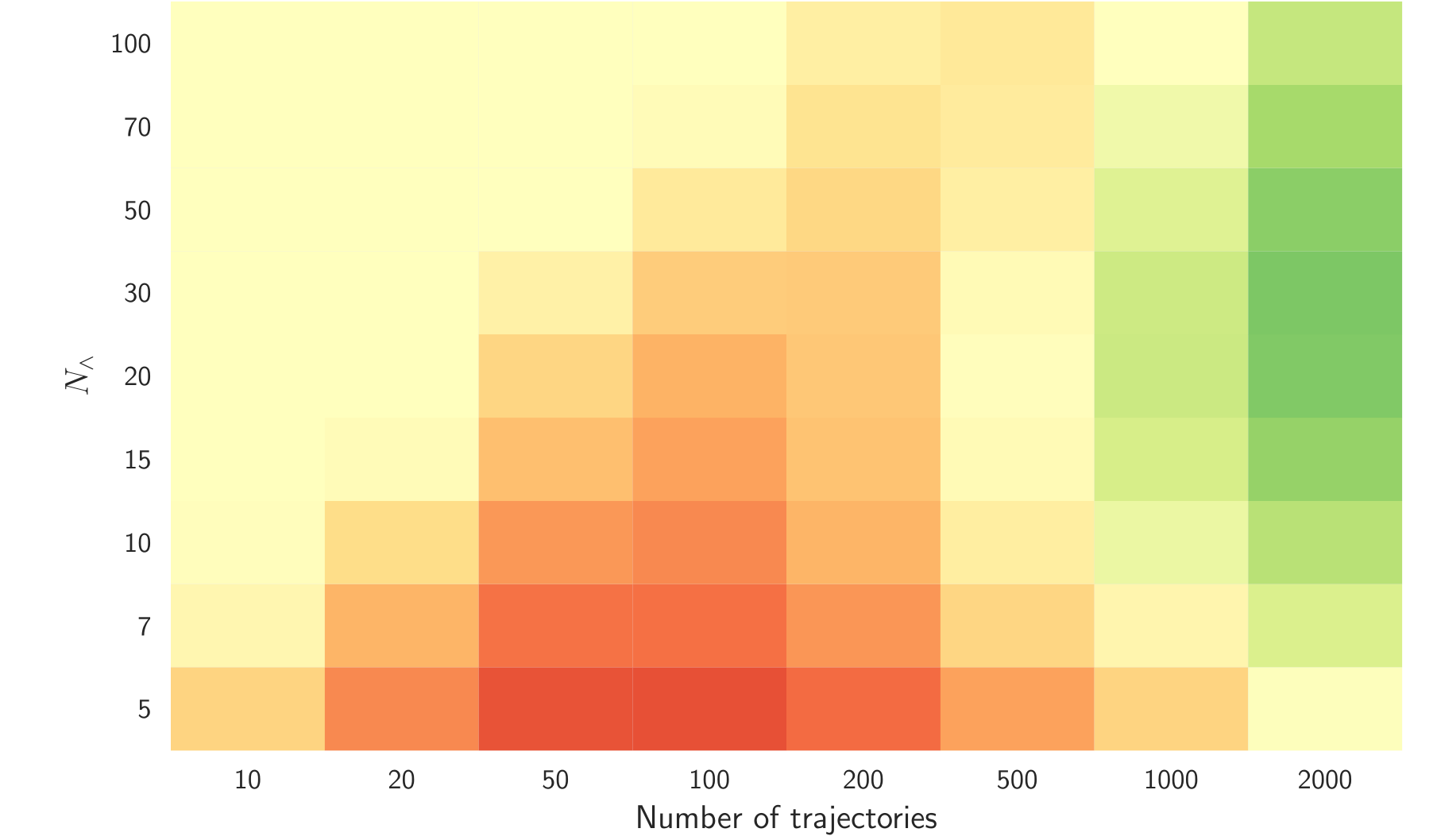}
			\label{fig:heatmap_per_arg_norm_percentile_perf_Pi_b_SPIBB_ratio}
		}
		\subfloat[1\%-CVaR $\Pi_{\leq b}$-SPIBB with $\eta=0.9$]{
			\includegraphics[trim = 10pt 140pt 45pt 60pt, clip, width=0.5\textwidth]{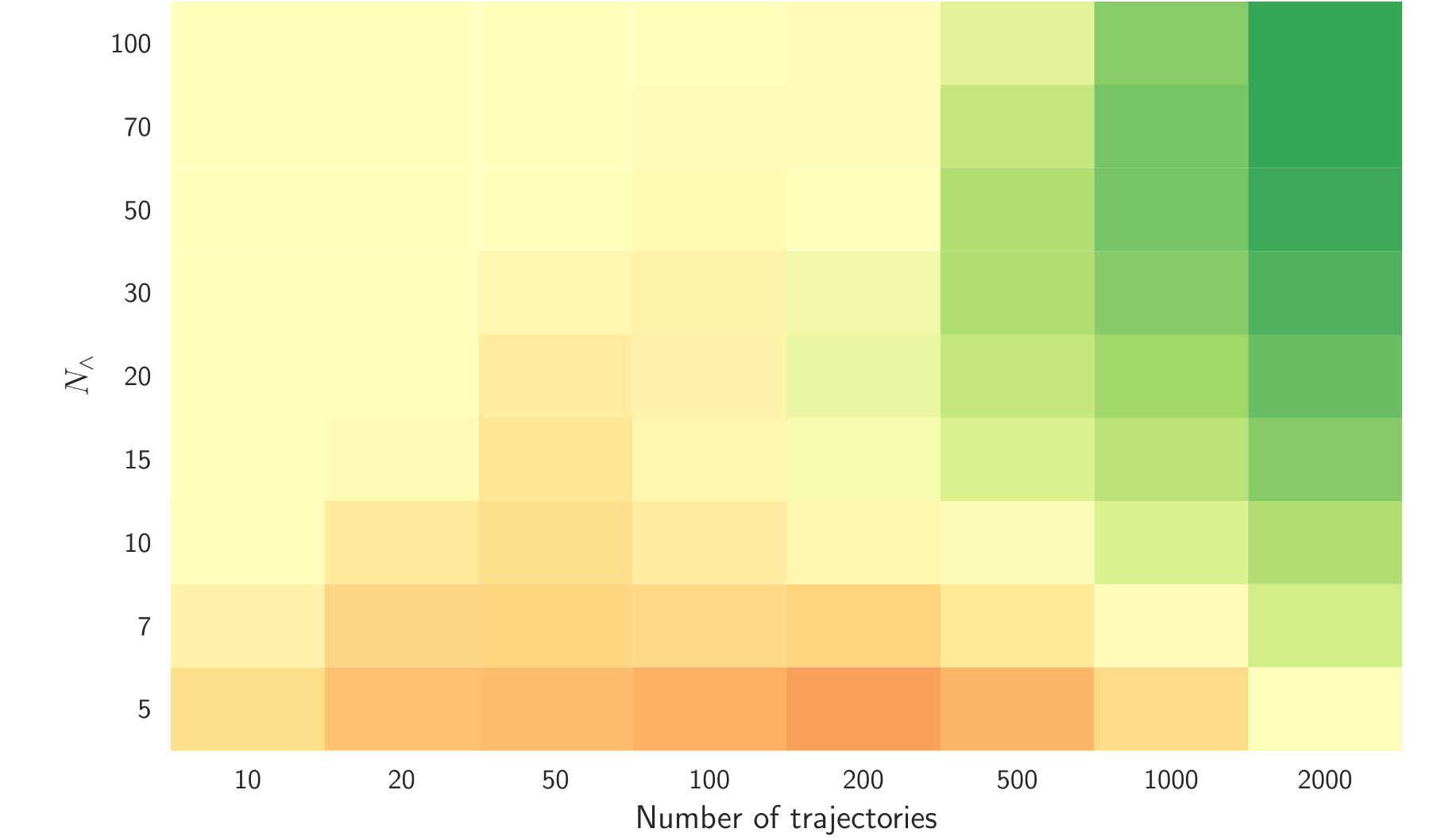}
			\label{fig:heatmap_per_arg_norm_percentile_perf_Pi_leq_b_SPIBB_ratio}
		} \\
		\centering
		\subfloat[1\%-CVaR $\Pi_b$-SPIBB ($N_\wedge=10$)]{
			\includegraphics[trim = 10pt 140pt 45pt 60pt, clip, width=0.5\textwidth]{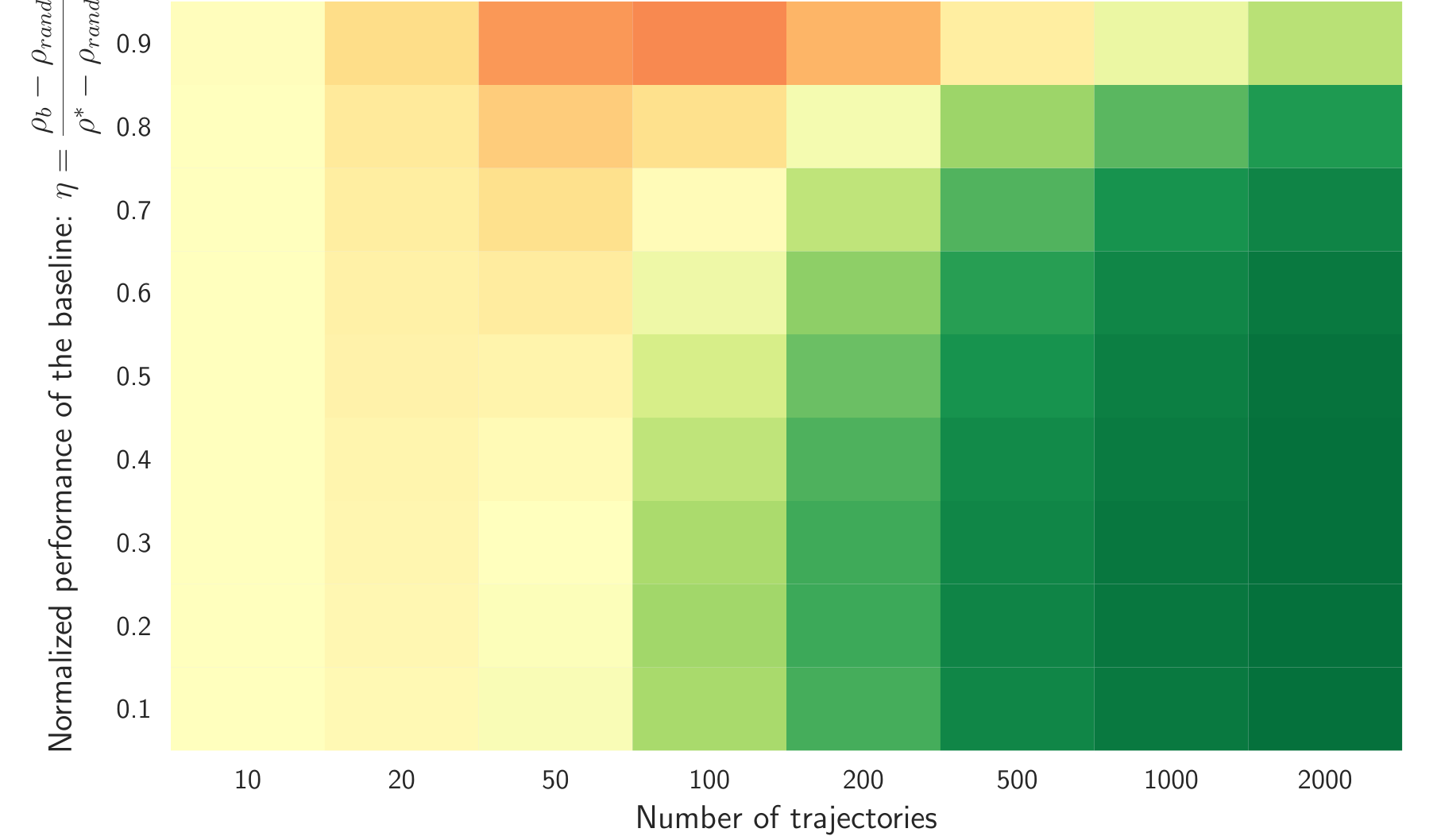}
			\label{fig:heatmap_norm_percentile_perf_Pi_b_SPIBB_N_wedge}
		}
		\subfloat[1\%-CVaR $\Pi_{\leq b}$-SPIBB ($N_\wedge=10$)]{
			\includegraphics[trim = 10pt 140pt 45pt 60pt, clip, width=0.5\textwidth]{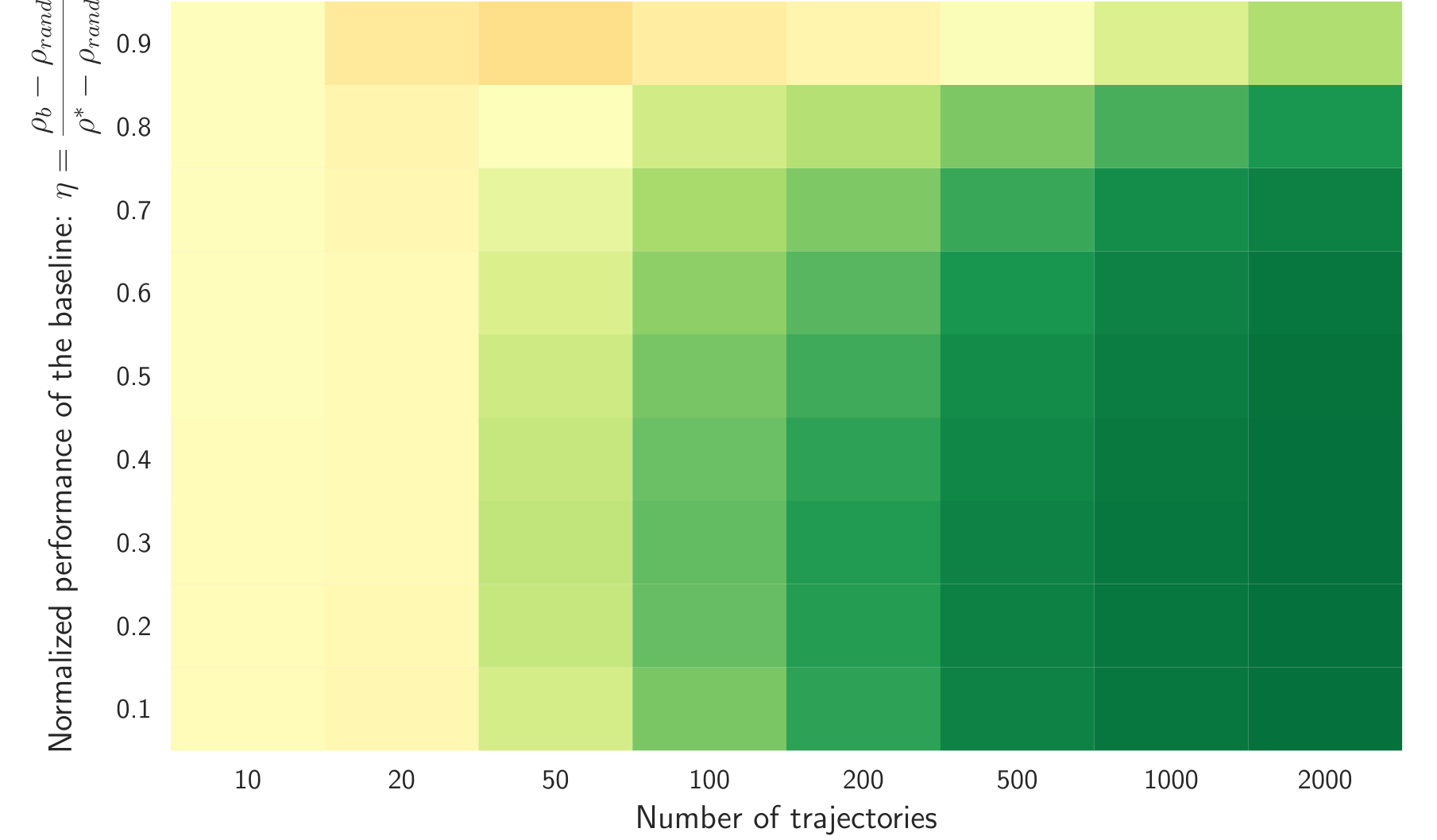}
			\label{fig:heatmap_norm_percentile_perf_Pi_leq_b_SPIBB_N_wedge}
		}
		\caption{Random MDPs (no additional goal): SPIBB hyper-parameter search results: (a-d) Mean and 1\%-CVaR performance heatmaps as a function of $N_\wedge$ (e-f) 1\%-CVaR performance heatmaps as a function of $\eta$ with the best hyper-parameter ($N_\wedge = 10$).}
		\label{fig:SPIBB}
	\end{figure*}
	
		\begin{figure*}[ht]
		\centering
		\subfloat[Mean: Basic RL.]{
			\includegraphics[trim = 10pt 140pt 45pt 60pt, clip, width=0.5\textwidth]{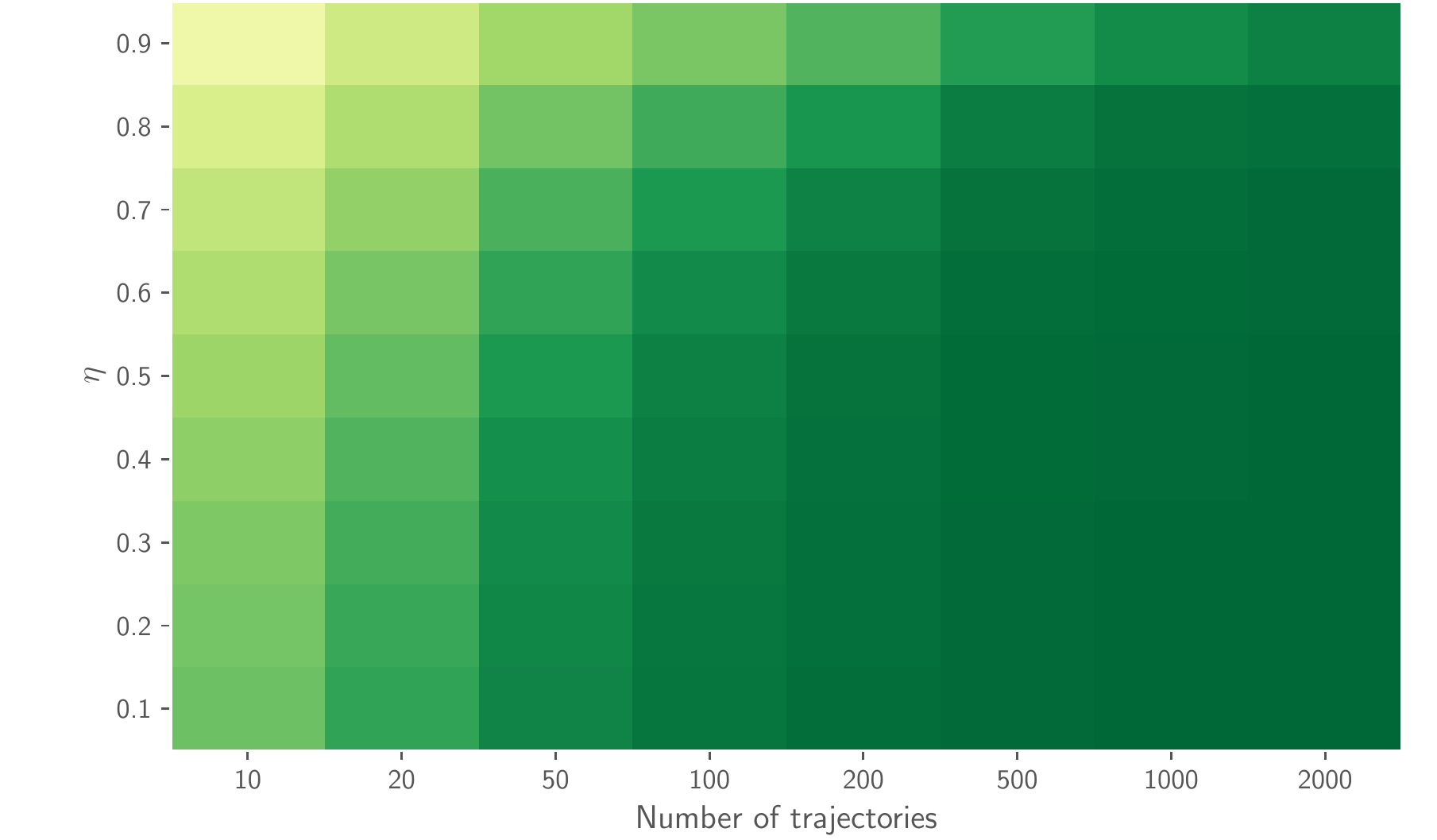}
			\label{fig:heatmap_eta_mean_perfrl}
		}
		\subfloat[Mean: RaMDP.]{
			\includegraphics[trim = 10pt 140pt 45pt 60pt, clip, width=0.5\textwidth]{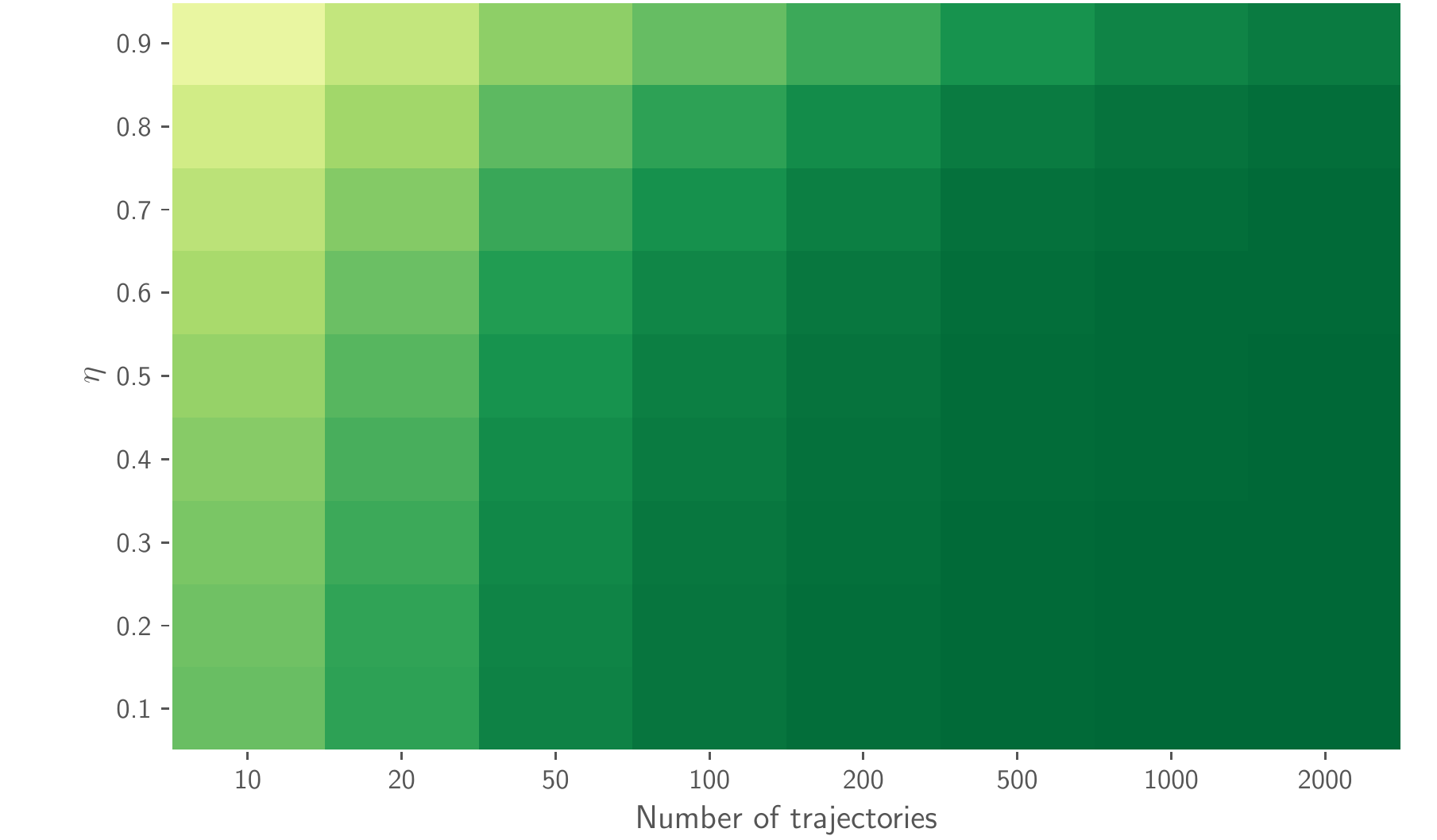}
			\label{fig:heatmap_eta_mean_ramdp}
		} \\
		\centering
		\subfloat[Mean: Robust MDP.]{
			\includegraphics[trim = 10pt 140pt 45pt 60pt, clip, width=0.5\textwidth]{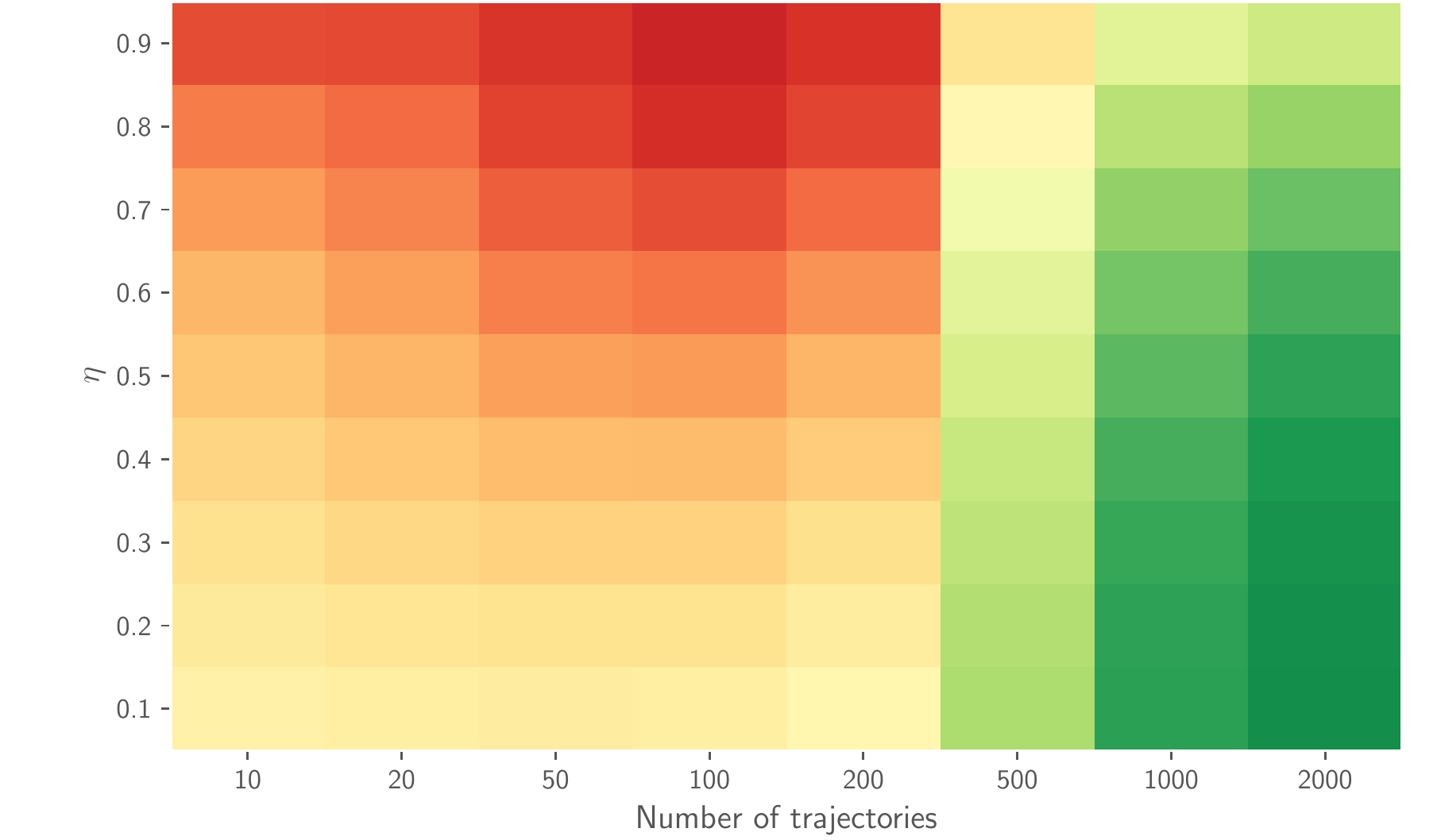}
			\label{fig:heatmap_eta_mean_rmdp}
		}
		\subfloat[Mean: HCPI doubly robust.]{
			\includegraphics[trim = 10pt 140pt 45pt 60pt, clip, width=0.5\textwidth]{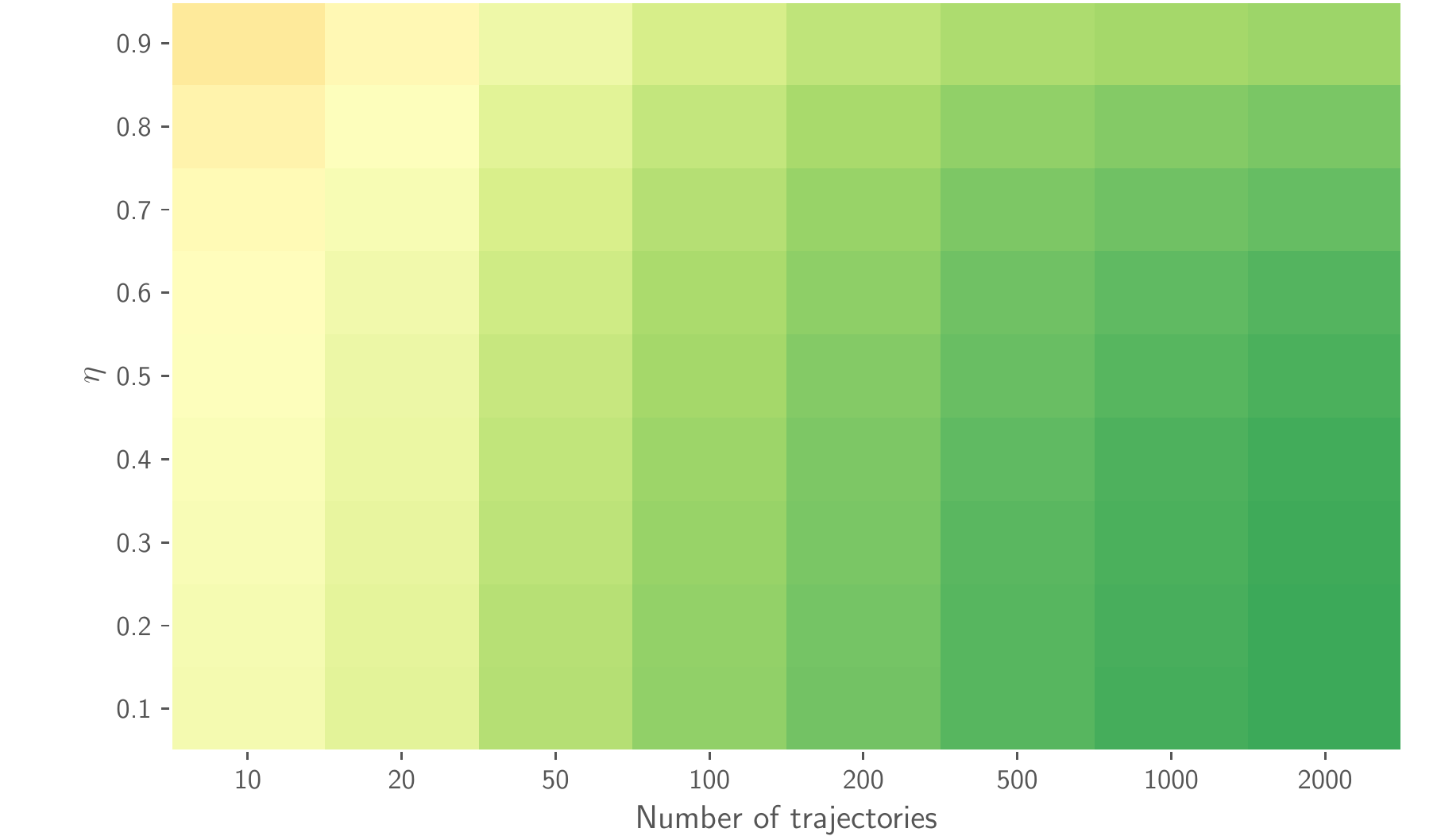}
			\label{fig:heatmap_eta_mean_hcpi}
		}\\
		\centering
		\subfloat[Mean: $\Pi_b$-SPIBB.]{
			\includegraphics[trim = 10pt 140pt 45pt 60pt, clip, width=0.5\textwidth]{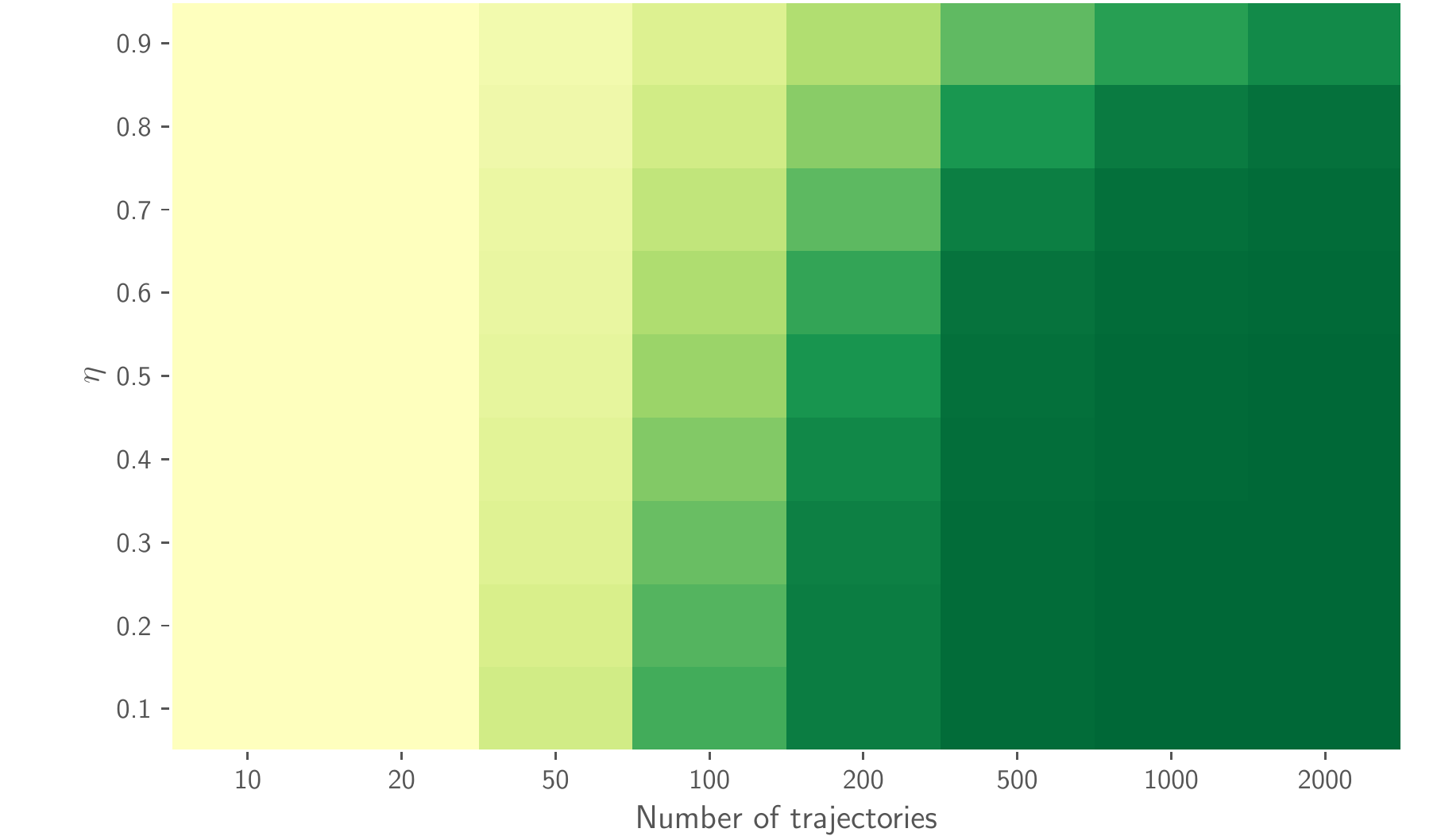}
			\label{fig:heatmap_eta_mean_pib}
		}
		\subfloat[Mean: $\Pi_{\leq b}$-SPIBB.]{
			\includegraphics[trim = 10pt 140pt 45pt 60pt, clip, width=0.5\textwidth]{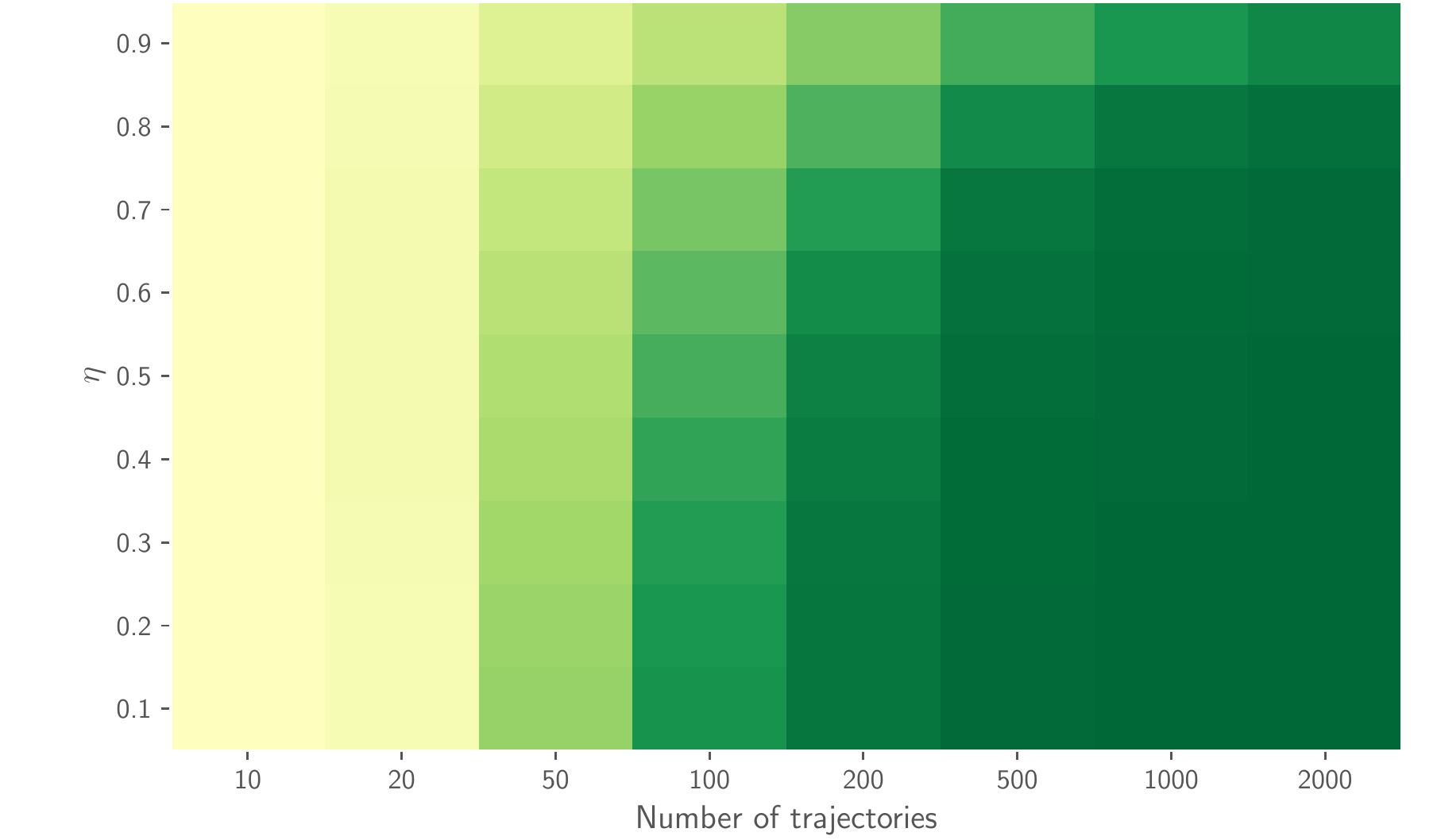}
			\label{fig:heatmap_eta_mean_pi_leq_b}
		} \\
		\centering
		\subfloat[Mean: Exact-Soft-SPIBB 1-step ($\epsilon=2$).]{
			\includegraphics[trim = 10pt 140pt 45pt 60pt, clip, width=0.5\textwidth]{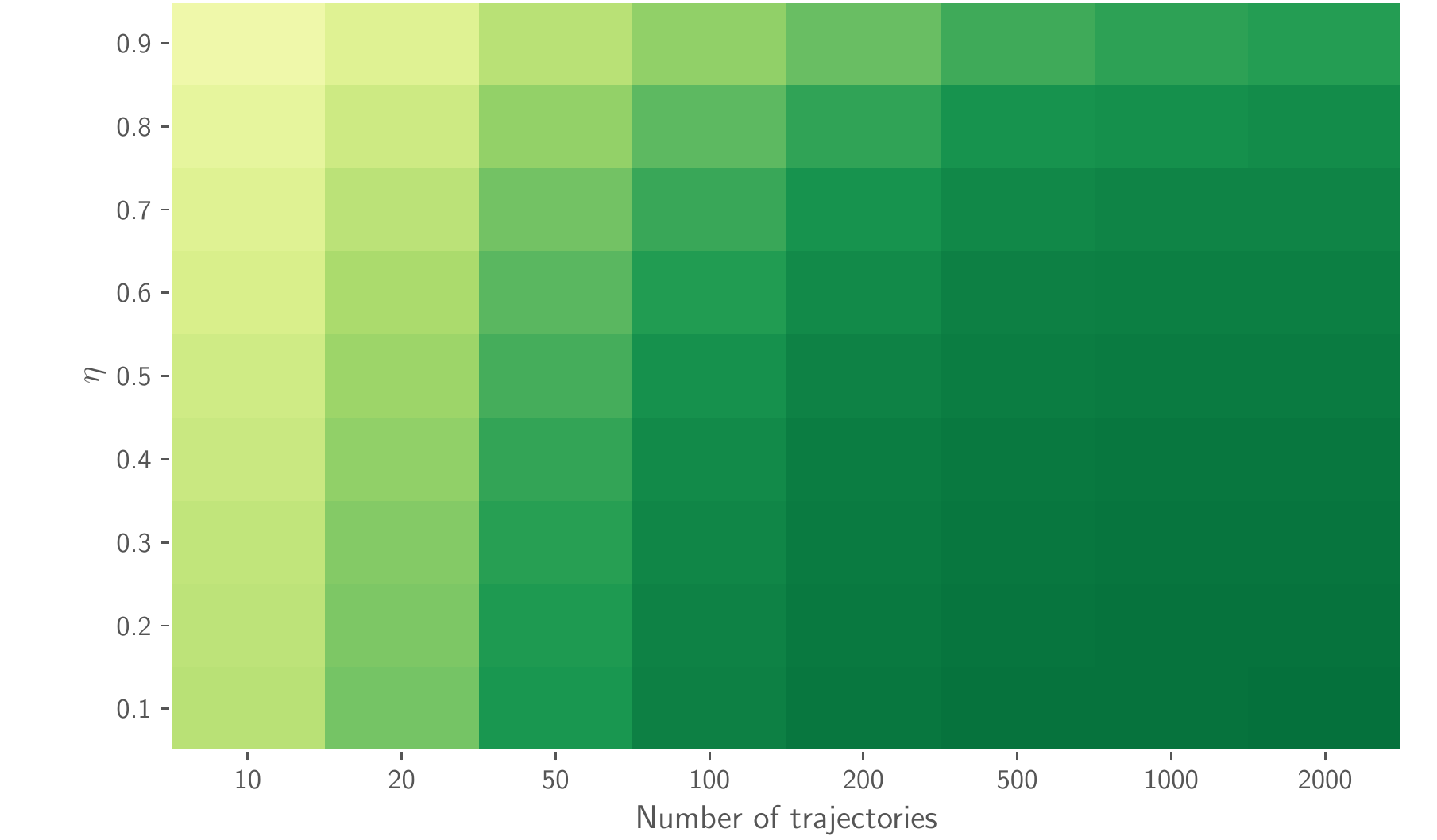}
			\label{fig:heatmap_eta_mean_exact_soft_spibb_1step_eps_2}
		}
		\subfloat[Mean: Exact-Soft-SPIBB ($\epsilon=2$).]{
			\includegraphics[trim = 10pt 140pt 45pt 60pt, clip, width=0.5\textwidth]{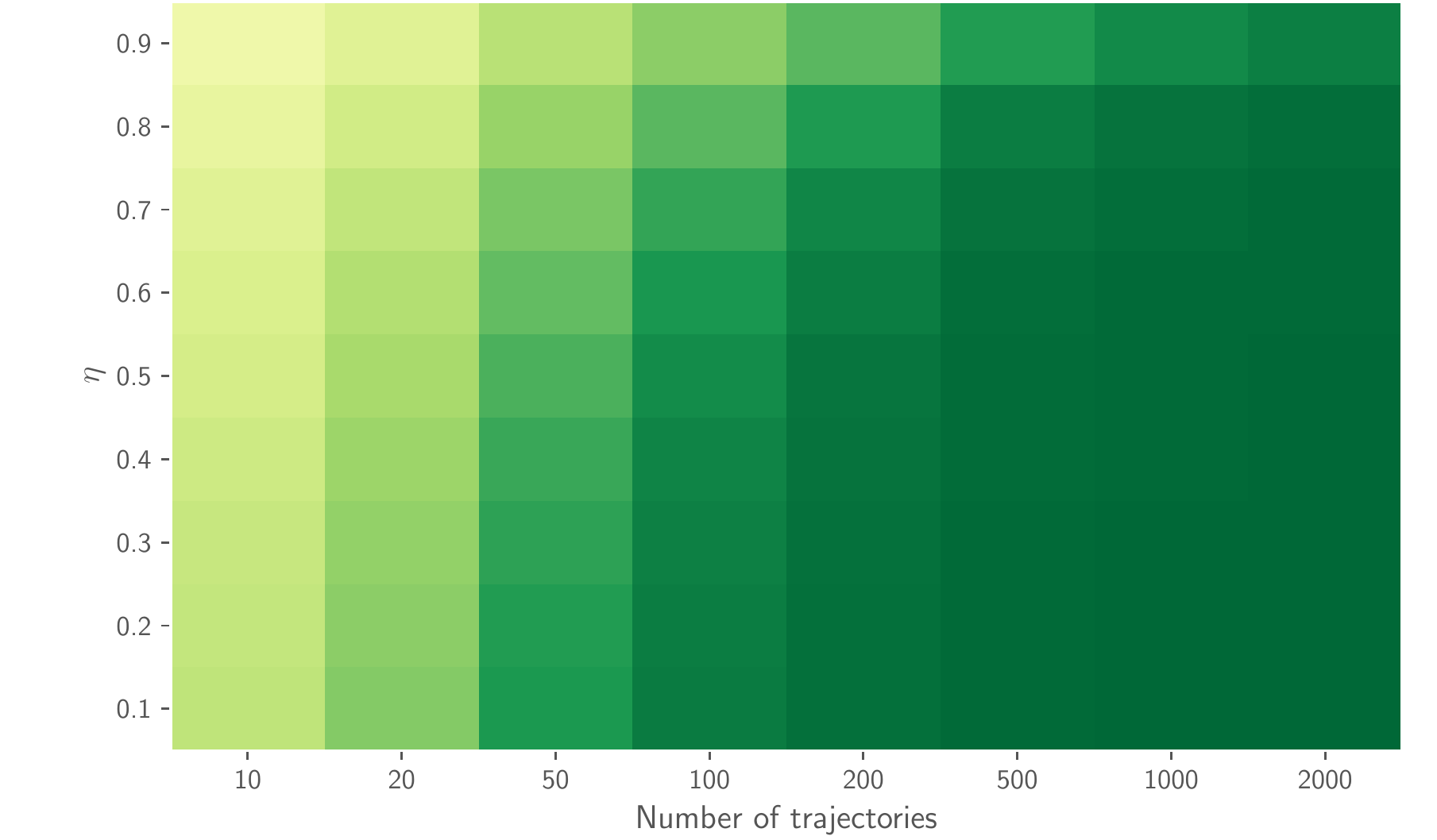}
			\label{fig:heatmap_eta_mean_exact_soft_spibb_eps_2}
		} \\
		\centering
		\subfloat[Mean: Approx-Soft-SPIBB 1-step ($\epsilon=2$).]{
			\includegraphics[trim = 10pt 140pt 45pt 60pt, clip, width=0.5\textwidth]{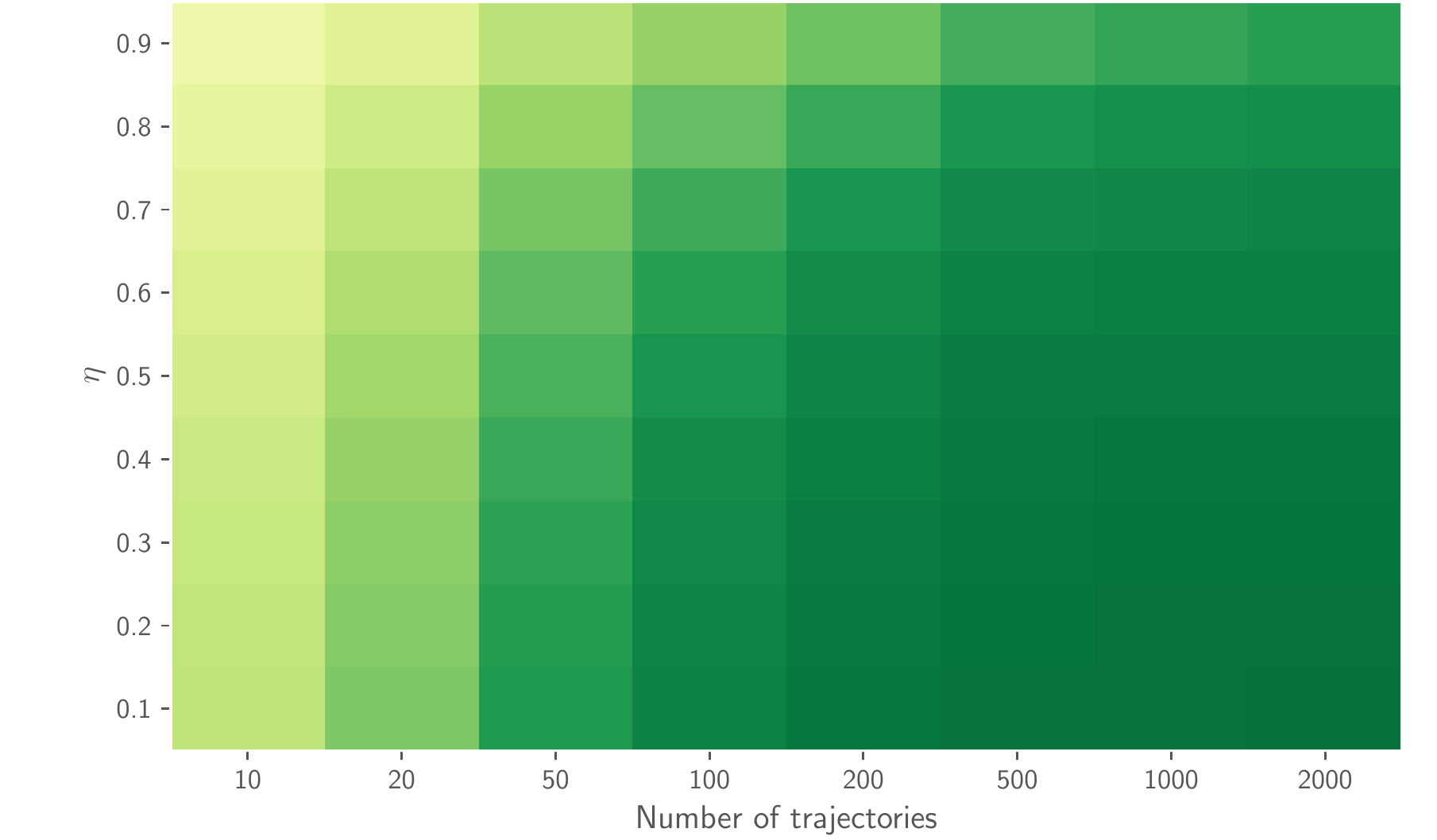}
			\label{fig:heatmap_eta_mean_approx_soft_spibb_1step_eps_2}
		}
		\subfloat[Mean: Approx-Soft-SPIBB  ($\epsilon=2$).]{
			\includegraphics[trim = 10pt 140pt 45pt 60pt, clip, width=0.5\textwidth]{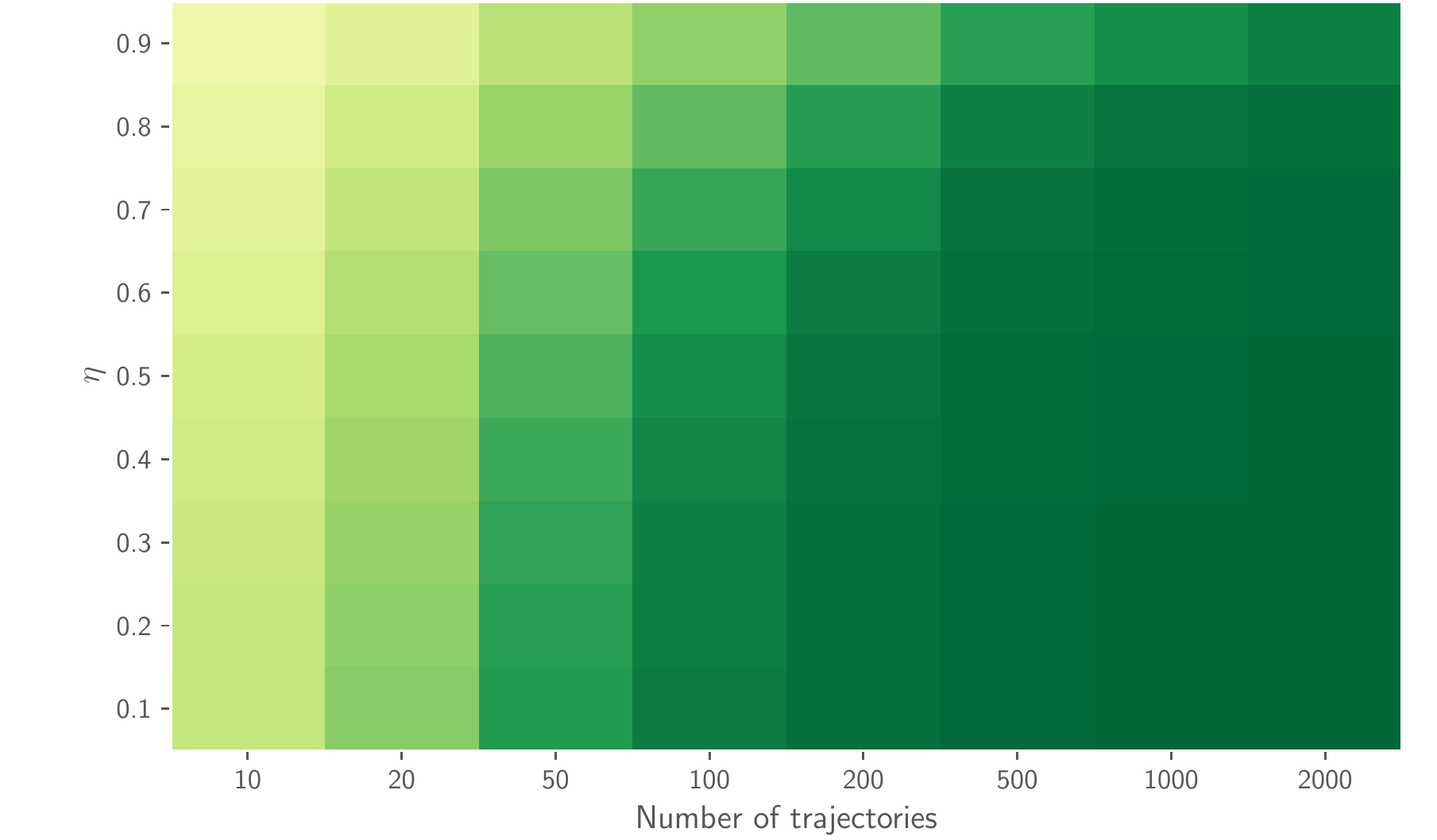}
			\label{fig:heatmap_eta_mean_approx_soft_spibb_eps_2}
		}
		\caption{Random MDPs: mean performance heatmaps.}
		\label{fig:random_mdps_full_mean_heatmaps_eta}
	\end{figure*}
	
	\begin{figure*}[t!]
		\centering
		\subfloat[0.1\%-CVaR: Basic RL.]{
			\includegraphics[trim = 10pt 140pt 45pt 60pt, clip, width=0.5\textwidth]{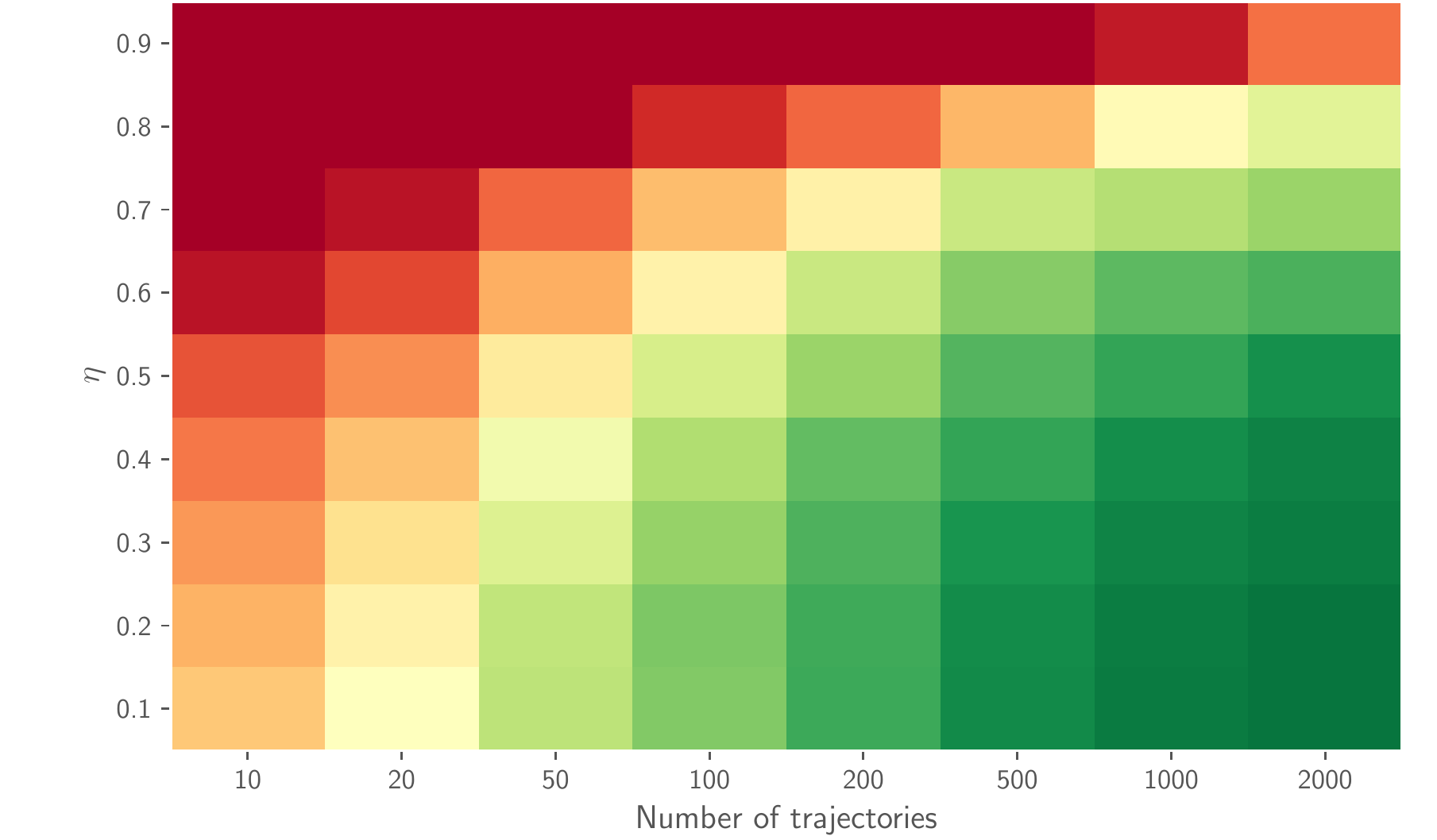}
			\label{fig:heatmap_eta_permille_perfrl}
		}
		\subfloat[0.1\%-CVaR: RaMDP.]{
			\includegraphics[trim = 10pt 140pt 45pt 60pt, clip, width=0.5\textwidth]{{figures/heatmaps_eta/heatmap_norm_permille_perf_RaMDP_epsilon=2}.pdf}
			\label{fig:heatmap_eta_permille_ramdp}
		} \\
		\centering
		\subfloat[0.1\%-CVaR: Robust MDP.]{
			\includegraphics[trim = 10pt 140pt 45pt 60pt, clip, width=0.5\textwidth]{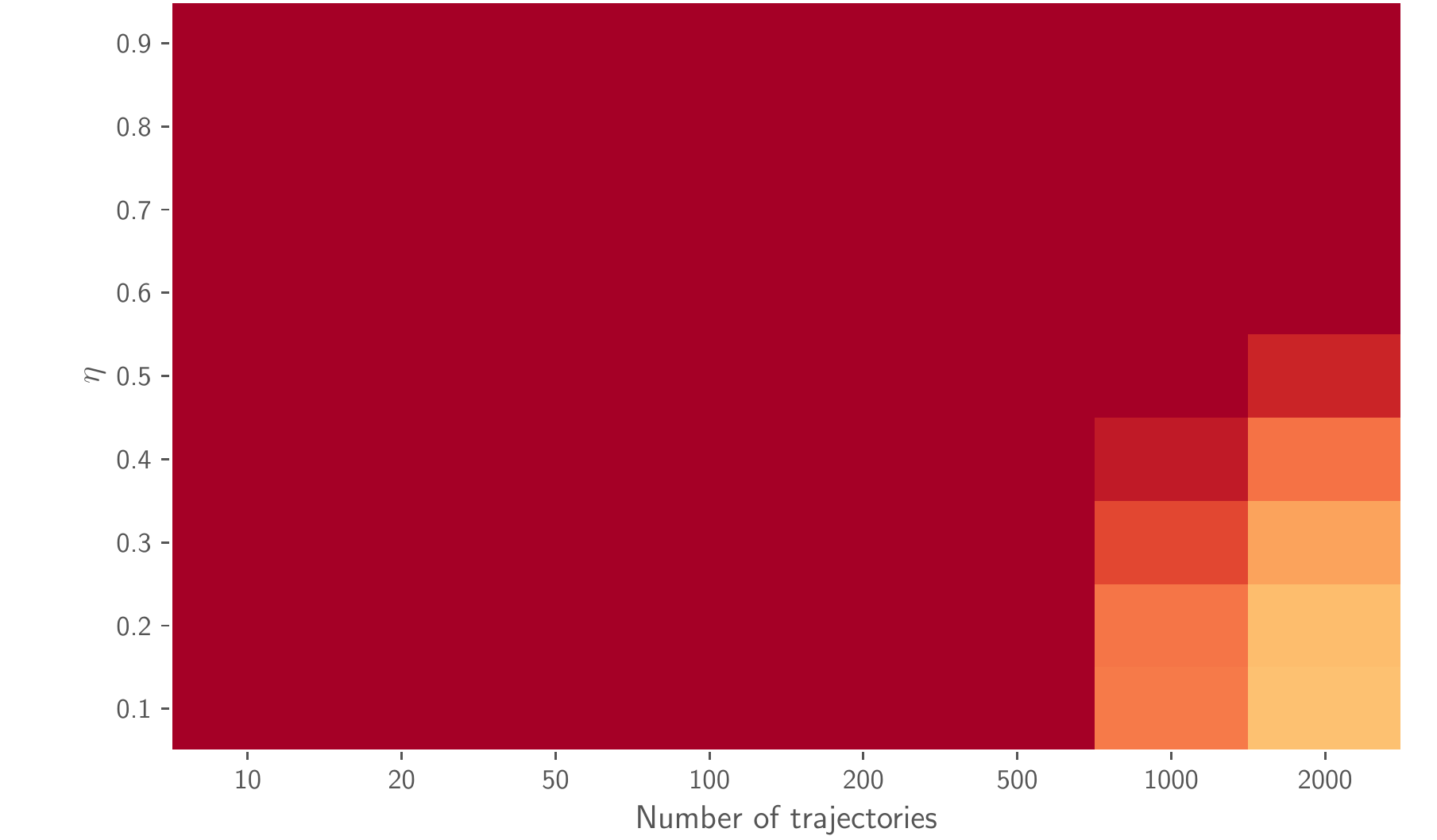}
			\label{fig:heatmap_eta_permille_rmdp}
		}
		\subfloat[0.1\%-CVaR: HCPI doubly robust.]{
			\includegraphics[trim = 10pt 140pt 45pt 60pt, clip, width=0.5\textwidth]{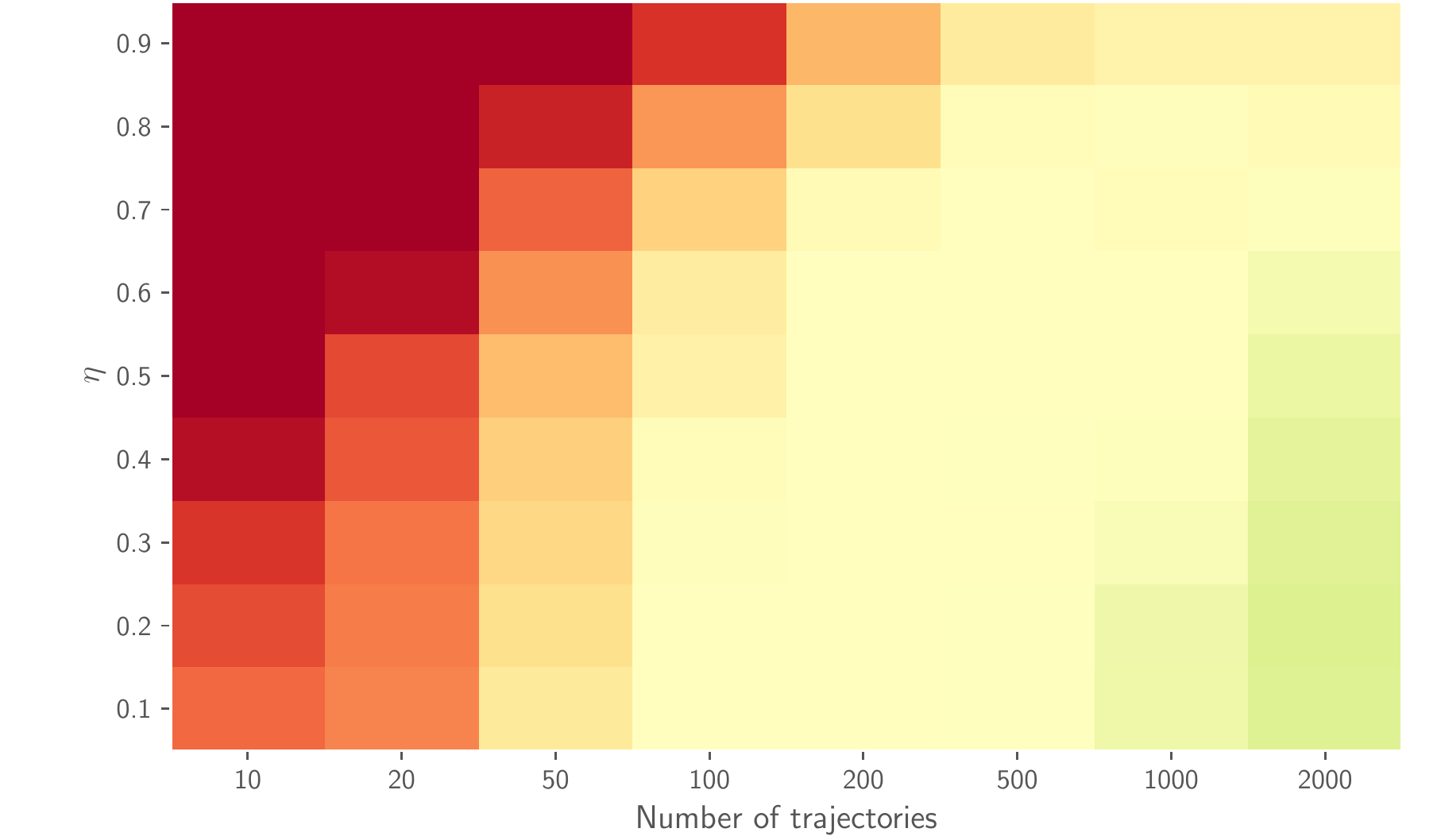}
			\label{fig:heatmap_eta_permille_hcpi}
		} \\
		\centering
		\subfloat[0.1\%-CVaR: $\Pi_b$-SPIBB.]{
			\includegraphics[trim = 10pt 140pt 45pt 60pt, clip, width=0.5\textwidth]{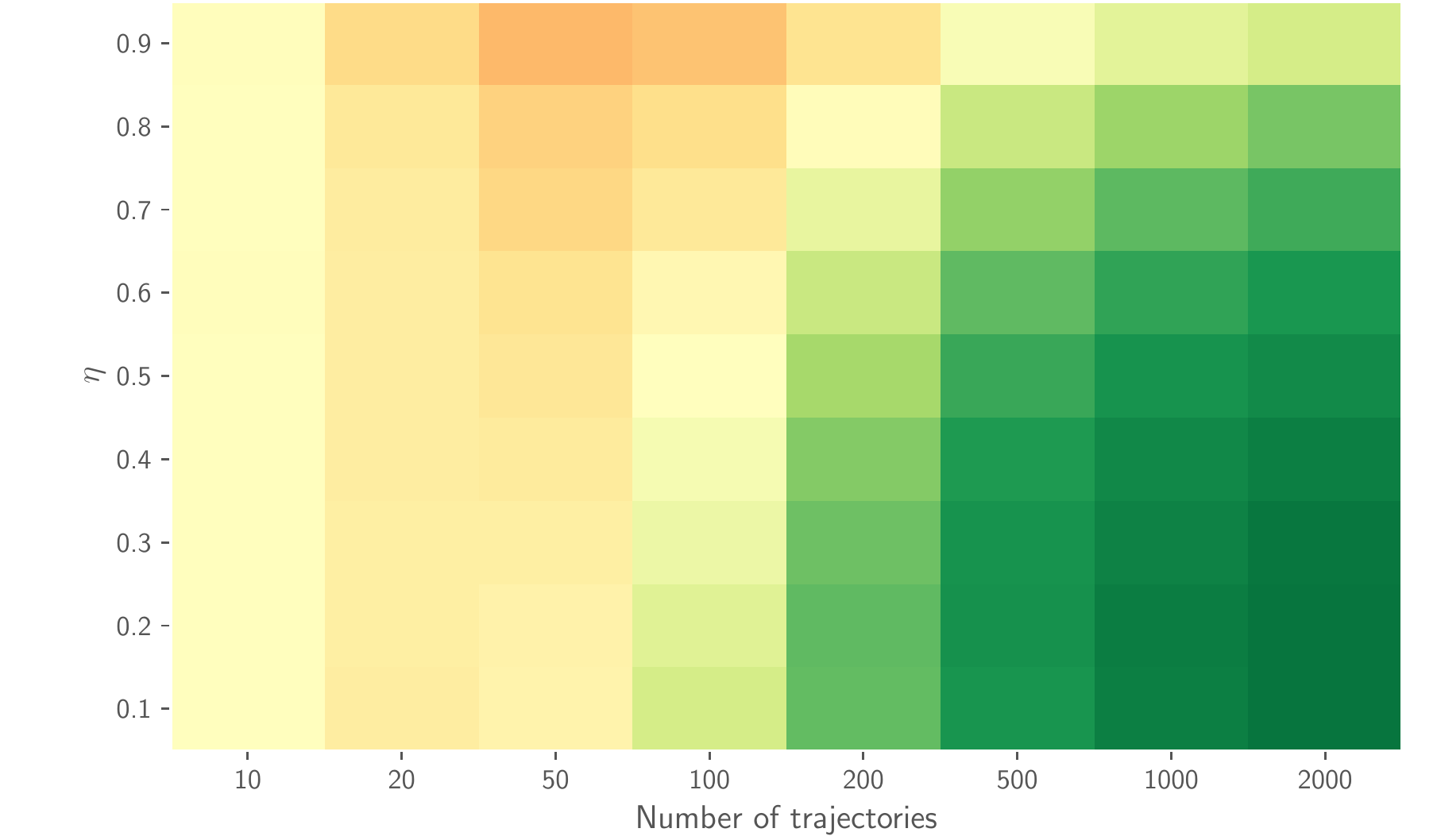}
			\label{fig:heatmap_eta_permille_pib}
		}
		\subfloat[0.1\%-CVaR: $\Pi_{\leq b}$-SPIBB.]{
			\includegraphics[trim = 10pt 140pt 45pt 60pt, clip, width=0.5\textwidth]{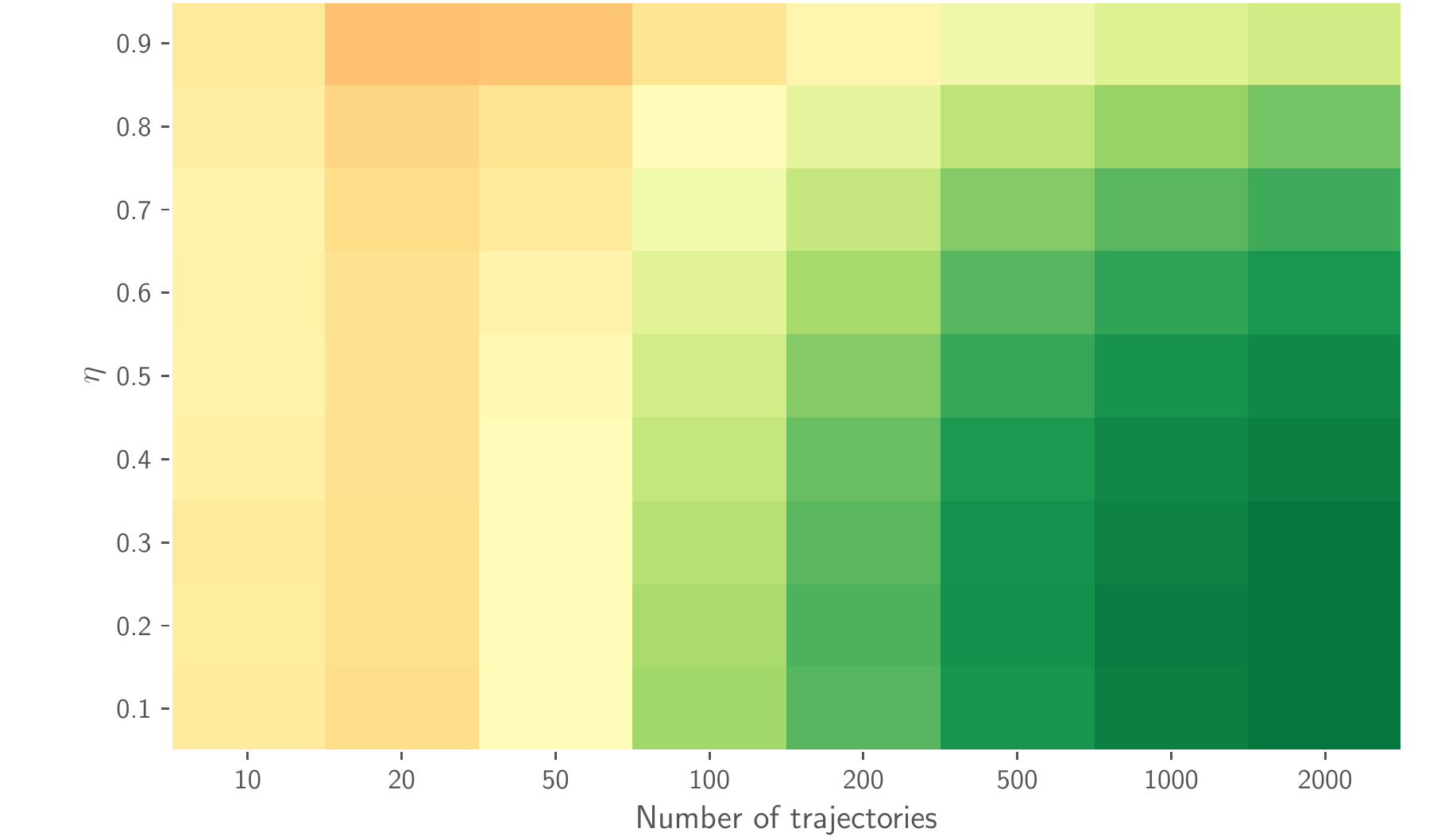}
			\label{fig:heatmap_eta_permille_pi_leq_b}
		} \\
		\centering
		\subfloat[0.1\%-CVaR: Exact-Soft-SPIBB 1-s ($\epsilon=2$).]{
			\includegraphics[trim = 10pt 140pt 45pt 60pt, clip, width=0.5\textwidth]{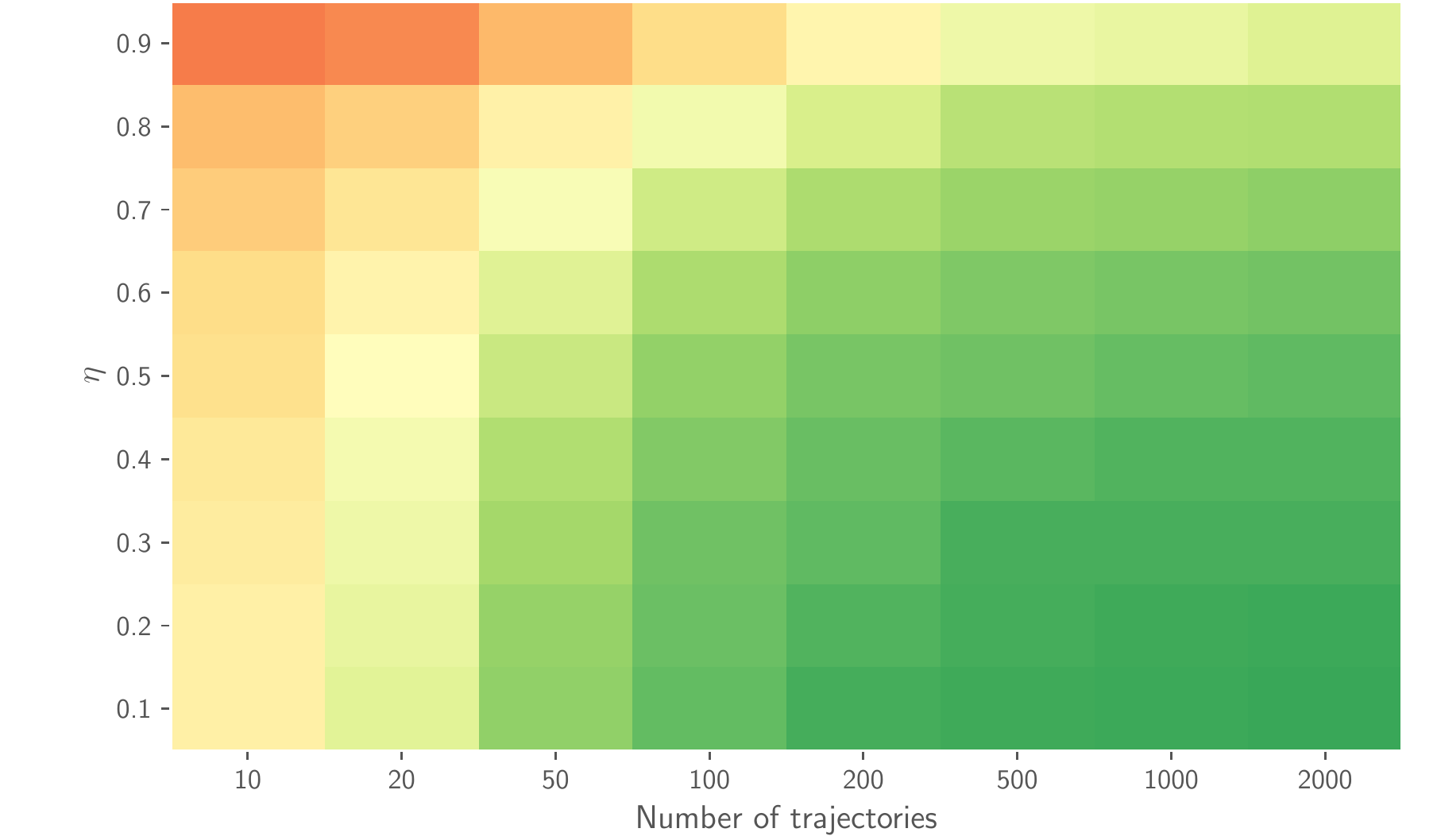}
			\label{fig:heatmap_eta_permille_exact_soft_spibb_1step_eps_2}
		}
		\subfloat[0.1\%-CVaR: Exact-Soft-SPIBB ($\epsilon=2$).]{
			\includegraphics[trim = 10pt 140pt 45pt 60pt, clip, width=0.5\textwidth]{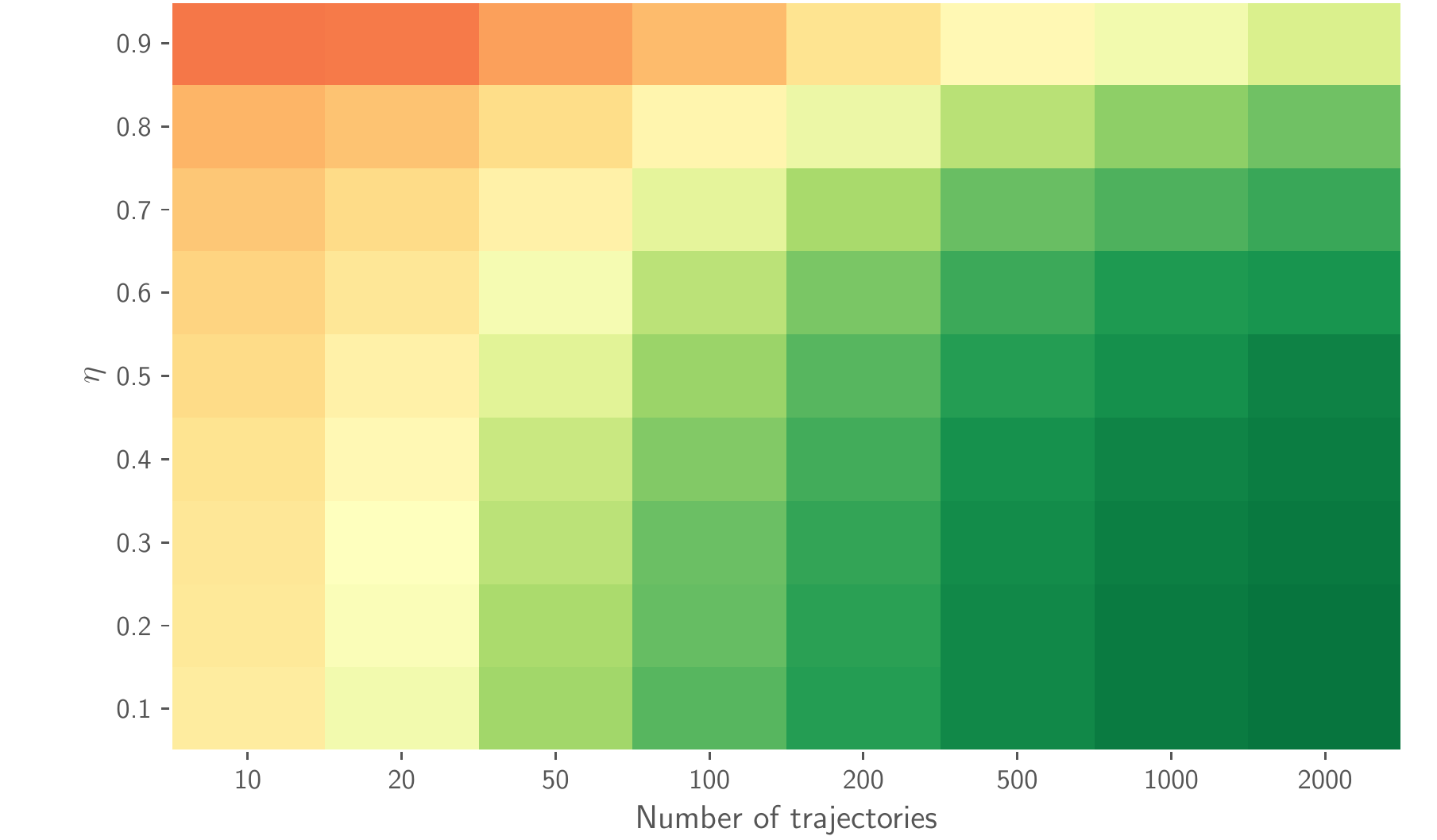}
			\label{fig:heatmap_eta_permille_exact_soft_spibb_eps_2}
		} \\
		\centering
		\subfloat[0.1\%-CVaR: Approx-Soft-SPIBB 1-s ($\epsilon=2$).]{
			\includegraphics[trim = 10pt 140pt 45pt 60pt, clip, width=0.5\textwidth]{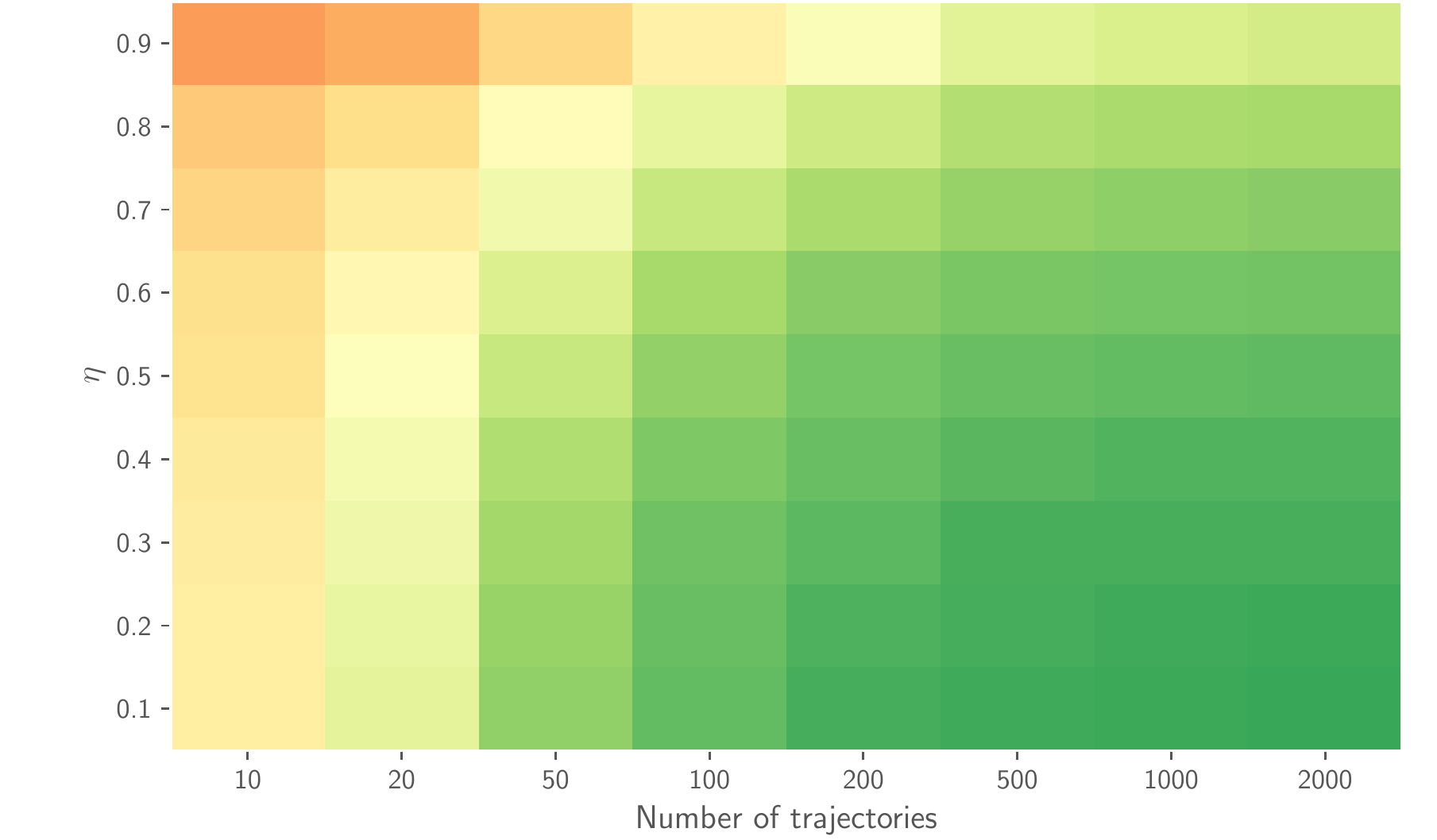}
			\label{fig:heatmap_eta_permille_approx_soft_spibb_1step_eps_2}
		}
		\subfloat[0.1\%-CVaR: Approx-Soft-SPIBB  ($\epsilon=2$).]{
			\includegraphics[trim = 10pt 140pt 45pt 60pt, clip, width=0.5\textwidth]{{figures/heatmaps_eta/heatmap_norm_permille_perf_soft_SPIBB_sort_Q_epsilon=2}.pdf}
			\label{fig:heatmap_eta_permille_approx_soft_spibb_eps_2}
		}
		\caption{Random MDPs: 0.1\%-CVaR performance heatmaps.}
		\label{fig:random_mdps_full_heatmaps_eta}
	\end{figure*}
	
	\begin{figure*}[t!]
		\centering
		\subfloat[Mean: Exact-Soft-SPIBB 1-step ($\eta=0.1$).]{
			\includegraphics[trim = 10pt 140pt 45pt 60pt, clip, width=0.5\textwidth]{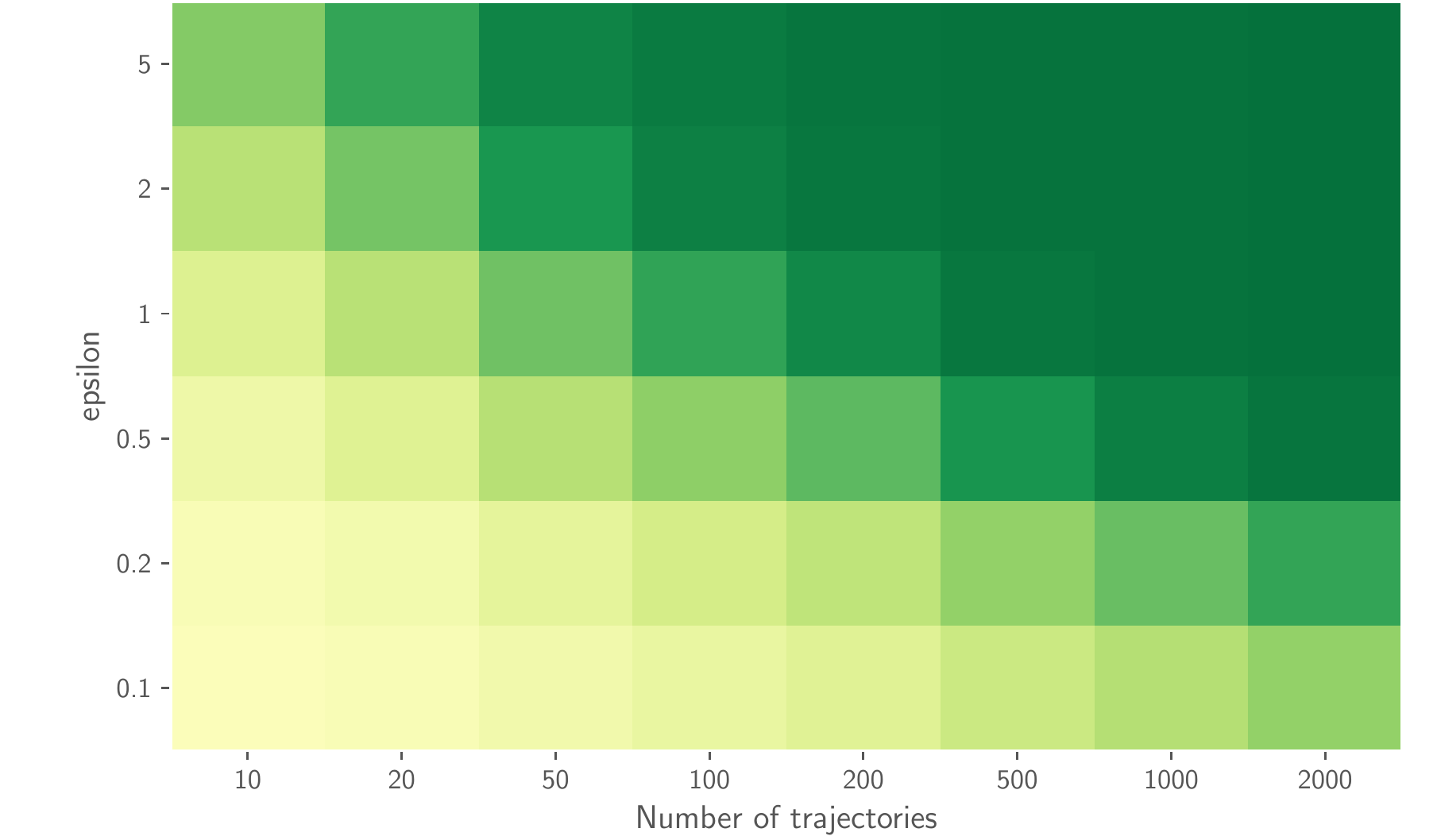}
			\label{fig:heatmap_eps_mean_exact_soft_spibb_1step_eta_0.1}
		}
		\subfloat[Mean: Approx-Soft-SPIBB 1-step ($\eta=0.1$).]{
			\includegraphics[trim = 10pt 140pt 45pt 60pt, clip, width=0.5\textwidth]{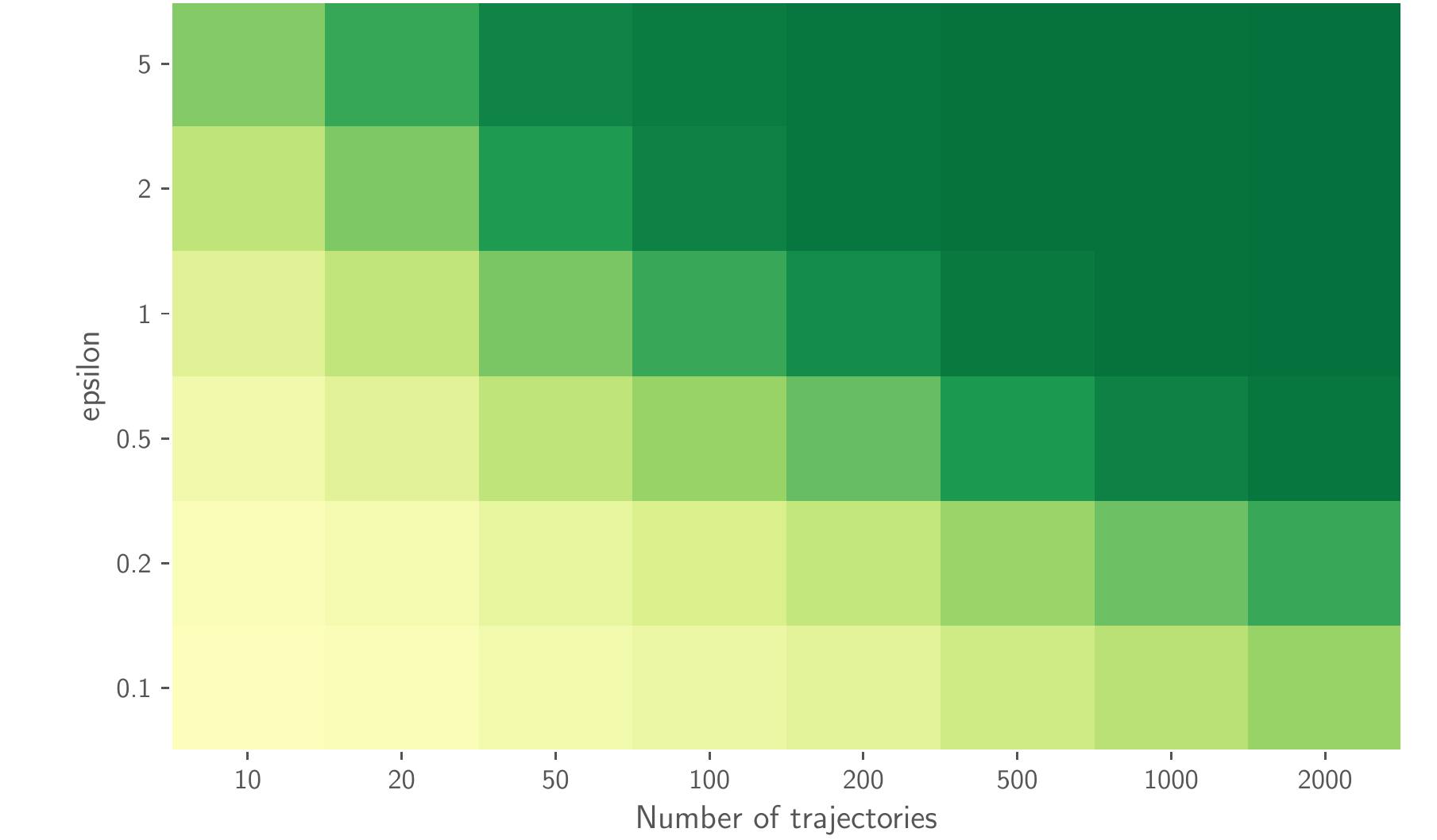}
			\label{fig:heatmap_eps_mean_approx_soft_spibb_1step_eta_0.1}
		} \\
		\centering
		\subfloat[Mean: Exact-Soft-SPIBB ($\eta=0.1$).]{
			\includegraphics[trim = 10pt 140pt 45pt 60pt, clip, width=0.5\textwidth]{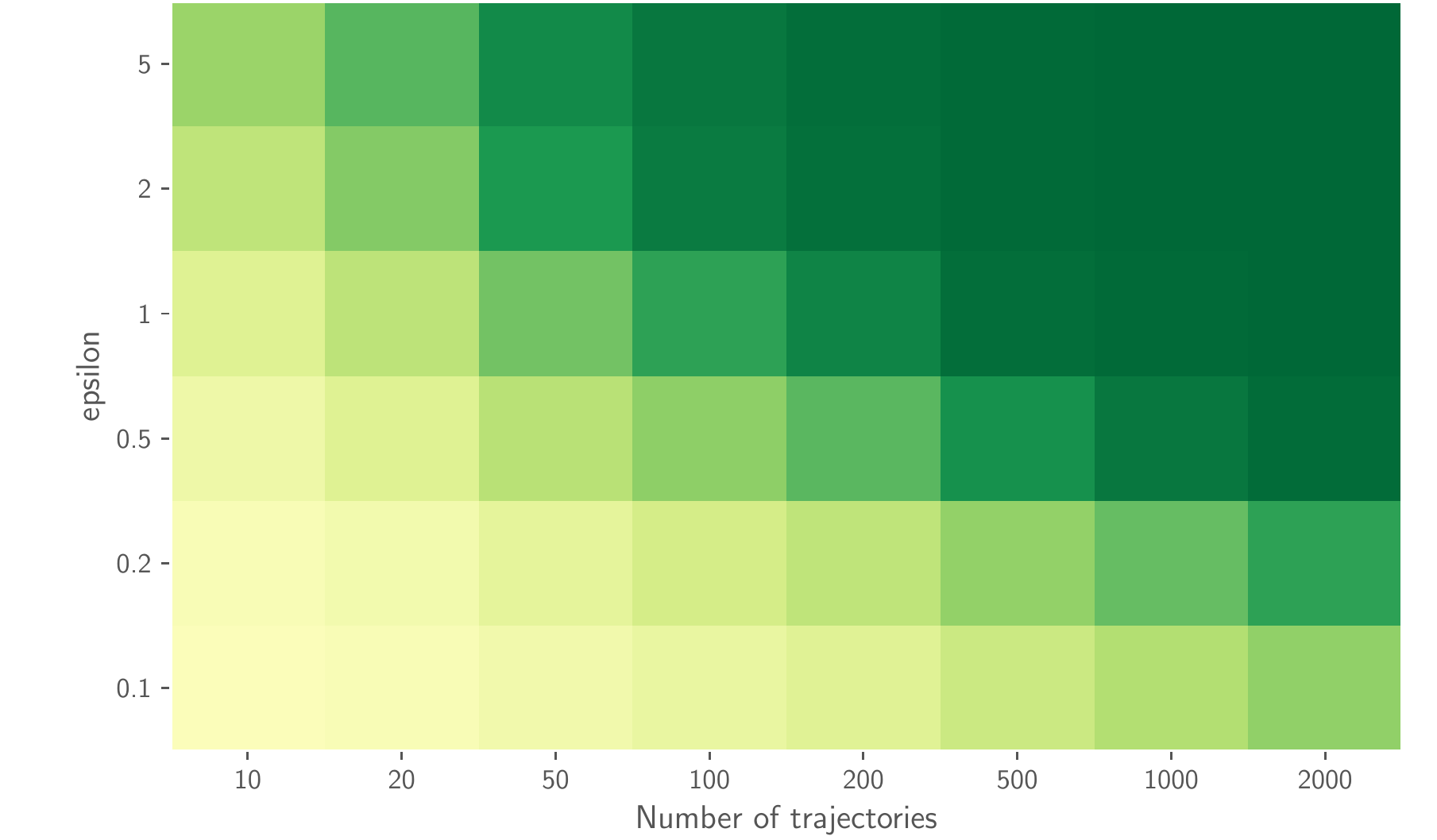}
			\label{fig:heatmap_eps_mean_exact_soft_spibb_eta_0.1}
		}
		\subfloat[Mean: Approx-Soft-SPIBB ($\eta=0.1$).]{
			\includegraphics[trim = 10pt 140pt 45pt 60pt, clip, width=0.5\textwidth]{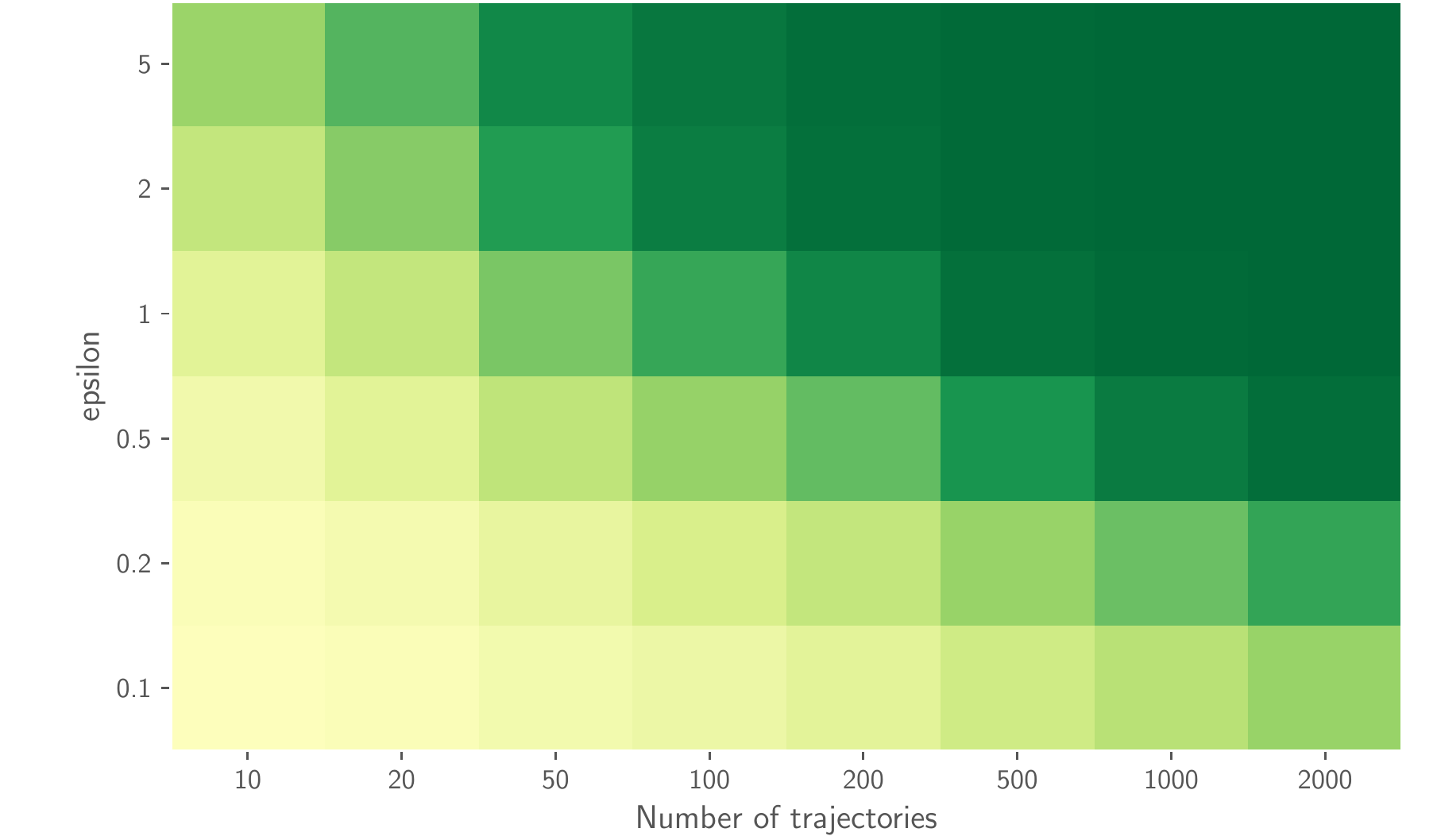}
			\label{fig:heatmap_eps_mean_approx_soft_spibb_eta_0.1}
		} \\
		\centering
		\subfloat[Mean: Exact-Soft-SPIBB 1-step ($\eta=0.9$).]{
			\includegraphics[trim = 10pt 140pt 45pt 60pt, clip, width=0.5\textwidth]{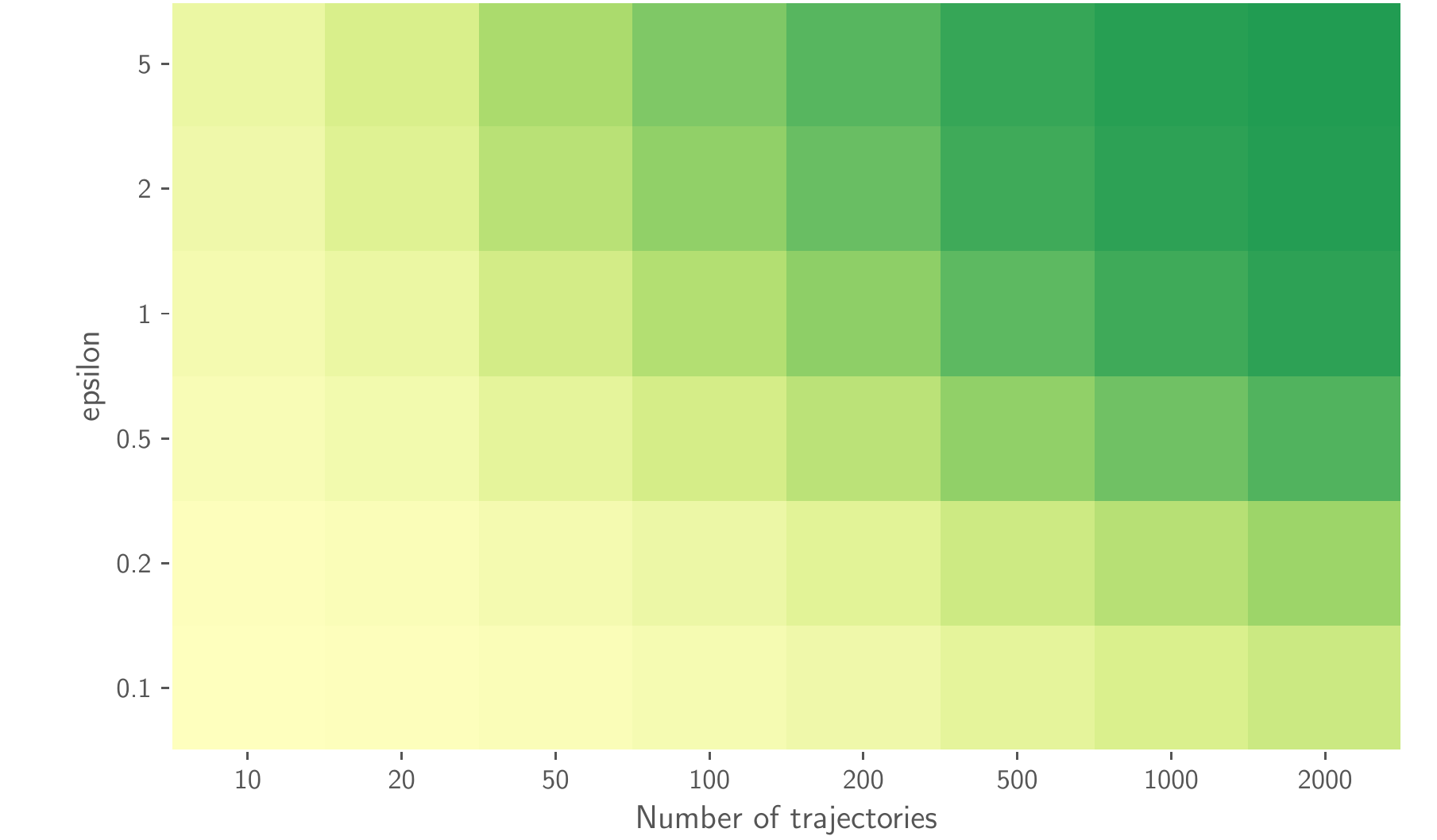}
			\label{fig:heatmap_eps_mean_exact_soft_spibb_1step_eta_0.9}
		}
		\subfloat[Mean: Approx-Soft-SPIBB 1-step ($\eta=0.9$).]{
			\includegraphics[trim = 10pt 140pt 45pt 60pt, clip, width=0.5\textwidth]{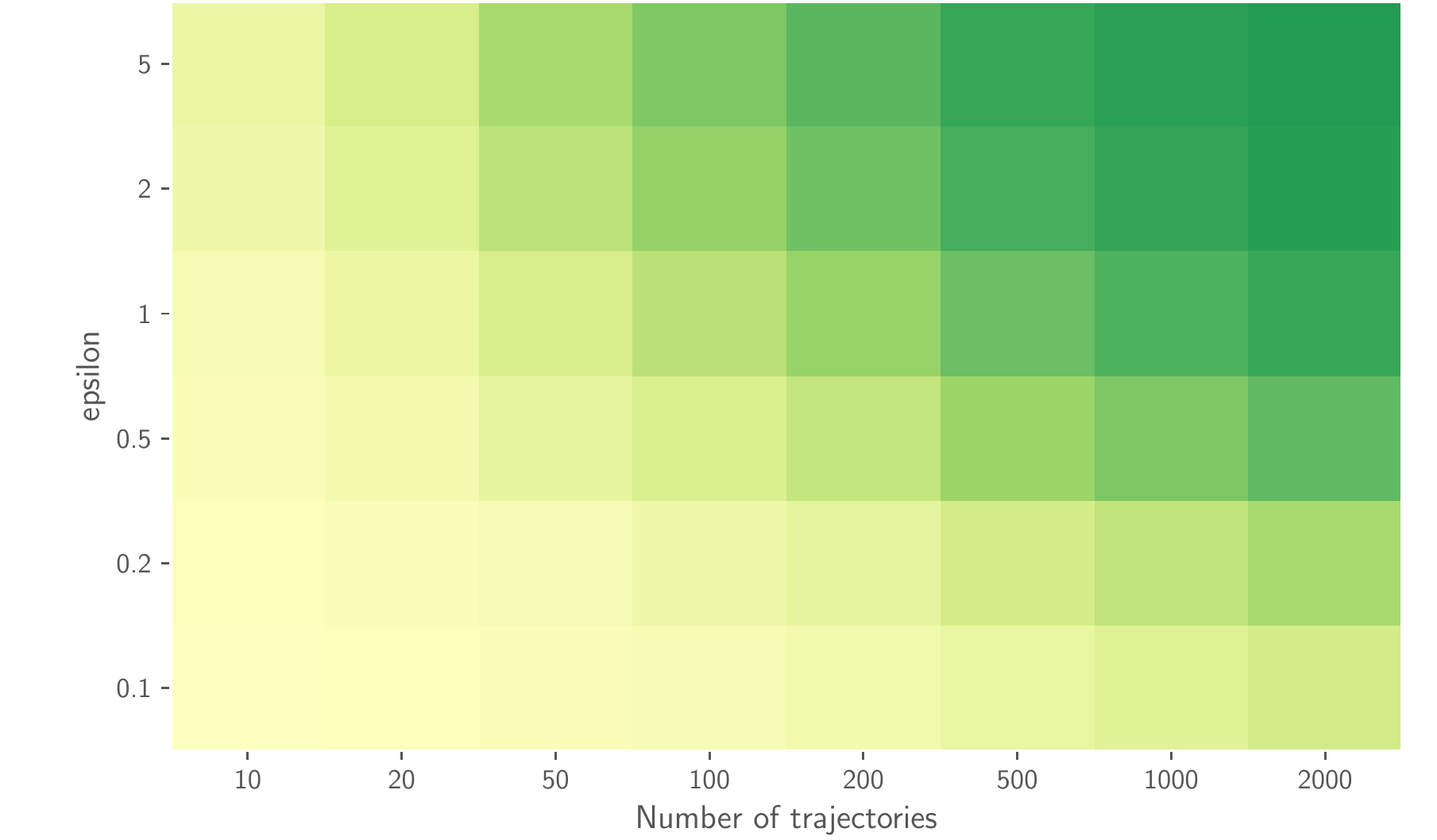}
			\label{fig:heatmap_eps_mean_approx_soft_spibb_1step_eta_0.9}
		} \\
		\centering
		\subfloat[Mean: Exact-Soft-SPIBB ($\eta=0.9$).]{
			\includegraphics[trim = 10pt 140pt 45pt 60pt, clip, width=0.5\textwidth]{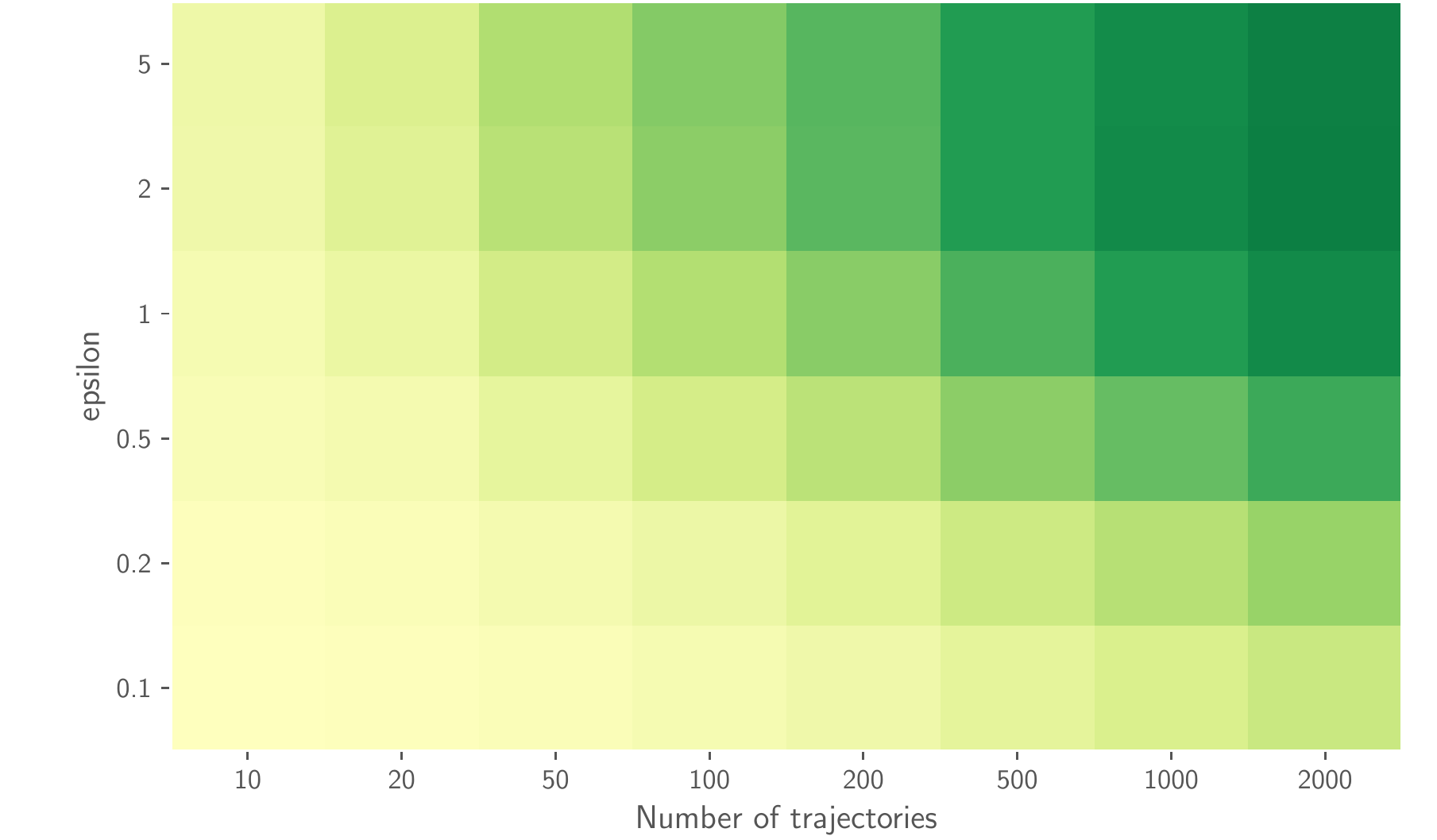}
			\label{fig:heatmap_eps_mean_exact_soft_spibb_eta_0.9}
		}
		\subfloat[Mean: Approx-Soft-SPIBB ($\eta=0.9$).]{
			\includegraphics[trim = 10pt 140pt 45pt 60pt, clip, width=0.5\textwidth]{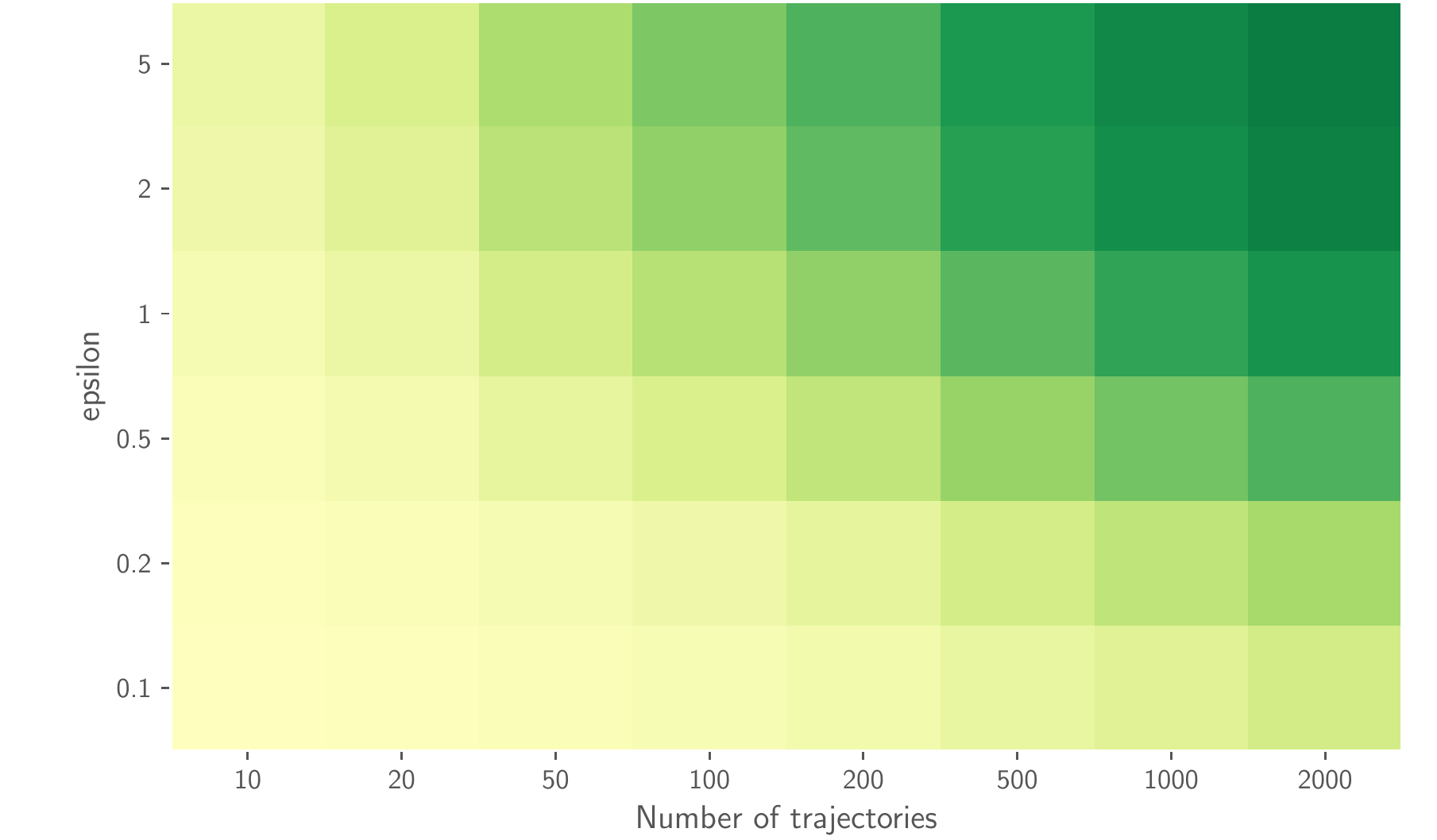}
			\label{fig:heatmap_eps_mean_approx_soft_spibb_eta_0.9}
		}
		\caption{Random MDPs: hyper-parameter mean performance heatmaps for Soft-SPIBB methods under a weak ($\eta = 0.1$) and a strong ($\eta = 0.9$) baseline.}
		\label{fig:random_mdps_full_mean_heatmaps_eta_0.9}
	\end{figure*}

	\begin{figure*}[t!]
		\centering
		\subfloat[1\%-CVaR: Exact-Soft-SPIBB 1-s ($\eta=0.1$).]{
			\includegraphics[trim = 10pt 140pt 45pt 60pt, clip, width=0.5\textwidth]{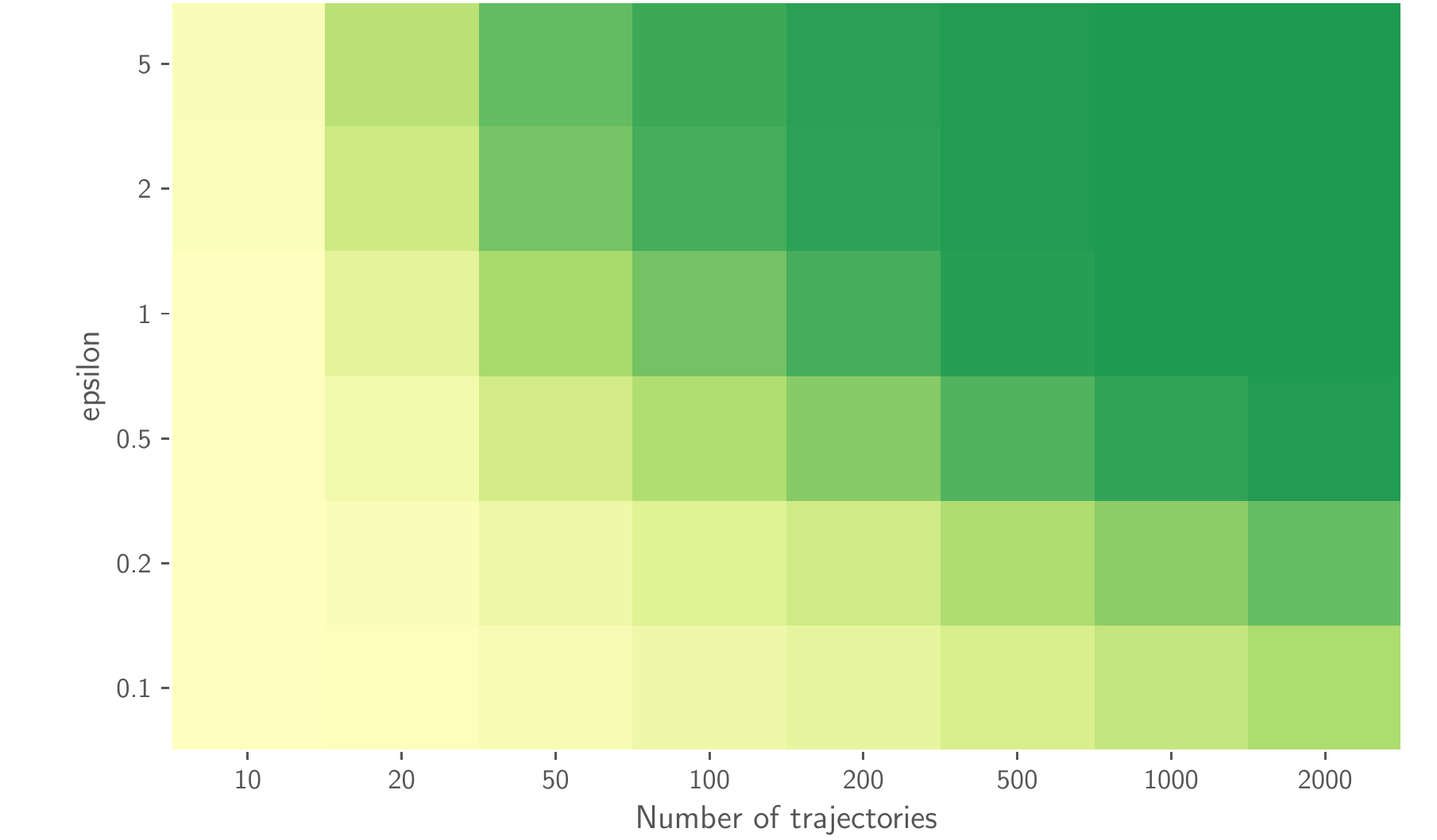}
			\label{fig:heatmap_eps_percentile_exact_soft_spibb_1step_eta_0.1}
		}
		\subfloat[1\%-CVaR: Approx-Soft-SPIBB 1-s ($\eta=0.1$).]{
			\includegraphics[trim = 10pt 140pt 45pt 60pt, clip, width=0.5\textwidth]{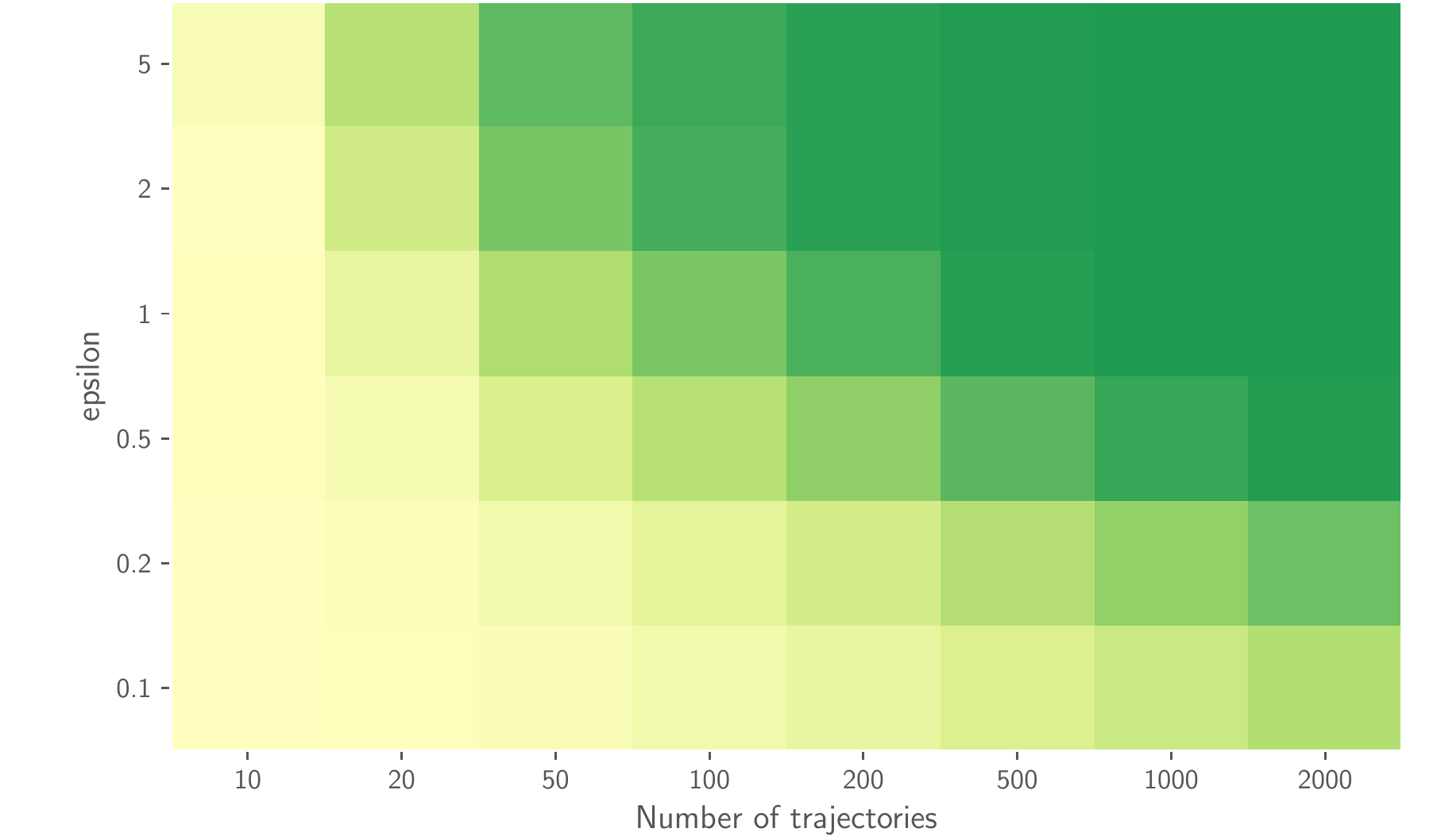}
			\label{fig:heatmap_eps_percentile_approx_soft_spibb_1step_eta_0.1}
		} \\
		\centering
		\subfloat[1\%-CVaR: Exact-Soft-SPIBB ($\eta=0.1$).]{
			\includegraphics[trim = 10pt 140pt 45pt 60pt, clip, width=0.5\textwidth]{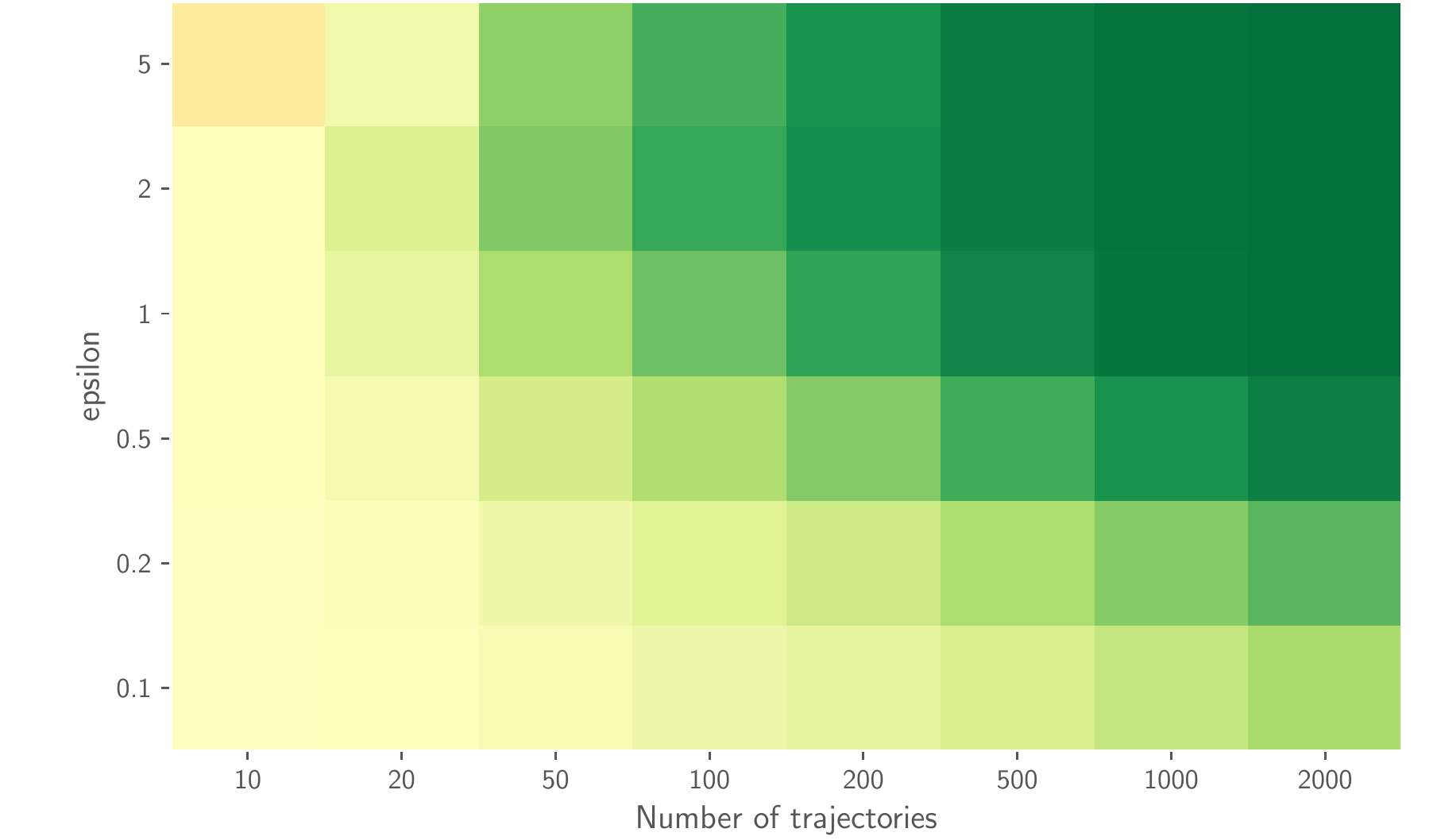}
			\label{fig:heatmap_eps_percentile_exact_soft_spibb_eta_0.1}
		}
		\subfloat[1\%-CVaR: Approx-Soft-SPIBB ($\eta=0.1$).]{
			\includegraphics[trim = 10pt 140pt 45pt 60pt, clip, width=0.5\textwidth]{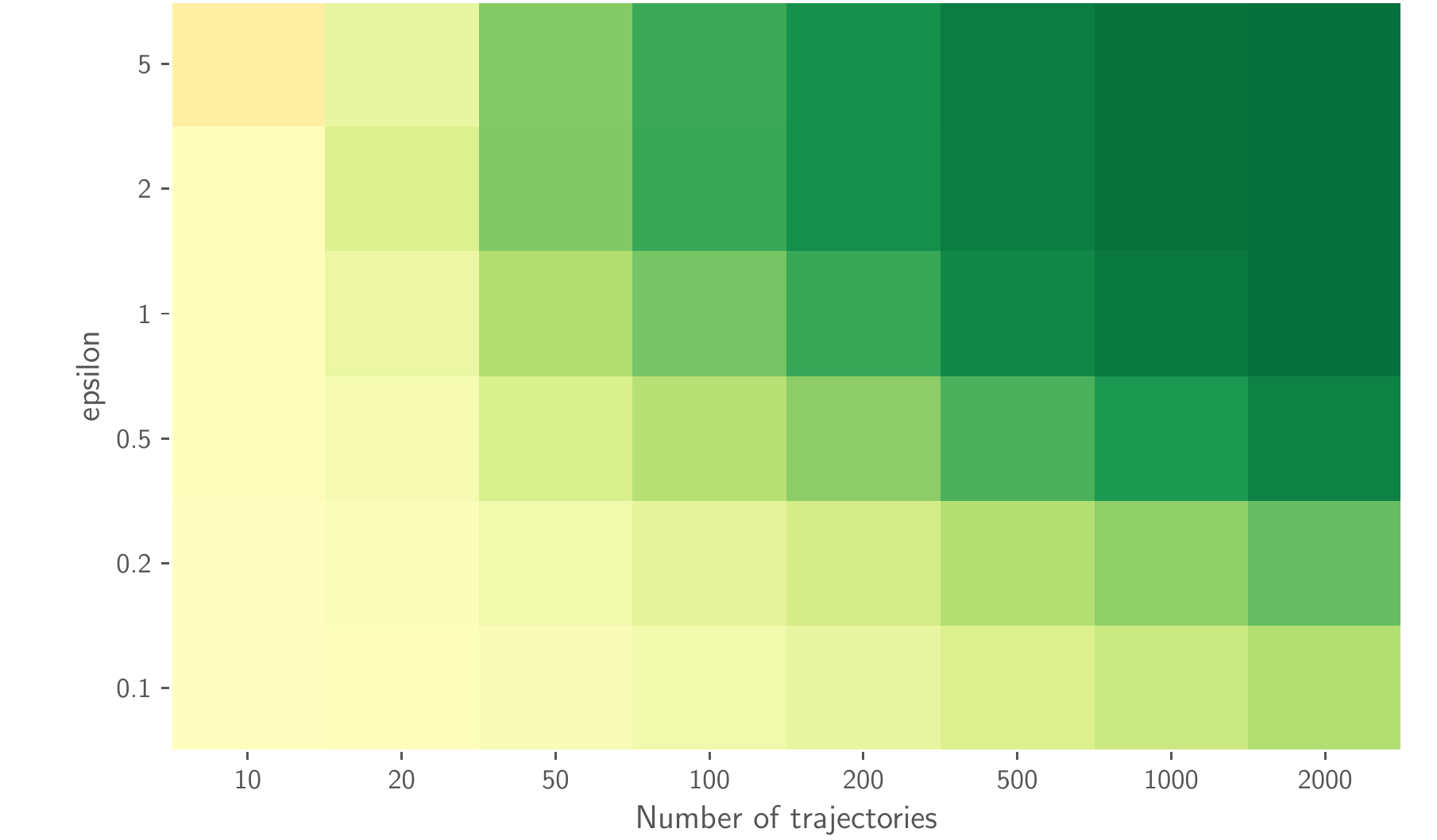}
			\label{fig:heatmap_eps_percentile_approx_soft_spibb_eta_0.1}
		}\\
		\centering
		\subfloat[1\%-CVaR: Exact-Soft-SPIBB 1-s ($\eta=0.9$).]{
			\includegraphics[trim = 10pt 140pt 45pt 60pt, clip, width=0.5\textwidth]{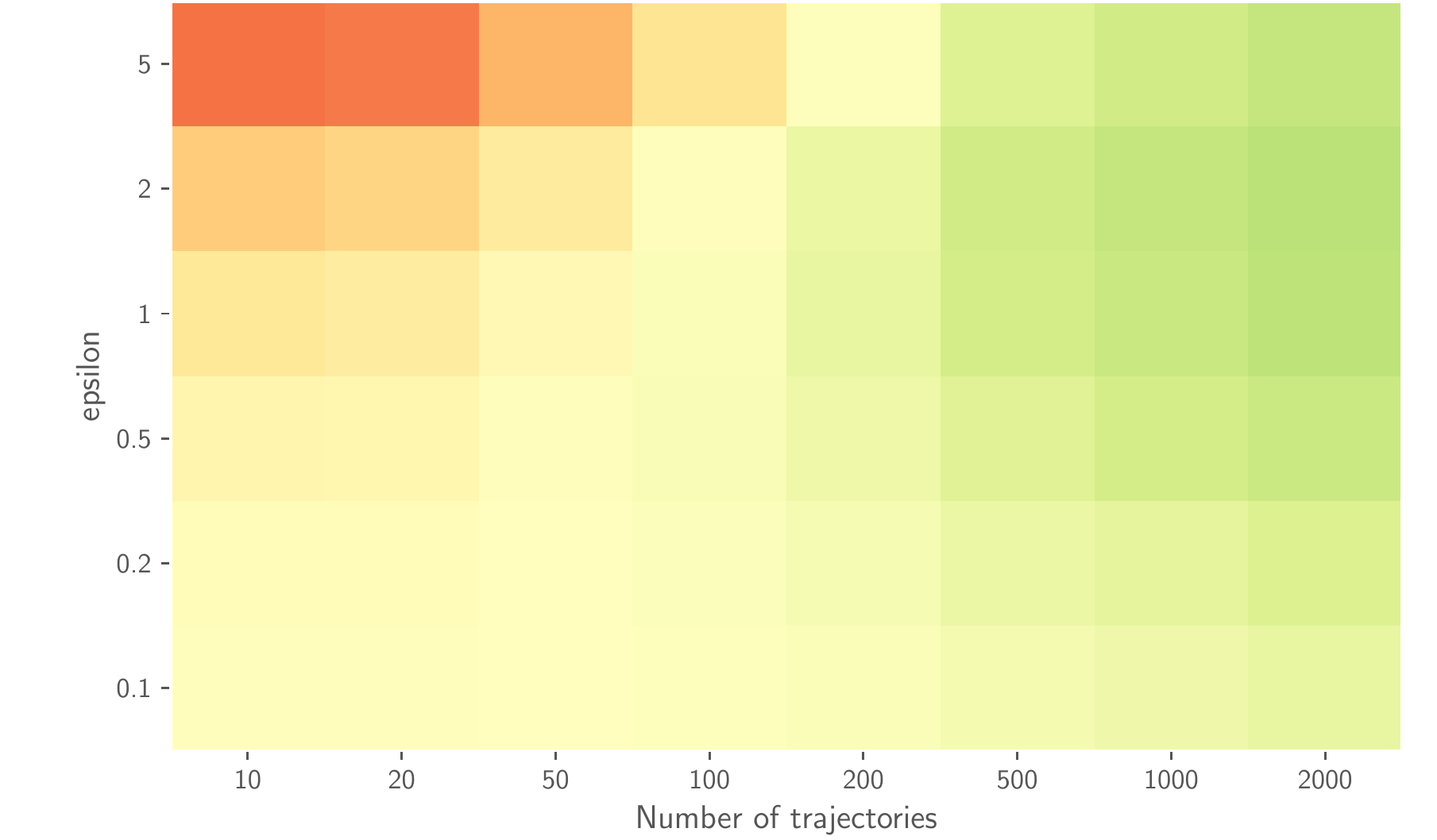}
			\label{fig:heatmap_eps_percentile_exact_soft_spibb_1step_eta_0.9}
		}
		\subfloat[1\%-CVaR: Approx-Soft-SPIBB 1-s ($\eta=0.9$).]{
			\includegraphics[trim = 10pt 140pt 45pt 60pt, clip, width=0.5\textwidth]{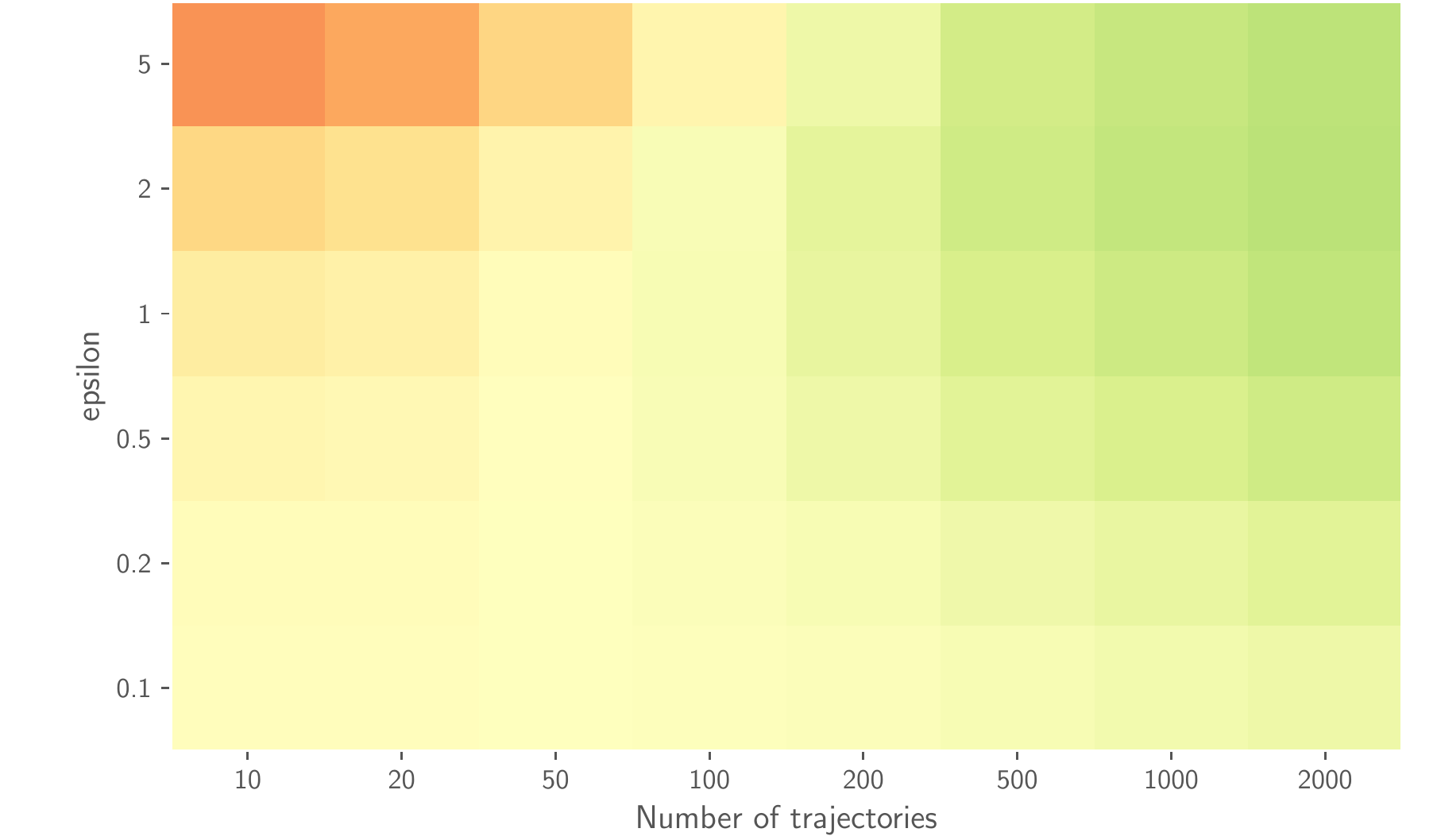}
			\label{fig:heatmap_eps_percentile_approx_soft_spibb_1step_eta_0.9}
		} \\
		\centering
		\subfloat[1\%-CVaR: Exact-Soft-SPIBB ($\eta=0.9$).]{
			\includegraphics[trim = 10pt 140pt 45pt 60pt, clip, width=0.5\textwidth]{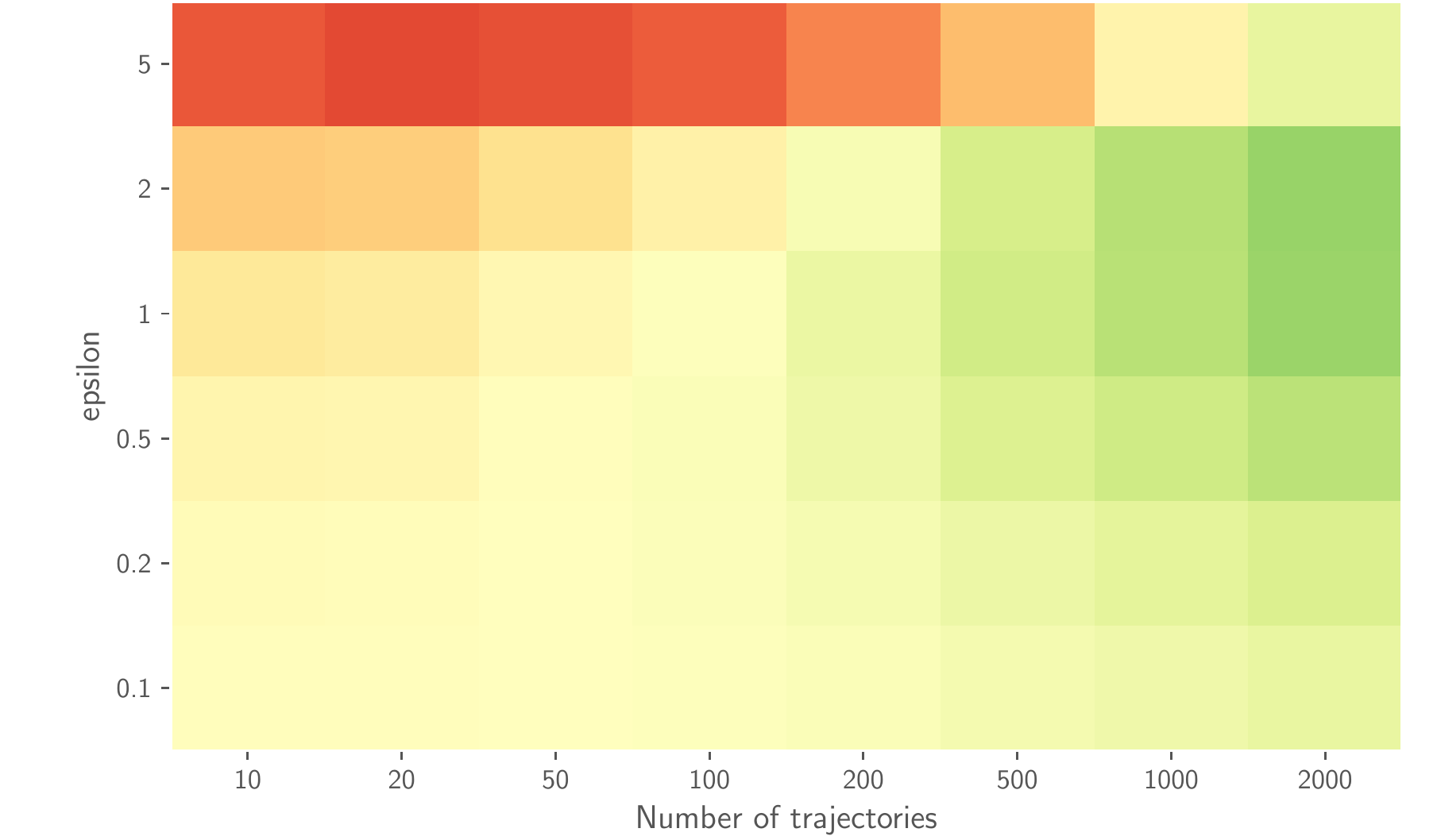}
			\label{fig:heatmap_eps_percentile_exact_soft_spibb_eta_0.9}
		}
		\subfloat[1\%-CVaR: Approx-Soft-SPIBB ($\eta=0.9$).]{
			\includegraphics[trim = 10pt 140pt 45pt 60pt, clip, width=0.5\textwidth]{{figures/heatmaps_eps/heatmap_per_arg_norm_percentile_perf_soft_SPIBB_sort_Q_ratio=0.9}.pdf}
			\label{fig:heatmap_eps_percentile_approx_soft_spibb_eta_0.9}
		}
		\caption{Random MDPs: hyper-parameter 1\%-CVaR performance heatmaps for Soft-SPIBB methods under a weak ($\eta = 0.1$) and a strong ($\eta = 0.9$) baseline.}
		\label{fig:random_mdps_full_percentile_heatmaps_eta_0.1}
	\end{figure*}

	\begin{figure*}[t!]
		\centering
		\subfloat[Mean: RaMDP ($\eta=0.1$).]{
			\includegraphics[trim = 10pt 140pt 45pt 60pt, clip, width=0.5\textwidth]{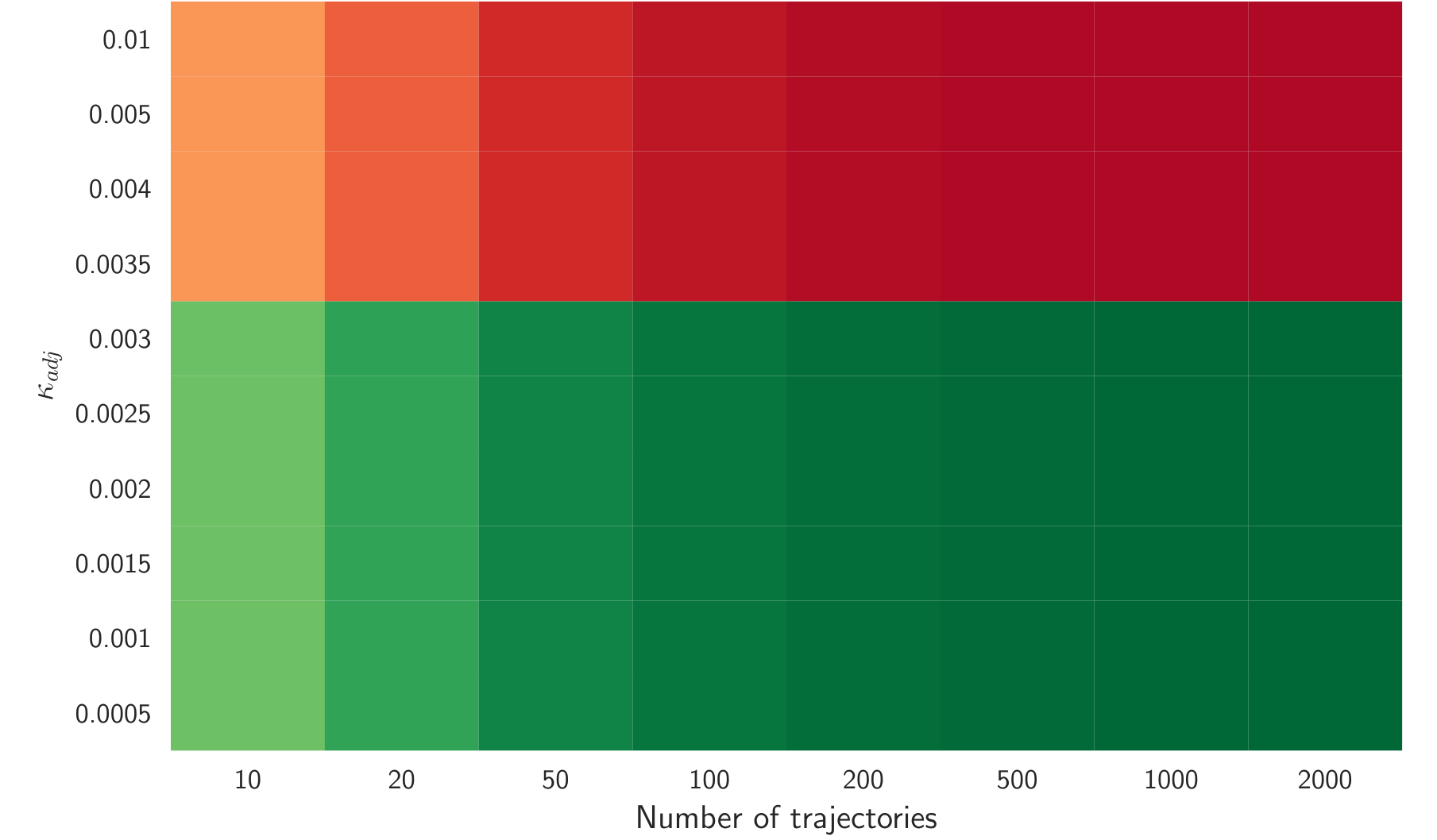}
		}
		\subfloat[1\%-CVaR: RaMDP ($\eta=0.1$).]{
			\includegraphics[trim = 10pt 140pt 45pt 60pt, clip, width=0.5\textwidth]{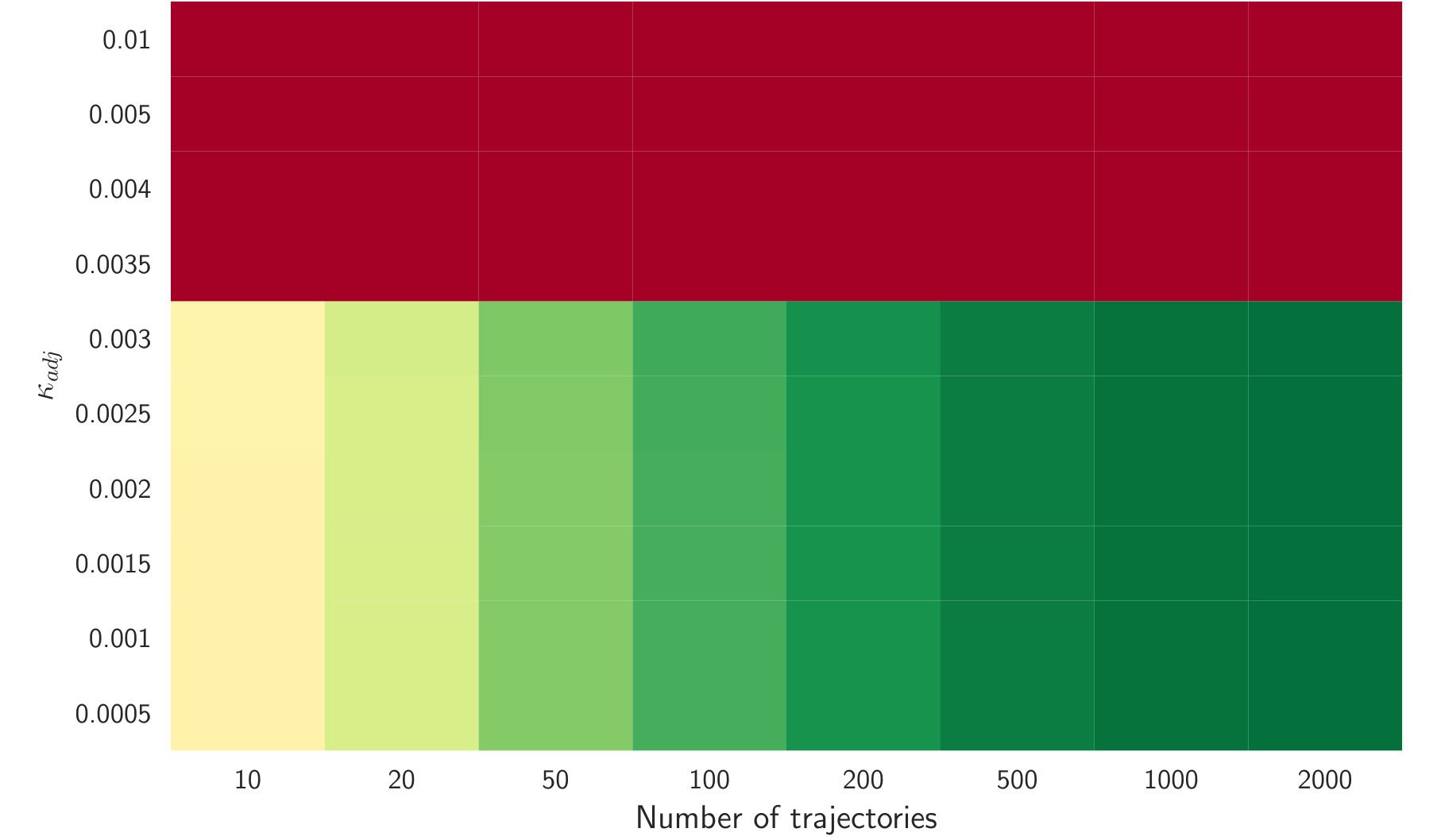}
		} \\
		\centering
		\subfloat[Mean: RaMDP ($\eta=0.4$).]{
			\includegraphics[trim = 10pt 140pt 45pt 60pt, clip, width=0.5\textwidth]{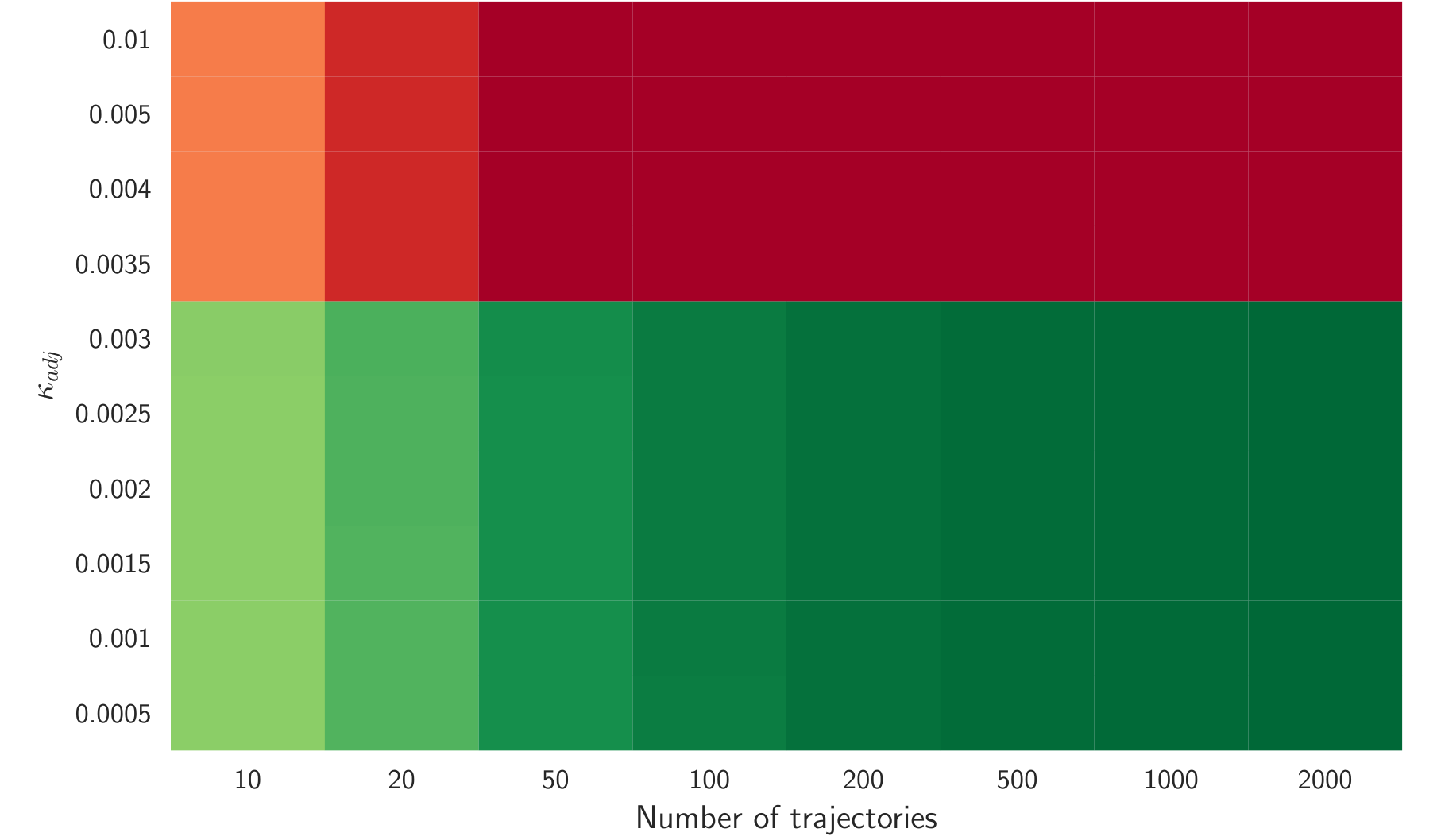}
		}
		\subfloat[1\%-CVaR: RaMDP ($\eta=0.4$).]{
			\includegraphics[trim = 10pt 140pt 45pt 60pt, clip, width=0.5\textwidth]{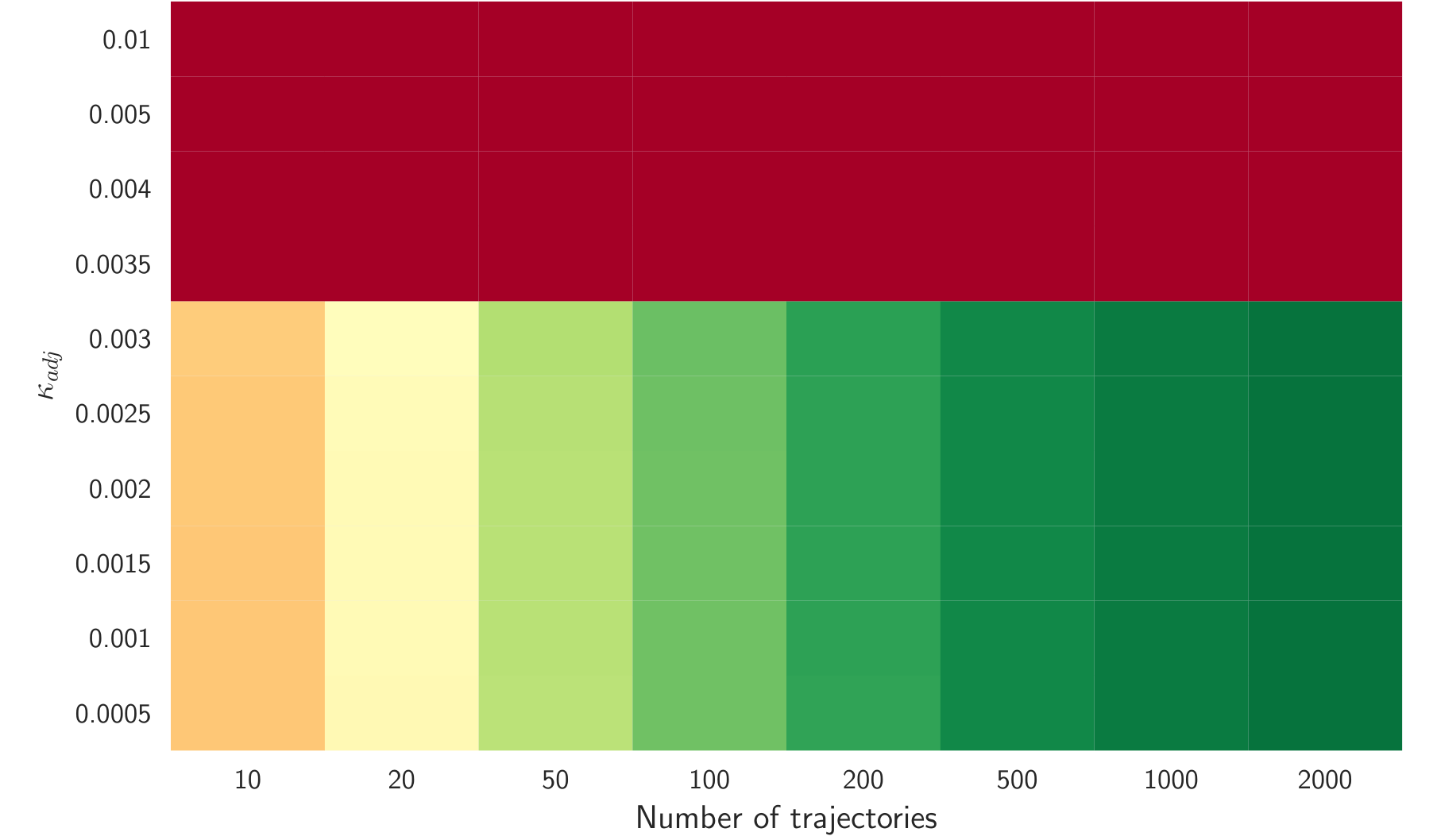}
		} \\
		\centering
		\subfloat[Mean: RaMDP ($\eta=0.6$).]{
			\includegraphics[trim = 10pt 140pt 45pt 60pt, clip, width=0.5\textwidth]{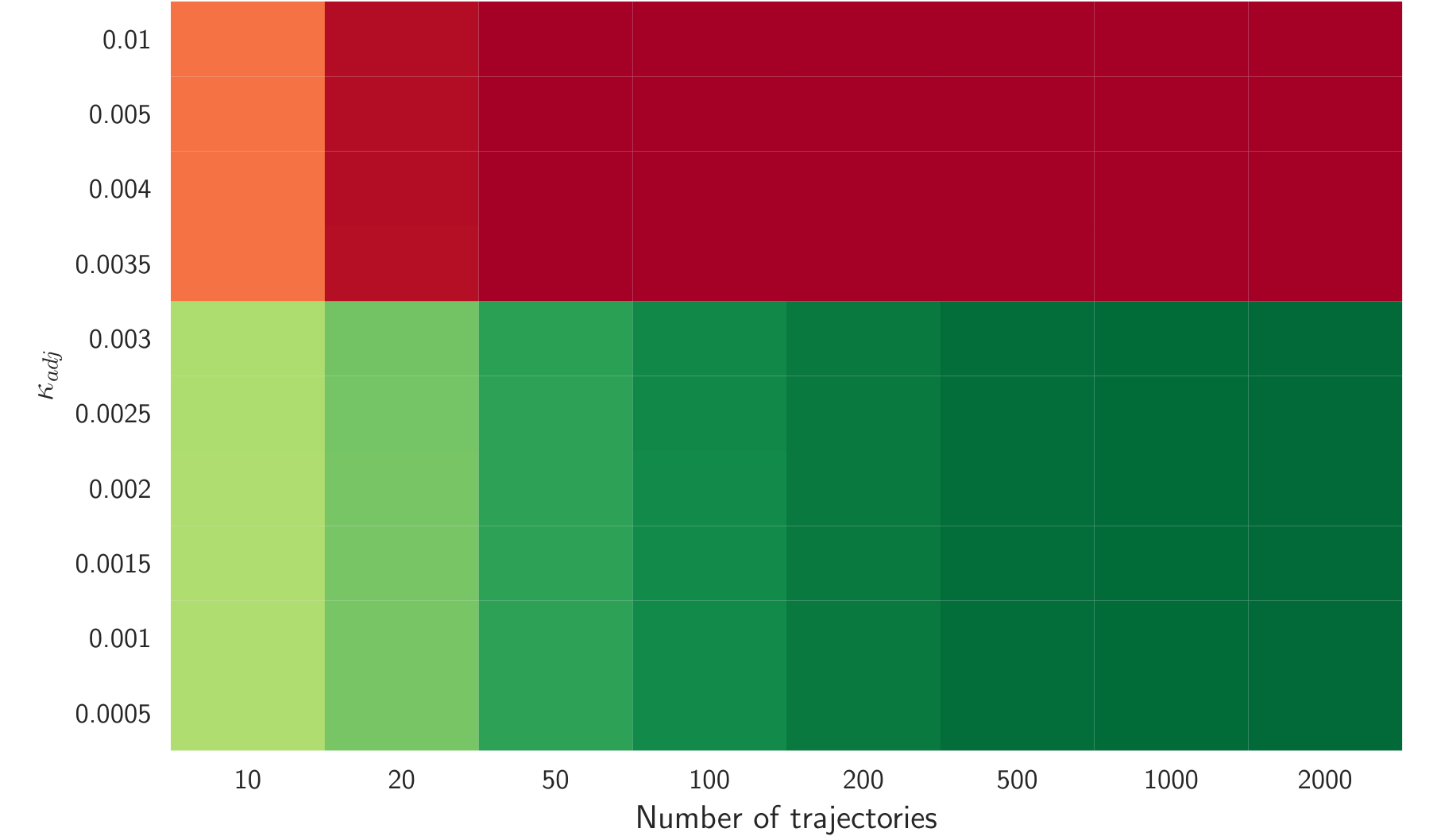}
		}
		\subfloat[1\%-CVaR: RaMDP ($\eta=0.6$).]{
			\includegraphics[trim = 10pt 140pt 45pt 60pt, clip, width=0.5\textwidth]{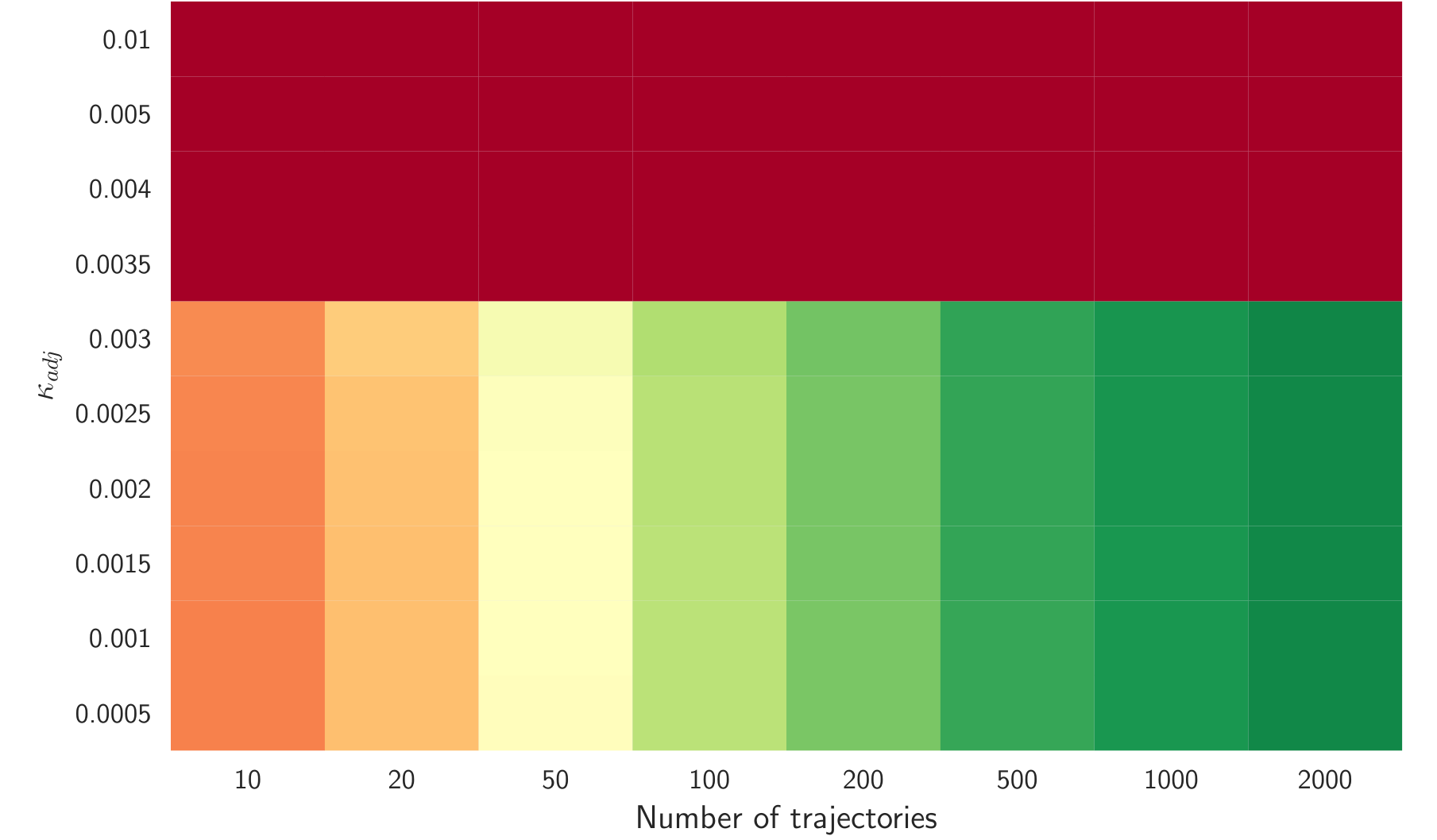}
		} \\
		\centering
		\subfloat[Mean: RaMDP ($\eta=0.9$).]{
			\includegraphics[trim = 10pt 140pt 45pt 60pt, clip, width=0.5\textwidth]{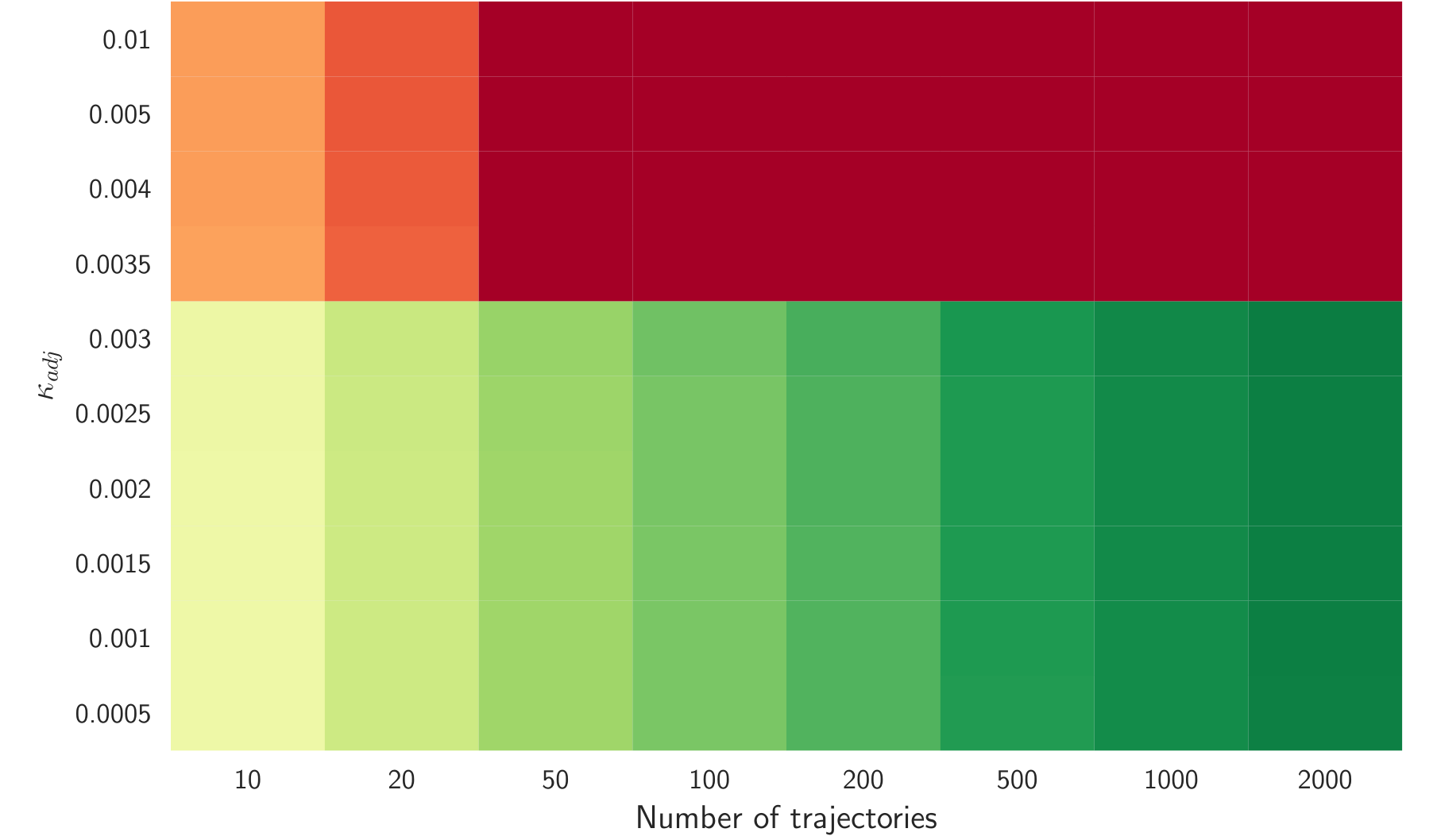}
		}
		\subfloat[1\%-CVaR: RaMDP ($\eta=0.9$).]{
			\includegraphics[trim = 10pt 140pt 45pt 60pt, clip, width=0.5\textwidth]{{figures/RaMDP/heatmap_per_arg_norm_percentile_perf_RaMDP_ratio=0.9}.pdf}
		} \\
		\caption{Random MDPs: hyper-parameter mean and 1\%-CVaR performance heatmaps for RaMDP methods under a weak ($\eta = 0.1$), a medium weak ($\eta = 0.4$), a medium strong ($\eta = 0.6$), and a strong ($\eta = 0.9$) baseline.}
		\label{fig:random_mdps_RaMDP}
	\end{figure*}
	
	\clearpage
	
	\section{Helicopter experiment details}
	\label{sup:expe_helicopter}
	
	
	\subsection{Details about the helicopter environment}
	\label{sup:dummy-parameters}
	See \citep[Appendix D.1]{Laroche2019}.

	\subsection{Helicopter experimental process}
	
	See Pseudo-code~\ref{alg:helicopter_benchmark}.
	\begin{pseudocode}[ht]
		\caption{Helicopter experimental process}
		\KwIn{List of algorithms}
		\KwIn{List of hyper-parameter values}
		\KwIn{List of dataset sizes}
		\BlankLine
		\RepTimes{$20$}{
		    \For{each dataset size}{
		            
	            Generate a dataset.
	            
	            Compute the pseudo-counts.
	            
	            \RepTimes{$15$}{
		            
        		    \For{each algorithm}{
            		    \For{each hyper-parameter value}{
            		    
        		            Train a policy.
        		                    
                            Evaluate the trained policy.
        		                    
                            Record the performance of the trained policy.
                        }
    		        }
    		    }
		    }
		}
		
		\label{alg:helicopter_benchmark}
	\end{pseudocode}
	
	\subsection{Details about the DQN implementations}
	\label{sup:helicopter_training}
    The batch version of DQN simply consists in replacing the experience replay buffer by the dataset we are training on. Effectively, we are not sampling from the environment anymore but from the transitions collected a priori following the baseline. RaMDP follows the same principle with a target modified according to the reward hacking:
    \begin{align}
	    y^{(i)}_j &= r_j - \cfrac{\kappa}{\sqrt{\tilde{N}_{\mathcal{D}}(x_j,a_j)}} + \gamma \max_{a'\in\mathcal{A}} Q^{(i)}(x_j',a'),
	\end{align}
    where $\tilde{N}_{\mathcal{D}}(x_j,a_j)$ is the pseudo-count. For the sake of simplicity and to be able to repeat several runs of each experiment efficiently, instead of applying pseudo-count methods from the literature~\citep{Bellemare2016,Fox2018,Burda2019}, we consider here a pseudo-count heuristic based on the Euclidean state-distance, and a task where it makes sense to do so. The pseudo-count of a state-action $(x,a)$ is defined as the sum of its similarity with the state-action pairs $(x_i,a_i)$ found in the dataset. The similarity between $(x,a)$ and $(x_i,a_i)$ is equal to 0 if $a_i\neq a$, and to $\max(0,1-5d(x,x_i))$ otherwise, where $d(\cdot,\cdot)$ is the Euclidean distance between two states:
    \begin{align}
	    \tilde{N}_{\mathcal{D}}(x,a) &= \sum_{\langle x_j,a_j=a,r_j,x'_j\rangle\in\mathcal{D}} \max(0,1-5d(x,x_j)),
	\end{align}
    
    The same pseudo-counts were used for all algorithms using it. In particular, they allowed to compute the error function at each state-action pair as (the cardinal of continuous state spaces cannot be used as in the tabular case):
    \begin{align}
        e_P(x,a) = \frac{1}{\sqrt{\tilde{N}_{\mathcal{D}}(x,a)}}.
    \end{align} 
    The same general methodology applies for Soft-SPIBB, according to the definitions from Section \ref{subsec:Soft-SPIBB-DQN}. Although, SPIBB is a special case of Soft-SPIBB, its implementation has been performed independently in order to be more computationally efficient. For SPIBB, the targets we are using for our $Q$-values updates verify the following modified Bellman equation:
    \begin{align}
	    y^{(i)}_j &= r_j + \gamma \max_{\pi\in\Pi_b} \sum_{a'\in\mathcal{A}} \pi(a'|x_j') Q^{(i)}(x_j',a')  \\
	    &= r_j + \gamma\sum_{(x_j',a')\in\mathfrak{B}} \pi_b(a'|x_j') Q^{(i)}(x_j',a') \nonumber \\
	    & \quad\quad + \gamma\left(\sum_{(x_j',a')\notin\mathfrak{B}} \pi_b(a'|x_j')\right) \max_{(x_j',a')\notin\mathfrak{B}} Q^{(i)}(x_j',a') \label{eq:spibb-DQN-app}
	\end{align}
    We notice in particular that when $\mathfrak{B} = \emptyset$ the targets fall back to the traditional Bellman ones.
    We used the now classic target network trick \citep{Mnih2015}, combined with Double-DQN \citep{HasseltGS15}.

    The network used for the baseline and for the algorithms in the benchmark is a fully connected network with 3 hidden layers of $32$, $128$ and $32$ neurons, initialized using he\_uniform \citep{he2015delving}. The network has $9$ outputs corresponding to the $Q$-values of the $9$ actions in the game.
    We train the $Q$-networks with RMSProp \citep{tieleman2012lecture} with a momentum of $0.95$ and $\epsilon = 10^{-7}$ on mini-batches of size $32$. The learning rate is initialized at $0.01$ and is annealed every $20$k transitions or every pass on the dataset, whichever is larger. The networks are trained for $2$k passes on the dataset, and are fully converged by that time. The models are trained with Pytorch \citep{pytorch}. The policy is tested for $1$k steps at the end of training, with the initial states of each trajectory sampled as described in Section \ref{sup:dummy-parameters}.

	\subsection{Additional experimental results}
	\label{sup:helicopter_more_results}
	
	See Table~\ref{tab:helicopter-numerical-results} and Figure~\ref{fig:Helicopter_experiments_sup}.
	
	\begin{table}[ht]
		\caption{Helicopter: numerical (mean and 10\%-CVaR) results with optimal hyper-parameters} 
		\centering
		\setlength\tabcolsep{3pt}
		\def\arraystretch{1.6}
		\begin{tabular}{|c|c|c|c|c|c|c|c|c|c|c|}
			\hline
			$|\mathcal{D}|$ & \multicolumn{2}{c|}{Baseline} & \multicolumn{2}{c|}{DQN} & \multicolumn{2}{c|}{RaMDP-DQN} & \multicolumn{2}{c|}{$\Pi_b$-SPIBB} & \multicolumn{2}{c|}{Soft-SPIBB} \\
			\hline
			 & Mean & CVaR & Mean & CVaR & Mean & CVaR & Mean & CVaR & Mean & CVaR  \\ 
			\hline
			1,000 & 2.30 & 1.59 & -0.91 & -1.00 & N/A & N/A & \textbf{2.74} & \textbf{1.91} & \textbf{2.72} & \textbf{1.88} \\
			3,000 & 2.30 & 1.59 & -0.72 & -1.00 & 0.85 & -0.70 & \textbf{3.19} & \textbf{2.37} & \textbf{3.23} & \textbf{2.38} \\
			6,000 & 2.30 & 1.59 & 0.07 & -1.00 & N/A & N/A & 3.30 & 2.41 & \textbf{3.39} & \textbf{2.54} \\
			10,000 & 2.30 & 1.59 & 0.22 & -1.00 & 3.21 & 1.65 & 3.39 & 2.46 & \textbf{3.61} & \textbf{2.69} \\
			\hline
		\end{tabular}
		\normalsize
		\label{tab:helicopter-numerical-results}
	\end{table}
	
	\begin{figure}[ht]
        \centering
    	\subfloat[Helicopter benchmark with $|\mathcal{D}|=6,000$]{
            \includegraphics[trim = 5pt 5pt 15pt 5pt, clip, width=\columnwidth,left]{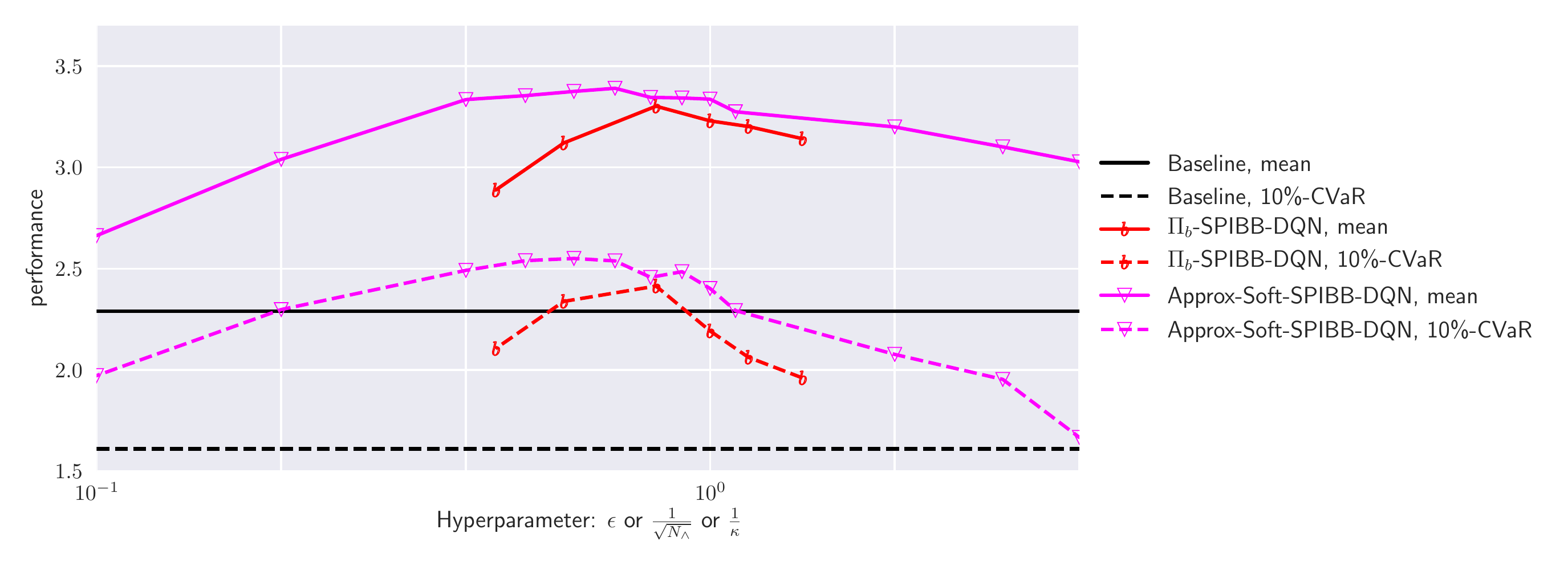}
            \label{fig:soft_spibb_dqn_6k}
    	}\\
        \centering
    	\subfloat[Helicopter benchmark with $|\mathcal{D}|=1,000$]{
            \includegraphics[trim = 5pt 5pt 15pt 5pt, clip, width=\columnwidth]{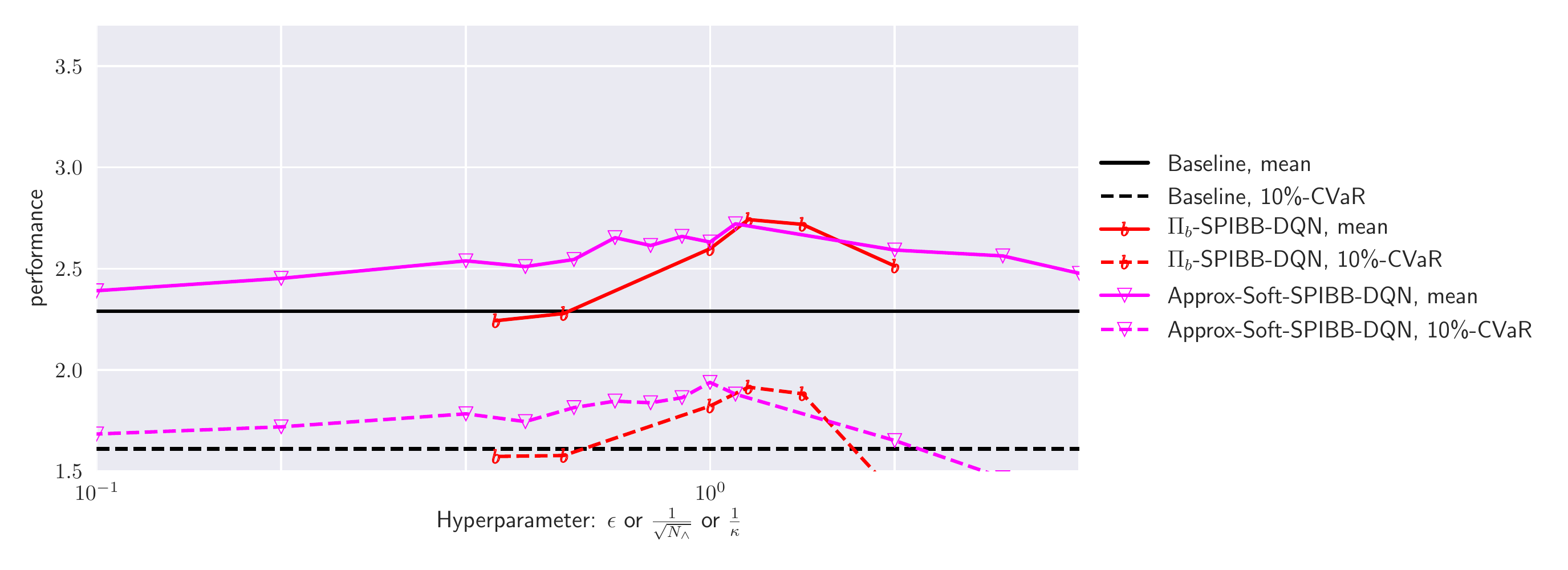}
            \label{fig:soft_spibb_dqn_1k}
    	}
        \caption{Helicopter: mean and 10\%-CVaR as a function of the hyper-parameter}
        \label{fig:Helicopter_experiments_sup}
    \end{figure}

	\clearpage
	\section{Reproducible, reusable, and robust Reinforcement Learning}
	This paper's objective is to improve the robustness and the reliability of Reinforcement Learning algorithms. Inspired from Joelle Pineau's talk at NeurIPS 2018 about reproducible, reusable, and robust Reinforcement Learning\footnote{https://nips.cc/Conferences/2018/Schedule?showEvent=12486}, we intend to also make our work reusable and reproducible.
	
    \newlist{arrowlist}{itemize}{1}
    \setlist[arrowlist]{label=$\Rightarrow$}
	\subsection{Pineau's checklist (slide 33)}
	For all algorithms presented, check if you include:
	\begin{itemize}
	    \item A clear description of the algorithm.
            \begin{arrowlist}
                \item See Section \ref{sec:spisbb}.
            \end{arrowlist}
	    \item An analysis of the complexity (time, space, sample size) of the algorithm.
            \begin{arrowlist}
                \item We provide formal analysis for the complexity in Section \ref{subsec:complexity}, and a rapid empirical assessment in Figure \ref{subfig:complarge} for the finite MDP analysis. For Soft-SPIBB-DQN, we empirically observed a $50\%$ increase in runtime compared to SPIBB-DQN.
            \end{arrowlist}
	    \item A link to downloadable source code, including all dependencies.
            \begin{arrowlist}
                \item The code is being made available. The code will be made public on acceptation. See Section \ref{sup:code}.
            \end{arrowlist}
	\end{itemize}
	
    For any theoretical claim, check if you include:
	\begin{itemize}
    	\item A statement of the result.
            \begin{arrowlist}
                \item See Theorems \ref{th:1-step}, \ref{thm:SSPIBB}, and \ref{thm:complexity}.
            \end{arrowlist}
    	\item A clear explanation of any assumptions.
            \begin{arrowlist}
                \item The main assumption for Theorem \ref{thm:SSPIBB} is formalized as Assumption \ref{ass:kappa}. See Section \ref{sec:spisbb} for discussion.
            \end{arrowlist}
    	\item A complete proof of the claim.
            \begin{arrowlist}
                \item See Section \ref{sup:proofs}.
            \end{arrowlist}
	\end{itemize}
	
    For all figures and tables that present empirical results, check if you include:
	\begin{itemize}
    	\item A complete description of the data collection process, including sample size.
            \begin{arrowlist}
                \item See Sections \ref{sec:emp}, \ref{sup:MDPgen}, and \ref{sup:dummy-parameters}.
            \end{arrowlist}
    	\item A link to downloadable version of the dataset or simulation environment.
            \begin{arrowlist}
                \item See Section \ref{sup:code}.
            \end{arrowlist}
    	\item An explanation of how sample were allocated for training / validation / testing.
            \begin{arrowlist}
                \item The complete dataset is used for training. There is no need for validation set. Testing is performed in the true environment.
            \end{arrowlist}
    	\item An explanation of any data that was excluded.
            \begin{arrowlist}
                \item Does not apply to our simulated environments.
            \end{arrowlist}
    	\item The range of hyper-parameters considered, method to select the best hyper-parameter configuration, and specification of all hyper-parameters used to generate results.
            \begin{arrowlist}
                \item See Sections \ref{sup:benchmarkalgos} and \ref{sup:helicopter_training}.
            \end{arrowlist}
    	\item The exact number of evaluation runs.
            \begin{arrowlist}
                \item 100,000+ for finite MDPs experiments and 300 for SPIBB-DQN experiments.
            \end{arrowlist}
    	\item A description of how experiments were run.
            \begin{arrowlist}
                \item See Sections \ref{sec:emp}, \ref{sup:expe_finite}, and \ref{sup:expe_helicopter}.
            \end{arrowlist}
    	\item A clear definition of the specific measure or statistics used to report results.
            \begin{arrowlist}
                \item Mean and X\% conditional value at risk (CVaR), described in Sections \ref{sec:emp} and \ref{sup:evaluationgen}.
            \end{arrowlist}
    	\item Clearly defined error bars.
            \begin{arrowlist}
                \item Given the high number of runs we considered, the error bar are too thin to be displayed. Any difference visible with the naked eye is significant. We use CVaR everywhere instead to account for the uncertainty.
            \end{arrowlist}
    	\item A description of results including central tendency (e.g. mean) and variation (e.g. stddev).
            \begin{arrowlist}
                \item All our work is motivated and analyzed with respect to this matter.
            \end{arrowlist}
    	\item The computing infrastructure used.
            \begin{arrowlist}
                \item For the finite-MDPs experiment, we used clusters of CPUs. The full results were obtained by running the benchmarks with 100 CPUs running independently in parallel during 24h. For the helicopter experiment, we used a GPU cluster. However, only one GPU is necessary for a single run. Using a cluster allowed to launch several runs in parallel and considerably sped up the experiment. On a single GPU (a GTX 1080 Ti), a dataset of $|\mathcal{D}| = 10$k transitions is generated in $~5$ seconds. The dataset generation scales linearly in $|\mathcal{D}|$. Computing the counts for that dataset takes approximately $20$ minutes, it scales quadratically with the size of the dataset. As far as training is concerned, $2000$ passes on a dataset of $10$k transitions takes around $25$ minutes, it scales linearly in $N$. Finally, evaluation of the trained policy on $10$k trajectories takes $15$ minutes. It scales linearly in $|\mathcal{D}|$ as it requires the computation of the pseudo-count for each state encountered during the evaluation and this pseudo-count computation is linear in $|\mathcal{D}|$. Overall, a single run for a dataset of $10$k transitions takes around one hour.
            \end{arrowlist}
	\end{itemize}
	
	\subsection{Code attached to the submission}
	\label{sup:code}
    The attached code can be used to reproduce the experiments presented in the submitted paper. It is split into two projects: one for finite MDPs, and one for Soft-SPIBB-DQN.
    \begin{itemize}
        \item \url{https://github.com/RomainLaroche/SPIBB}
        \item \url{https://github.com/rems75/SPIBB-DQN}
    \end{itemize}
    
    \subsubsection{Finite MDPs}
	\label{sup:code_finite}
    
    \paragraph{Prerequisites}
    The finite MDP project is implemented in Python 3.5 and only requires *numpy* and *scipy*.
    
    \paragraph{Content}
    We include the following:
    \begin{itemize}
        \item Libraries of the following algorithms:
            \begin{itemize}
                \item Basic RL,
                \item Soft-SPIBB:
                    \begin{itemize}
                        \item Exact-Soft-SPIBB (1-step or not),
                        \item Approx-Soft-SPIBB (1-step or not),
                    \end{itemize}
                \item SPIBB:
                    \begin{itemize}
                        \item $\Pi_b$-SPIBB,
                        \item $\Pi_{\leq b}$-SPIBB,
                    \end{itemize}
                \item HCPI:
                    \begin{itemize}
                        \item doubly-robust,
                        \item importance sampling,
                        \item weighted importance sampling,
                        \item weighted per decision IS,
                        \item per decision IS,
                    \end{itemize}
                \item Robust MDP,
                \item and Reward-adjusted MDP.
            \end{itemize}
        \item Environments:
            \begin{itemize}
                \item Random MDPs environment.
            \end{itemize}
        \item Random MDPs experiment of Section \ref{sec:RandomMDPs}. Run:
        
    		\footnotesize{\texttt{python soft\_randomMDPs\_main.py \#name\_of\_experiment\# \#random\_seed\#}}
    		
    \end{itemize}

    \paragraph{Not included} We DO NOT include the following:
    \begin{itemize}
        \item The hyper-parameter search (appendix, Section \ref{sup:benchmarkalgos}): it should be easy to re-implement.
        \item The figure generator: it has too many specificities to be made understandable for a user at the moment. Also, it is not hard to re-implement with one's own visualization tools.
    \end{itemize}
    
    \paragraph{License} This project is BSD-licensed.
    
    \subsubsection{Soft-SPIBB-DQN}
	\label{sup:code_spibb_dqn}

    \paragraph{Prerequisites}
    Soft-SPIBB-DQN is implemented in Python 3 and requires the following libraries: PyTorch, pickle, glob, yaml, argparse, numpy, yaml, pathlib, csv, scipy and click.
    
    \paragraph{Content}
    The SoftSPIBB-DQN project contains the helicopter environment, the baseline used for our experiments and the code required to generate datasets and train vanilla DQN, RaMDP-DQN, SPIBB-DQN and Soft-SPIBB-DQN.
    
    \paragraph{Not included} We DO NOT include the following:
    \begin{itemize}
        \item The multi-CPU/multi-GPU implementation: its structure is too much dependent on the cluster tools. It would be useless for somebody from another lab and might divulge author affiliations.
    \end{itemize}
    
    \paragraph{License} This project is BSD-licensed.

\end{document}